\documentclass[review,onefignum,onetabnum,onealgnum]{siamart250106}

\usepackage{amsmath,amsfonts,amsopn,amssymb}
\usepackage{graphicx}
\usepackage{multirow,relsize,makecell}
\usepackage{subfigure}
\usepackage{url}
\ifpdf
  \DeclareGraphicsExtensions{.eps,.pdf,.png,.jpg}
\else
  \DeclareGraphicsExtensions{.eps}
\fi
\usepackage{enumitem}

\usepackage{algorithm,algorithmic}

\numberwithin{theorem}{section}
\newsiamremark{remark}{Remark}
\newsiamremark{assumption}{Assumption}
\newsiamremark{example}{Example}
\crefname{remark}{Remark}{Remarks}
\crefname{assumption}{Assumption}{Assumptions}
\crefname{example}{Example}{Examples}

\headers{Frequency principle}{Qijia Zhai}
\title{Is the Frequency Principle always valid?}

\author{Qijia Zhai\thanks{School of Mathematics, Sichuan University, Chengdu 610064, China
 (\email{zhaiqijia@stu.scu.edu.cn})}}

\ifpdf
\hypersetup{
  pdftitle={Revisiting the Frequency Principle on the Sphere:
Spherical Harmonic Analysis and Neural Network Training Dynamics},
  pdfauthor={Qijia Zhai}
}
\fi

\usepackage[normalem]{ulem}
\usepackage{color}

\begin{document}

\maketitle

\begin{abstract}
We investigate the learning dynamics of shallow ReLU neural networks on the unit sphere \(S^2\subset\mathbb{R}^3\) in polar coordinates \((\tau,\phi)\), considering both fixed and trainable neuron directions \(\{w_i\}\). For fixed weights, spherical harmonic expansions reveal an intrinsic low-frequency preference with coefficients decaying as \(O(\ell^{5/2}/2^\ell)\), typically leading to the Frequency Principle (FP) of lower-frequency-first learning. However, this principle can be violated under specific initial conditions or error distributions. With trainable weights, an additional rotation term in the harmonic evolution equations preserves exponential decay with decay order \(O(\ell^{7/2}/2^\ell)\) factor, also leading to the Frequency Principle (FP) of lower-frequency-first learning. But like fixed weights case, the principle can be violated under specific initial conditions or error distributions. Our numerical results demonstrate that trainable directions increase learning complexity and can either maintain a low-frequency advantage or enable faster high-frequency emergence. This analysis suggests the FP should be viewed as a tendency rather than a rule on curved domains like \(S^2\), providing insights into how direction updates and harmonic expansions shape frequency-dependent learning.
\end{abstract}



\section{Introduction}
\label{Sec:Introduction}
In recent years, deep learning implemented by deep neural networks~(DNNs) has achieved great success in many applications such as computer vision~\cite{feng2019computer, xie2019source, o2020deep}, speech recognition~\cite{pan2012investigation, deng2013new, seltzer2013investigation}, translation~\cite{singh2017machine, zhang2015deep, yang2020survey}, and natural language processing~\cite{chowdhary2020natural, nadkarni2011natural, bird2009natural, khurana2023natural, liddy2001natural, eisenstein2019introduction}. Meanwhile, it has become an indispensable method for solving various scientific problems. However, DNNs sometimes fail and generate many issues, such as poor interpretability and limited generalization ability. In theory, DNNs' research has been like exploring a black box for decades. Many researchers compare the theoretical study and practical application of DNNs to alchemy. Therefore, establishing a better theoretical understanding of DNNs has become important.

There has been some progress in the theoretical research of DNNs in recent years. However, we still need to clearly understand how these theoretical results provide key insights and guidance for the practical research of DNNs. Regarding Frequency Principle research, some literature has conducted numerical and theoretical studies on the Frequency Principle, DNNs often fit target functions from low to high frequencies during the training  \cite{HSTX:2022, XZLXM:2020, LMXZ:2021, XZX:2019, Rahaman:2019, ma2020machine, zhang2023shallow, choraria2022spectral, xu2024overview}. Through these findings and research, we have gained a basic understanding of the capabilities and limitations of deep learning, namely the difficulty in learning and achieving good generalization for high-frequency data. At the same time, it is relatively simple for low-frequency data. Based on the guiding principle of this frequency, many algorithms have been developed to utilize the low-frequency bias of DNN to fit smooth data well, or to design special techniques/architectures to alleviate the difficulties of DNN in fitting known highly oscillatory data \cite{liu2020multi, jagtap2020adaptive, tancik2020fourier, cai2019multi}.

In numerous studies, the Frequency Principle has been subjected to theoretical analyses. For instance, in \cite{XZLXM:2020, LMXZ:2021}, the authors conducted frequency analyses on neural networks employing nonlinear activation functions such as $\tanh$ and general activation functions. They estimated the relation between the convergence speed of neural networks under Fourier transformation in the entire space ($\mathbb{R}$) and the frequency, revealing that higher frequencies correspond to lower convergence speeds. In \cite{Rahaman:2019}, the authors focused on the ReLU activation function and extended the neural network function across the entire Fourier space ($\mathbb{R}^d$). They utilized classical Fourier-based interpretations that may fail when artificially introduced high-frequency terms are present by $O(k^{-\delta-1})$, where $k$ represents the $k$th frequency, $1\le \delta\le d$, and $d$ denotes the dimension of the space. In \cite{HSTX:2022}, the authors presented a theoretical elucidation for the spectral bias observed in ReLU neural networks by drawing upon connections with the theory of finite element methods and eigenvalue decomposition. Subsequently, building upon this theoretical framework, they demonstrated that replacing the activation function with a piecewise linear B-spline, specifically the Hat function, could mitigate this spectral bias. They empirically verified this assertion across various settings.

However, our findings reveal that frequency alone cannot fully determine the convergence behavior of neural networks. While the Frequency Principle suggests a natural inclination towards learning low-frequency components first, we show that classical Fourier-based interpretations may fail when artificially introduced high-frequency terms are present. In these scenarios, the interplay between frequency and training error becomes critical. Specifically, we observe that:
\begin{itemize}
    \item When the training error is not negligible, certain high-frequency components can converge faster than low-frequency ones.
    \item Conversely, when the training error is extremely small, low-frequency components typically dominate, leading to faster convergence of low-frequency terms relative to high-frequency ones.
\end{itemize}

We go beyond the traditional Fourier analysis to provide deeper theoretical insights and employ spherical harmonic expansions of ReLU functions on the unit sphere. By examining the decay rates of spherical harmonic coefficients, we demonstrate conditions under which the Frequency Principle does not universally apply. Complementing this theoretical perspective, we employ standard Fourier series approaches numerically to illustrate that these conditions can indeed lead to a breakdown of the Frequency Principle.

This paper focuses on exploring the Frequency Principle under the following settings:
\begin{itemize}
\item Neural networks with ReLU activation functions and fixed weights;
\item Neural networks with ReLU activation functions and trained weights.
\end{itemize}

We choose to utilize the Fourier series method because it naturally accommodates bounded domains, making it suitable for both function approximation and the numerical study of partial differential equations. By periodically extending the function without introducing extraneous frequencies, the Fourier series analysis ensures a smooth testing ground for our experiments, allowing us to isolate and examine the effects of introduced high-frequency components and training errors in a controlled manner.

The main contributions of this article are as follows:
\begin{itemize}
\item \textbf{Refined Analysis of the Frequency Principle via Spherical Harmonics:} We re-examine the Frequency Principle through spherical harmonic expansions, uncovering scenarios where it fails to hold.
\item \textbf{Elucidation of Neural Network Learning Priorities Under High-Frequency Perturbations:} We highlight that frequency considerations do not solely govern neural network training dynamics. Instead, the interplay between frequency and error can invert the usual learning order, enabling high-frequency components to converge faster under certain conditions.
\item \textbf{Interpretable Numerical Experiments Across the Frequency Spectrum:} Through carefully designed numerical experiments, we illustrate how both frequency and error influence learning behavior. These experiments, using the familiar Fourier series framework, offer a more interpretable understanding of why the Frequency Principle may or may not emerge, depending on the target function’s properties and network parameters.
\end{itemize}

\section{ReLU Shallow Neural Networks on the Unit Sphere}
\label{Sec:Preliminaries}

We now formulate our shallow neural network (SNN) on the unit sphere 
\(S^2 \;=\;\bigl\{\,x(\tau,\phi)\in\mathbb{R}^3 :\|x(\tau,\phi)\|=1,\;0\le\tau\le\pi,\;0\le\phi<2\pi\bigr\}\) exclusively in \emph{polar coordinates}.  In these coordinates, each point on \(S^2\) is given by
\begin{equation}\label{spherecordinate}
    x(\tau,\phi) 
\;=\; 
\bigl(\sin\tau\cos\phi,\;\sin\tau\sin\phi,\;\cos\tau\bigr).
\end{equation}

We consider an SNN with a single hidden layer of \(m\) neurons and the activation function is ReLU:
\begin{equation}\label{Relu}
    \sigma(z) \;=\;\max\{0,\,z\}, 
\quad z\in\mathbb{R}.
\end{equation}

The output of the network is then a linear combination of these neuron outputs. Network architecture, parameters, loss functional and training dynamics in polar coordinates are described as follows. Concretely, let
$$\theta = \bigl((a_1,\dots,a_m)^\top,\,(w_1,\dots,w_m)^\top\bigr)$$
be the SNN parameters.  For each neuron \(i\) $(i = 1,2,\ldots,m)$, \(\,a_i \in \mathbb{R}\) is the scalar output weight, and
\(\,w_i\in\mathbb{R}^3\) is the weight vector specifying the linear functional
\(\,w_i^\top x(\tau,\phi)\).  
With the ReLU activation \(\sigma\) from \eqref{Relu}, the network output at \(\,x(\tau,\phi)\in S^2\) is
\begin{equation}\label{neuralnetwork-sphere}
u\bigl(\tau,\phi;\theta\bigr)
\;=\;
\sum_{i=1}^m 
a_i\,\sigma\bigl(w_i^\top x(\tau,\phi)\bigr)
\;=\;
\sum_{i=1}^m 
a_i\,\mathrm{ReLU}\bigl(w_i^\top x(\tau,\phi)\bigr).
\end{equation}

\begin{remark}
Throughout our analysis in Section 3, we often impose the normalization \(\|w_i\|=1\) for all \(i\), as in~\cite[Proposition~2.1]{AS:2021}.  This restricts each \(w_i\) to lie on \(S^2\) (up to orientation), yet still admits a wide class of representable functions. Under the normalization \(\|w_i\|=1\), aligning coordinate systems so \(w_i\) matches the polar axis \(\bigl(0,0,1\bigr)\) merely rotates the sphere.  Thus, such normalization and realignment simplify the theoretical treatment without loss of generality:
\begin{equation}\label{nnrepresent-sphere}
u\bigl(\tau,\phi;\theta\bigr)
\;=\;
\sum_{i=1}^m 
a_i \,\mathrm{ReLU}\bigl(w_i^\top x(\tau,\phi)\bigr),
\quad 
a_i\in\mathbb{R},\quad \|w_i\|=1.
\end{equation}
\end{remark}

Let \(h(\tau,\phi)\) be a prescribed target function on the sphere.  We measure the discrepancy between \(u(\tau,\phi;\theta)\) and \(h(\tau,\phi)\) using an \(L^2\)-type loss:
\begin{equation}\label{l2spheredomain}
L(\theta)
\;=\;
\frac12
\int_{0}^{2\pi}\!\!\int_{0}^{\pi}
\Bigl(u\bigl(\tau,\phi;\theta\bigr)-h(\tau,\phi)\Bigr)^2
\,\sin\tau\,d\tau\,d\phi.
\end{equation}

We train the parameters \(\theta\) via gradient descent over time \(t\in(0,T]\):
\begin{equation}\label{trainingode-sphere}
\begin{aligned}
\frac{d\theta}{dt} 
&=\;
-\nabla_{\theta}L(\theta),\\
\theta(0)
&=\;\theta_0,
\end{aligned}
\end{equation}
where \(\theta_0\) is the given initialization.

To track the evolution of the network’s fit to the target \(h(\tau,\phi)\), we define the \emph{error function}.  We set
\begin{equation}\label{errorxi-sphere}
    D(\tau,\phi)
\;=\;
u\bigl(\tau,\phi;\theta\bigr)
\;-\;
h(\tau,\phi),
\end{equation}
and if we integrate over \(S^2\), we have
\[
\int_{S^2}D(x)\,dx
\;=\;
\int_{0}^{2\pi}\!\!\int_{0}^{\pi}
D(\tau,\phi)\;\sin\tau\,d\tau\,d\phi,
\]
which shows that the same error function appears in either coordinate system.

Since \(h(\tau,\phi)\) is fixed, the temporal evolution of \(D(\tau,\phi)\) directly reflects changes in \(\theta\).  By the chain rule,
\begin{equation}\label{changeofD-sphere}
\frac{\partial D(\tau,\phi)}{\partial t}
\;=\;
\frac{\partial u(\tau,\phi;\theta)}{\partial \theta}
\;\cdot\;
\frac{d\theta}{dt}
\;=\;
\frac{\partial u(\tau,\phi;\theta)}{\partial \theta}
\;\cdot\;
\frac{\partial L(\theta)}{\partial \theta}.
\end{equation}

This formalizes how, under gradient descent, the error function \(D(\tau,\phi)\) evolves over time \(t\).  In the following sections, we investigate these dynamics by examining the spherical harmonic expansions of the ReLU activation on the sphere, providing a detailed frequency-level perspective on how the SNN learns \(h(\tau,\phi)\).

\subsection{Spherical Harmonic Expansion in Polar Coordinates}

We now review the standard \emph{spherical harmonics} on the unit sphere \(S^2\), and expressed in polar coordinates \((\tau,\phi)\).  Recall that each point on \(S^2\subset\mathbb{R}^3\) can be written as
\[
x(\tau,\phi) \;=\; 
\bigl(\sin\tau\cos\phi,\;\sin\tau\sin\phi,\;\cos\tau\bigr),
\quad
0\le\tau\le\pi,\,
0\le\phi<2\pi.
\]

In these angles, the spherical harmonics \(Y_{\ell}^j(\tau,\phi)\) are defined by
\begin{equation}\label{YelljDefPolar}
Y_\ell^j(\tau,\phi)
\;=\;
N_\ell^j 
\,P_\ell^j\!\bigl(\cos\tau\bigr)
\,e^{\,i\,j\,\phi},
\quad
\ell\ge0,
\quad
-\ell\le j\le \ell,
\end{equation}
where
\begin{equation}\label{NelljDef}
N_\ell^j
\;=\;
\sqrt{\frac{(2\ell+1)\,(\ell-j)!}{4\pi\,(\ell+j)!}}
\end{equation}
is the normalization constant, and 
\[
P_\ell^j(x)
\;=\;
(-1)^j\,(1-x^2)^{\tfrac{j}{2}}\,
\frac{d^j}{dx^j}\,P_\ell(x), \quad x\in \mathbb{R}
\]
are the \emph{associated Legendre polynomials}, derived from the usual Legendre polynomials \(P_\ell(x)\) via Rodrigues’ formula.

These functions \(\{Y_{\ell}^j(\tau,\phi)\}\) form a complete orthonormal basis in \(L^2(S^2)\).  More precisely:
\begin{enumerate}
\item 
\emph{Orthogonality:} For any \(\ell,\ell'\ge0\) and \(-\ell\le j\le \ell\), \(-\ell'\le j'\le \ell'\),
\begin{equation}
\int_{0}^{2\pi}\!\!\!\int_{0}^{\pi}
Y_{\ell}^j(\tau,\phi)\,\overline{Y_{\ell'}^{j'}(\tau,\phi)}
\;\sin\tau\,d\tau\,d\phi
\;=\;
\delta_{\ell\ell'}\,\delta_{jj'},
\end{equation}
where $\overline{Y_{\ell'}^{j'}(\tau,\phi)}$ denotes the complex conjugate of the spherical harmonic \(Y_{\ell'}^{j'}(\tau,\phi)\).
\item 
\emph{Completeness:} Any square-integrable function \(h(\tau,\phi)\in L^2(S^2)\) admits an expansion
\begin{equation}\label{uxexpansionPolar}
h(\tau,\phi)
\;=\;
\sum_{\ell=0}^\infty 
\sum_{j=-\ell}^\ell 
h_{\ell j}\,
Y_\ell^j(\tau,\phi),
\end{equation}
where the coefficients \(h_{\ell j}\) can be computed by
\begin{equation}\label{cuxexpansionPolar}
h_{\ell j} 
\;=\;
\int_{0}^{2\pi}\!\!\!\int_{0}^{\pi}
h(\tau,\phi)\,\overline{Y_\ell^j(\tau,\phi)}
\;\sin\tau\,d\tau\,d\phi.
\end{equation}
\end{enumerate}

The index \(\ell\) represents the \emph{degree} of the harmonic and sets the fundamental angular frequency, while \(j\) (the \emph{order}) determines the azimuthal dependence.  Larger \(\ell\) corresponds to higher-frequency components in this expansion.

Since any real-valued function on \(S^2\) can be expanded in spherical harmonics, the same holds for the neural network output 
\(\,u(\tau,\phi;\theta)\).  Hence, one may write
\begin{equation}\label{uxthetaexpansionPolar}
u\bigl(\tau,\phi;\theta\bigr)
\;=\;
\sum_{\ell=0}^{\infty}
\sum_{j=-\ell}^{\ell}
u_{\ell j}(\theta)\,
Y_{\ell}^j(\tau,\phi),
\end{equation}
where the coefficient
\begin{equation}\label{cuxthetaexpansionPolar}
u_{\ell j}(\theta)
\;=\;
\int_{0}^{2\pi}\!\!\!\int_{0}^{\pi}
u\bigl(\tau,\phi;\theta\bigr)
\,\overline{Y_{\ell}^j(\tau,\phi)}
\;\sin\tau\,d\tau\,d\phi
\end{equation}
records how strongly \(u(\tau,\phi;\theta)\) projects onto each spherical harmonic mode \(Y_{\ell}^j\).  Later sections will exploit this expansion to investigate how low- and high-frequency components of \(u(\tau,\phi;\theta)\) evolve during gradient-based training.

\subsection{Error Evolution under Spherical Harmonic Expansion}
Starting from the error evolution equation \eqref{changeofD-sphere}, the temporal derivative of the error function \( D(\tau, \phi) = u(\tau, \phi; \theta) - h(\tau, \phi) \) is expanded in spherical harmonics:
\begin{equation}
    \frac{\partial D(\tau, \phi)}{\partial t} = \sum_{\ell=0}^{\infty} \sum_{j=-\ell}^{\ell} D_{\ell j} \, Y_{\ell}^j(\tau, \phi),
\end{equation}
where \( D_{\ell j} \) is the coefficient of the spherical harmonic \( Y_{\ell}^j(\tau, \phi) \), computed as
\begin{equation}
    D_{\ell j} = \int_{0}^{2\pi} \int_{0}^{\pi} \frac{\partial D(\tau, \phi)}{\partial t} \, \overline{Y_{\ell}^j(\tau, \phi)} \, \sin \tau \, d\tau \, d\phi.
\end{equation}

This expansion leverages the orthonormality of the spherical harmonics (see \eqref{uxexpansionPolar} and \eqref{cuxexpansionPolar}), decomposing the error dynamics into frequency modes defined by the degree \( \ell \) and order \( j \). The degree \( \ell \) governs the angular frequency, with larger \( \ell \) corresponding to finer spatial variations on \( S^2 \).


We define the conditions under which the ReLU SNN, as given by \eqref{neuralnetwork-sphere}, adheres to the Frequency Principle. For a fixed azimuthal order \( j \), consider two frequency degrees \( \ell_1 \) and \( \ell_2 \) with \( \ell_1 < \ell_2 \), representing lower and higher frequency components, respectively. The network adheres to the FP if the magnitude of the error coefficient for the lower-frequency mode exceeds that of the higher-frequency mode: $D_{\ell_1 j}<0$ and $D_{\ell_2 j} < 0$,
\begin{equation}
    |D_{\ell_1 j}| > |D_{\ell_2 j}|,
\end{equation}
indicating that the gradient descent dynamics \eqref{trainingode-sphere} reduce the error in the lower-frequency harmonic \( \ell_1 \) more rapidly than in \( \ell_2 \). This aligns with empirical observations where smoother, low-frequency features are captured earlier due to their broader spatial influence on \( S^2 \).


A violation of the Frequency Principle occurs when the network does not prioritize lower-frequency components. For a fixed \( j \) and \( \ell_1 < \ell_2 \), the FP is violated if
\begin{equation}
    |D_{\ell_1 j}| \leq |D_{\ell_2 j}|,
\end{equation}
suggesting that higher-frequency components are learned at a rate comparable to or exceeding that of lower-frequency ones. This may stem from the non-linear ReLU activation \eqref{Relu}, specific initializations \( \theta_0 \), or the geometry of the normalized weight vectors \( w_i \) (\( \|w_i\| = 1 \)), which could amplify higher-frequency modes under certain conditions.

\section{Frequency Analysis on the Unit Sphere for Fixed Weights}
\label{Sec:SphericalAnalysisFW}

In this section, we analyze the frequency behavior of shallow neural network (SNN) defined on the unit sphere \( S^2  \) under the assumption that the weight vectors \( w_i \) remain fixed throughout training. By holding the weights constant, we focus on understanding how the network’s frequency representation is inherently shaped by the ReLU activation function and the fixed directions \( w_i \). This setting provides a baseline frequency analysis, free from the complexities introduced by weight adaptation.

The unit sphere is a natural domain on which to conduct this analysis since spherical harmonics form a complete orthonormal basis for square-integrable functions on \( S^2 \). By expanding functions into spherical harmonics, we can precisely characterize their frequency components and thus investigate how a ReLU-activated SNN distributes energy across different spherical frequencies.

We consider the SNN of the form \eqref{neuralnetwork-sphere}, where each \( w_i \in S^2 \) is a fixed unit vector (\(\|w_i\|=1\)) and \( a_i \in \mathbb{R} \). The inner product \( w_i^\top x \) corresponds to projecting the input \( x \) onto the direction \( w_i \). Since the \( w_i \) do not change during training, the directional structure imposed on the input by \( w_i \) remains constant, allowing us to focus solely on how the ReLU activation reshapes the frequency content of the target function.

\subsection{Expansion of a Single Neuron with Fixed Weights in Polar Coordinates}
\label{Sec:SingleNeuronExpansion}

Having established the basic spherical harmonic machinery, we now analyze the simpler case of a \emph{single neuron} with ReLU activation:
\[
\mathrm{ReLU}\bigl(w_i^\top x(\tau,\phi)\bigr),
\]
where \(w_i\in\mathbb{R}^3\) is fixed (i.e., not trained), and \(x(\tau,\phi)\in S^2\) is the input in polar coordinates.  
To study its spherical harmonic expansion, we first interpret \(\,w_i^\top x(\tau,\phi)\) geometrically.

Since \(\|x(\tau,\phi)\|=1\) and  \(\|w_i\|=1\), the quantity \(w_i^\top x(\tau,\phi)\) is 
\[
w_i^\top x(\tau,\phi)
\;=\;
\cos(\tau'),
\]
where \(\tau'\in[0,\pi]\) is the angle between \(w_i\) and \(x(\tau,\phi)\).  By choosing a rotated coordinate system so that \(w_i\) coincides with the north pole (0,0,1), this angle \(\tau'\) can be identified with the polar angle.  Thus
\begin{equation}
\label{NeuronCosTau}
\mathrm{ReLU}\bigl(w_i^\top x(\tau,\phi)\bigr)
\;=\;
\max\!\bigl\{0,\,\cos(\tau')\bigr\}.
\end{equation}

In particular, \(\cos(\tau')\) is nonnegative only if \(\tau'\in[0,\tfrac{\pi}{2}]\).  This means the neuron output is positive (active) in a polar cap region about \(w_i\) and zero elsewhere, partitioning \(S^2\) into \emph{active} and \emph{inactive} zones.

Because \(\mathrm{ReLU}\bigl(w_i^\top x(\tau,\phi)\bigr)\) depends only on \(\tau'\) (and not on the azimuth \(\phi\)), it is an \emph{axisymmetric} function on the sphere.  Consequently, its spherical harmonic expansion will involve only the \(j=0\) modes:
\begin{equation}\label{expansionofReLU}
\mathrm{ReLU}\bigl(w_i^\top x(\tau,\phi)\bigr)
\;=\;
\sum_{\ell=0}^{\infty}
c_{\ell}
\,Y_\ell^0(\tau,\phi),
\end{equation}
where each coefficient \(c_{\ell}\) is given by an inner product with \(Y_{\ell}^0(\tau,\phi)\).  In polar coordinates,
\[
Y_{\ell}^0(\tau,\phi) 
\;=\;
\sqrt{\frac{2\ell+1}{4\pi}}\;
P_{\ell}\!\bigl(\cos\tau\bigr).
\]

Since \(\mathrm{ReLU}(\cos\tau')\) is nonzero only for \(\tau'\in[0,\frac{\pi}{2}]\), we may express
\[
c_{\ell}
\;=\;
\int_{S^2}
\mathrm{ReLU}\bigl(\cos\tau'\bigr)
\,\overline{Y_{\ell}^0(\tau',\phi')}\;d\Omega
\;=\;
2\pi\sqrt{\frac{2\ell+1}{4\pi}}
\int_{0}^{1}
t\,P_{\ell}(t)\,dt,
\]
where \(t=\cos\tau'\in[0,1]\) and \(d\Omega=\sin(\tau')\,d(\tau')\,d\phi'\).  

Carrying out the integral carefully, one finds for \(\ell\ge2\) \cite{bach2017breaking}:
\[
c_{\ell}
\;=\;
\frac{\sqrt{\pi}}{24}
\bigl(\tfrac12\bigr)^\ell
\sqrt{2\ell+1}\,\bigl(\ell^2 + 3\ell + 2\bigr),
\]
while for \(\ell=0\), \(c_{0}=\tfrac{{\pi}}{2}\), and for \(\ell=1\), \(c_{1}=\tfrac{\sqrt{3\pi}}{3}\).  In general, these coefficients satisfy
\[
c_{\ell}
\;=\;
O\!\Bigl(\frac{\ell^{5/2}}{2^\ell}\Bigr)
\quad
\text{as}\;\ell\to\infty,
\]
indicating that \(\mathrm{ReLU}(\cos\tau')\) has significant low-frequency components but rapidly decaying high-frequency components.

\begin{figure}[t]
    \centering
    \includegraphics[width=0.5\linewidth]{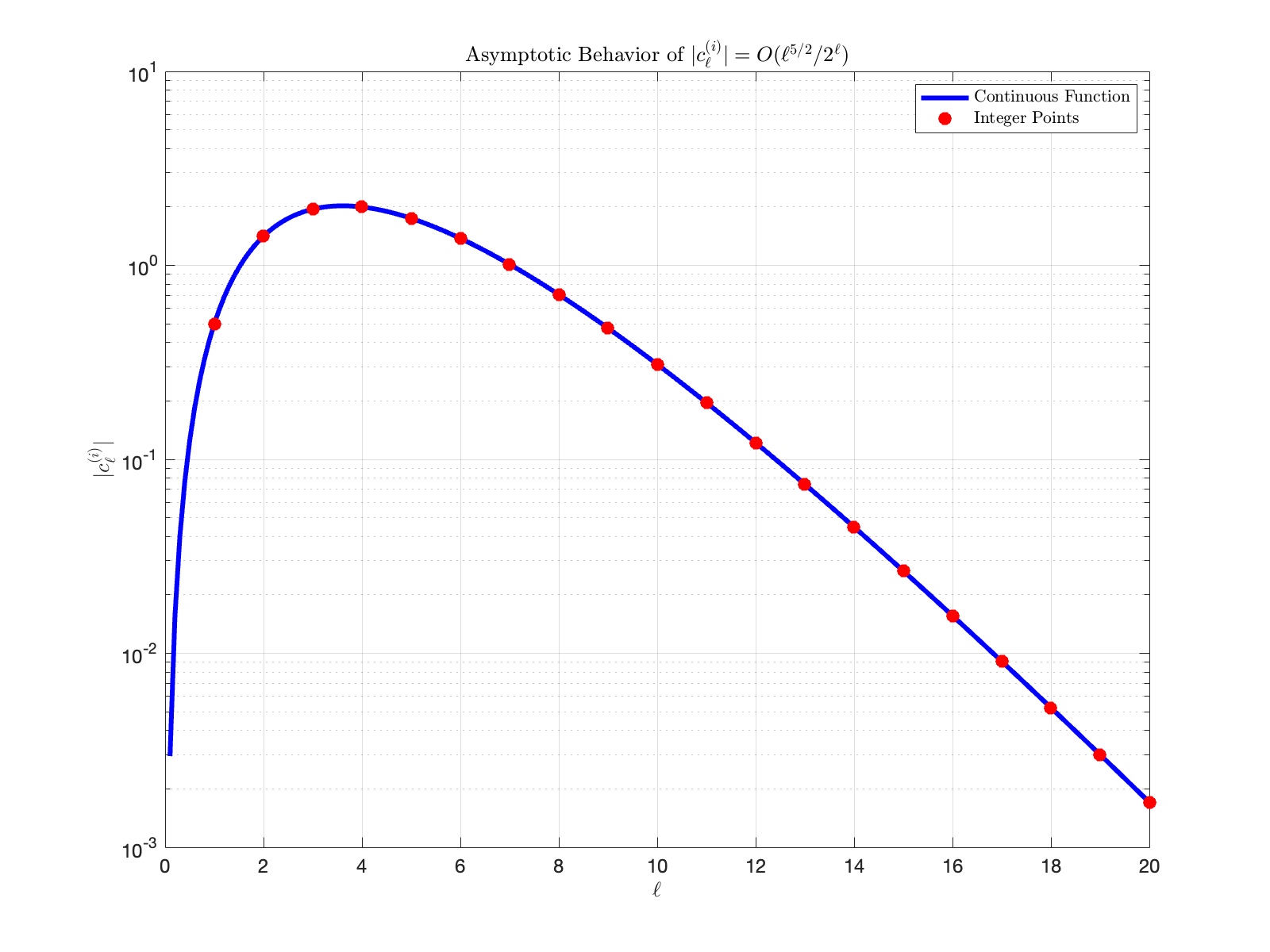}
    \caption{Asymptotic behavior of \(\bigl|c_\ell \bigr|\).  Note the polynomial factor \(\ell^{5/2}\) is dominated by the exponential decay \(2^{-\ell}\).}
    \label{behavierofci}
\end{figure}

For an SNN of the form~\eqref{neuralnetwork-sphere}, each \(\mathrm{ReLU}(w_i^\top x(\tau,\phi))\) can be expanded similarly.  The overall harmonic coefficient in the \(j=0\) mode becomes a linear combination of \(\{c_{\ell}\}\).  Specifically, if
\[
u\bigl(\tau,\phi;\theta\bigr)
\;=\;
\sum_{i=1}^{m}
a_i\;\mathrm{ReLU}\bigl(w_i^\top x(\tau,\phi)\bigr),
\]
then
\[
u_{\ell j}(\theta)
\;=\;
\int_{S^2}
u(\tau,\phi;\theta)
\,\overline{Y_{\ell}^j(\tau,\phi)}
\;d\Omega
\;=\;
\sum_{i=1}^{m}
a_i\;c_{\ell}\,\delta_{j0}.
\]

Hence, the low-frequency bias in each single neuron’s expansion (i.e., the decay of \(c_{\ell}\) with \(\ell\)) underpins the frequency behavior of the entire network on the sphere.  Analyzing these coefficients \(\{c_{\ell}\}\) thus reveals how the SNN distributes energy across spherical frequencies, providing insight into phenomena like the Frequency Principle.

\subsection{Error Dynamics under Gradient Descent with Fixed Directions}
\label{Sec:ErrorEvolution}

We now analyze the case where only the \emph{output-layer} coefficients 
\(\,\theta=(a_1,\dots,a_m)^\top\) are trained, while the directions \(w_i\in S^2\) remain 
\emph{fixed}.  
We define the \emph{error function} in \eqref{errorxi-sphere} and measure this error via the \(L^2\)-type loss \eqref{l2spheredomain} and we train \(\theta\) by gradient descent on \(L(\theta)\).

Observe that
\[
\frac{\partial u(\tau,\phi;\theta)}{\partial a_i}
\;=\;
\mathrm{ReLU}\bigl(w_i^\top x(\tau,\phi)\bigr),
\]
so 
\[
\frac{\partial D(\tau,\phi)}{\partial t}
\;=\;
\sum_{i=1}^m
\frac{\partial u(\tau,\phi;\theta)}{\partial a_i}
\;\frac{\partial L(\theta)}{\partial a_i}.
\]

Moreover, for the reason that
\[
\frac{\partial L(\theta)}{\partial a_i}
\;=\;
-\int_{0}^{2\pi}\!\!\int_{0}^{\pi}
\Bigl[u(\tau',\phi';\theta)-h(\tau',\phi')\Bigr]\,
\mathrm{ReLU}\bigl(w_i^\top x(\tau',\phi')\bigr)
\,\sin\tau'\,d\tau'\,d\phi',
\]
hence
\[
 \frac{\partial L(\theta)}{\partial a_i}
\;=\;
-\int_{0}^{2\pi}\!\!\int_{0}^{\pi}
D(\tau',\phi')\,
\mathrm{ReLU}\bigl(w_i^\top x(\tau',\phi')\bigr)
\,\sin\tau'\,d\tau'\,d\phi'.
\]

Consequently,
\begin{equation}\label{DupdateFixedDirPolar}
\begin{aligned}
&\frac{\partial D(\tau,\phi)}{\partial t}
\;=\;\\
& -\sum_{i=1}^m
\mathrm{ReLU}\bigl(w_i^\top x(\tau,\phi)\bigr)
\int_{0}^{2\pi}\!\!\int_{0}^{\pi}
D(\tau',\phi')\,\mathrm{ReLU}\bigl(w_i^\top x(\tau',\phi')\bigr)
\,\sin\tau'\,d\tau'\,d\phi'.
\end{aligned}
\end{equation}

If all \(w_i\) align along the polar axis (e.g., \(w_i^\top x(\tau,\phi) = \cos\tau'\)), then 
\(\mathrm{ReLU}(\cos\tau')\) is $\cos\tau'$ on \(\tau'\in[0,\tfrac{\pi}{2}]\) and $0$ otherwise.  
In that special case,
\[
u(\tau,\phi;\theta)
\;=\;
\Bigl(\sum_{i=1}^m a_i\Bigr)\,\cos(\tau'),
\quad
D(\tau,\phi)
\;=\;
\Bigl(\sum_{i=1}^m a_i\Bigr)\cos(\tau')\;-\;h(\tau,\phi).
\]

Then, for each $i$, $i = 1,2, \ldots, m$
\[
\begin{aligned}
    &\int_{0}^{2\pi}\!\!\!\!\int_{0}^{\pi}
D(\tau',\phi')\,\mathrm{ReLU}\bigl(w_i^\top x(\tau',\phi')\bigr)
\,\sin\tau'\,d\tau'\,d\phi' \\
&\;=\;
\int_{0}^{2\pi}\!\!\!\!\int_{0}^{\frac{\pi}{2}}
\Bigl(\sum_{i=1}^m a_i \cos(\tau') - h(\tau',\phi')\Bigr)
\cos(\tau')\,\sin(\tau')\,d\tau'\,d\phi'.
\end{aligned}
\]
Since 
\(\int_{0}^{2\pi}\int_{0}^{\tfrac{\pi}{2}}\cos^2(\tau')\sin(\tau')\,d\tau'\,d\phi'= \tfrac{2\pi}{3}\),
the integral becomes
\[
\frac{2\pi}{3}\sum_{i=1}^m a_i
\;-\;
\int_{0}^{2\pi}\!\!\int_{0}^{\tfrac{\pi}{2}}
h(\tau',\phi')\,\cos(\tau')\,\sin(\tau')\,d\tau'\,d\phi'
\;\;=\;\;
\frac{2\pi}{3}\sum_{i = 1}^m a_i \;-\;C(h),
\]
where
\(
C(h) = \int_{0}^{2\pi}\!\!\int_{0}^{\tfrac{\pi}{2}}
h(\tau',\phi')\,\cos(\tau')\,\sin(\tau')\,d\tau'\,d\phi'
\)
denotes the contribution from \(h(\tau',\phi')\), independent of the learnable parameters
\(\{a_i\}\).  Substituting back into \eqref{DupdateFixedDirPolar} provides a closed-form expression 
for \(\tfrac{\partial D(\tau,\phi)}{\partial t}\).

In general, if the directions $\{w_i\}$ differ (so $w_i^\top x(\tau',\phi')$ defines different upper hemispheres), one can still decompose each ReLU integral piecewise over \(\tau'\in[0,\tfrac{\pi}{2}]\) but must account for the orientation of each $w_i$.  While more involved, the integrals remain straightforward to evaluate numerically or in closed form (for certain symmetric weight configurations). Thus, even without invoking higher-dimensional mean-value arguments, one sees explicitly how each  ``cap'' $\{\tau'\le \tfrac{\pi}{2}\}$ relative to $w_i$ contributes to the gradient signal.  

Hence \eqref{DupdateFixedDirPolar} captures the interplay of the fixed directions, the ReLU activation, and the error distribution $D(\tau,\phi)$ in governing the time evolution of the error under gradient descent.

\subsection{Implications for the Frequency Principle}
\label{Sec:ImplicationsFP}

In the setting of Section~\ref{Sec:ErrorEvolution}, where only the output-layer parameters 
\(\theta=(a_1,\ldots,a_m)\) are trained and each weight direction \(w_i\in S^2\) remains fixed, we have established that 
the time evolution of the error \eqref{errorxi-sphere} takes the form
\begin{equation}\label{eq:DupdateNew}
\frac{\partial D(\tau,\phi)}{\partial t}
\;=\;
-\sum_{i=1}^m
\Bigl[\tfrac{2\pi}{3}\!\!\sum_{k=1}^m a_k \;-\;C(h)\Bigr]\,
\mathrm{ReLU}\bigl(w_i^\top x(\tau,\phi)\bigr).
\end{equation}

From Section~\ref{Sec:SingleNeuronExpansion}, each single-neuron function 
\(\mathrm{ReLU}\bigl(w_i^\top x(\tau,\phi)\bigr)\) depends only on the polar angle (when $w_i$ is aligned).

Substituting this into \eqref{eq:DupdateNew}, we find
\begin{equation}
    \begin{aligned}
\frac{\partial D(\tau,\phi)}{\partial t}
&\;=\;
-\sum_{\ell=0}^\infty
\Biggl[
\sum_{i=1}^m
\Bigl(\tfrac{2\pi}{3}\!\!\sum_{k=1}^m a_k - C(h)\Bigr)\,c_{\ell}
\Biggr]\,
Y_{\ell}^0(\tau,\phi)\\
&\;=\;
-\sum_{\ell=0}^\infty
\Biggl[
m\Bigl(\tfrac{2\pi}{3}\!\!\sum_{k=1}^m a_k - C(h)\Bigr)\,c_{\ell}
\Biggr]\,
Y_{\ell}^0(\tau,\phi).
\end{aligned}
\end{equation}
Define
\begin{equation}\label{D_ellfixed}
D_{\ell}
\;:=\;
m\Bigl(\tfrac{2\pi}{3}\!\!\sum_{k=1}^m a_k - C(h)\Bigr)\,c_{\ell}.
\end{equation}
Hence each spherical harmonic mode \(\ell\) of the error is updated by \(D_{\ell}\).  We can now state 
conditions under which the usual Frequency Principle---where lower-frequency modes decay faster than 
higher ones---either fails or is preserved.

\begin{theorem}[Special Case for FP Violation on the Sphere]\label{thm:fp-fails-zeroD-new}
Let \(\{w_i\}_{i=1}^m \subset S^2\) be fixed unit vectors and align along the polar axis, and consider the shallow ReLU network \eqref{neuralnetwork-sphere}. Let \(D(\tau,\phi)\) defined in \eqref{errorxi-sphere} be the error function on the sphere.  
Suppose that at any training time $t$, if $$
\sum_{k=1}^m a_k = \frac{3}{2\pi}C(h)$$
with $C(h) = \int_{0}^{2\pi}\!\!\int_{0}^{\tfrac{\pi}{2}}
h(\tau',\phi')\,\cos(\tau')\,\sin(\tau')\,d\tau'\,d\phi'$, then
\[
D_{\ell}
\;=\;
0
\quad
\text{for every }\ell.
\]
\end{theorem}


\begin{remark}
Intuitively, if the error \(D(\tau,\phi)\) is zero on the regions of the sphere 
(where $\mathrm{ReLU}(w_i^\top x)$ is active) that contribute to the integral, 
then each mode sees no gradient at initialization.  Thus no low-frequency advantage emerges, leading to 
immediate FP violation.
\end{remark}

Recalling that \(\bigl|c_{\ell}\bigr|\approx O\!\bigl(\ell^{5/2}/2^\ell\bigr)\) , we deduce that higher-frequency components of \(\mathrm{ReLU}(w_i^\top x)\) decay exponentially in \(\ell\).  Depending on how
\(\tfrac{2\pi}{3}\sum_{k=1}^m a_k - C(h)\) scales with $\ell$, the SNN can either maintain or lose low-frequency dominance.  
Below are two corollaries echoing that trade-off.

\begin{theorem}[General Conditions for FP Violation on the Sphere]
\label{thm:FPviolationSphere}
Let \(\{w_i\}_{i=1}^m \subset S^2\) be fixed unit vectors and align along the polar axis, and consider the shallow ReLU network \eqref{neuralnetwork-sphere}. Let \(D(\tau,\phi)\) defined in \eqref{errorxi-sphere} be the error function on the sphere.  
Suppose that at any training time $t$, if $$
\sum_{k=1}^m a_k < \frac{3}{2\pi}C(h)$$
with $C(h) = \int_{0}^{2\pi}\!\!\int_{0}^{\tfrac{\pi}{2}}
h(\tau',\phi')\,\cos(\tau')\,\sin(\tau')\,d\tau'\,d\phi'$, then for all \(\ell\), the frequency principle can be \emph{violated}.  
In other words, the exponential decay of \(c_{\ell}\) can be surmounted by such error terms at sufficiently low 
frequencies, enabling high-frequency components to persist.
\end{theorem}
\begin{proof}[Proof of Theorem \ref{thm:FPviolationSphere}]
    This proof can be directly calculated from \eqref{D_ellfixed} that $c_{\ell}$ is positive for all $\ell$ and $\Bigl(\tfrac{2\pi}{3}\!\!\sum_{k=1}^m a_k - C(h)\Bigr)$ is negative. So $\frac{\partial D_{\ell}}{\partial t} > 0$, which means the neural network is not convergent.
\end{proof}

\begin{theorem}[Conditions for FP Preservation on the Sphere]
\label{thm:FPholdSphere}
Under the similar setting and notation as in Theorem~\ref{thm:FPviolationSphere}, but$$
\sum_{k=1}^m a_k > \frac{3}{2\pi}C(h),$$
the frequency principle \emph{holds} for all \(\ell\) and the error terms \(D\bigl(\tau_i,\phi_i\bigr)\) remain of the same symbolic order.  
That is, for higher modes \(\ell\), the exponential decay of \(c_{\ell}\) cannot be balanced by constant-order (or moderately 
growing) error terms, thereby causing the high-frequency components to diminish naturally.
\end{theorem}
\begin{proof}[Proof of Theroem \ref{thm:FPholdSphere}]
    The proof is similar to the proof of Theorem \ref{thm:FPviolationSphere}.
\end{proof}
\begin{remark}
These results highlight that high-frequency spherical harmonic components, whose amplitudes decay on the order of 
\(\ell^{5/2}/2^\ell\), are difficult to sustain in the trained network unless the factors 
\(D\bigl(\tau_i,\phi_i\bigr)\,\sin\bigl(2\,\tau_i\bigr)\) grow at least as fast as \(2^\ell / \ell^{5/2}\). 
Consequently, \emph{violating} the frequency principle at large \(\ell\) requires significantly amplified error terms, while 
for moderate or slowly varying \(D\bigl(\tau_i,\phi_i\bigr)\,\sin\bigl(2\,\tau_i\bigr)\), the exponential decay ensures 
that the learned solution exhibits low-frequency dominance in accordance with the FP.
\end{remark}

\noindent
Hence, in summary, once the exponential decay of $c_{\ell}$ is accounted for, only a sufficiently 
large or fast-growing error feedback \(\bigl(\tfrac{2\pi}{3}\sum_{k=1}^m a_k - C(h)\bigr)\) can 
sustain high-frequency modes.  Otherwise, the FP is preserved, manifesting as a bias toward learning 
low-frequency harmonics more rapidly than high-frequency ones.

\subsection{Implications for the Frequency Principle: Numerical Example}
\label{NumericalExampleFixedweight}

In this section, we present a numerical experiment to illustrate how a shallow network on the sphere adheres (or in specific cases, deviates) from the Frequency Principle. We begin by transforming spherical coordinates to a Cartesian representation, describe the neural network model and its loss function, and provide details on the spherical harmonic analysis and error monitoring.

To ensure compatibility with standard neural network inputs, we transform spherical coordinates \(\tau\) (polar angle) and \(\phi\) (azimuthal angle) to Cartesian coordinates \(\mathbf{x} = (x, y, z)\). Concretely, each point on the spherical surface is mapped via
\[
    x = \sin(\tau)\cos(\phi), 
    \quad
    y = \sin(\tau)\sin(\phi), 
    \quad
    z = \cos(\tau).
\]

This representation allows the neural network to operate directly. The numerical experiment employs a shallow feed-forward network to approximate a target function defined on \(S^2\). The network comprises:
\begin{itemize}
    \item An input layer for the spherical coordinates.
    \item A single hidden layer of size \(h\) with ReLU activation.
    \item A linear output layer produces a scalar value \(u_1\).
\end{itemize}

Mathematically, if \(\mathbf{x} = \mathbf{x}(\tau, \phi)\), the model is
\[
    \mathbf{z}_1 
    = \mathrm{ReLU}\!\bigl(W_1\,\mathbf{x}(\tau, \phi)\bigr), 
    \qquad
    u_1 
    = W_2\,\mathbf{z}_1,
\]
where \(W_1,\,W_2\) are weight matrices, $W_1$ is fixed and $W_2$ is trainable.

To learn a function \(h(\tau,\phi)\) on the sphere, we minimize a weighted mean squared error (MSE) reflecting the spherical geometry:
\[
    \mathrm{Loss} 
    \;=\; \frac{1}{N} \sum_{i=1}^N 
        \Bigl(
            u(\tau_i,\phi_i,\theta)
            \;-\;
            h(\tau_i,\phi_i)
        \Bigr)^2
        \,\sin(\tau_i).
\]
The factor \(\sin(\tau_i)\) incorporates the differential area element \(d\Omega=\sin(\tau)\,d\tau\,d\phi\), ensuring uniform coverage of the sphere's surface in the loss computation.

To quantify approximation accuracy and analyze how different frequency components evolve, we compute the spherical harmonic coefficients \(c_{lj}\). The spherical harmonics \(Y_{l}^{j}(\tau,\phi)\) are defined by
\[
    Y_{l}^{j}(\tau,\phi)
    \;=\;
    \sqrt{\frac{2l+1}{4\pi}\,\frac{(l-j)!}{(l+j)!}}\;
    P_l^j(\cos\tau)\;e^{\,i\,j\,\phi},
\]
where \(P_l^j\) are the associated Legendre polynomials. The coefficients of the neural network's learned function \(h(\tau,\phi)\) are computed via
\[
    c_{lj}
    \;=\;
    \int_0^\pi\!\int_0^{2\pi}
        h(\tau,\phi)\;Y_{l}^{j^*}(\tau,\phi)\;\sin(\tau)\;d\phi\,d\tau,
\]
with \(Y_{l}^{j^*}\) denoting the complex conjugate of \(Y_{l}^{j}\). In practice, the code approximates this double integral using a suitable numerical quadrature over sampled points on the sphere.

We train the network by iteratively minimizing the MSE discussed above. At each epoch, two primary error metrics are tracked:
\begin{enumerate}
    \item \textit{Training loss}: The global mean squared error over the spherical dataset, often plotted on a logarithmic scale to highlight convergence behavior.
    \item \textit{Spherical harmonic coefficient errors}: We compare the learned coefficients \(c_{lj,\mathrm{pred}}\) with the true coefficients \(c_{lj,\mathrm{true}}\) for all degrees up to \(l_{\max}\). Specifically, we track 
    \(\bigl|c_{lj,\mathrm{pred}} \;-\; c_{lj,\mathrm{true}}\bigr|\) and visualize these discrepancies as heat maps over the \((l,j)\) grid.
\end{enumerate}
This approach, monitoring both the global MSE and the frequency decomposition, highlights whether the network follows the typical “low frequencies first” trajectory associated with the Frequency Principle. Prolonged or unexpected elevation in high-frequency modes can thus signal a deviation from the principle, shedding light on potential factors (e.g., network initialization, optimization hyper-parameters) that can alter the usual learning order.

\subsubsection{$u(\tau,\phi)=0$}

We consider a target function $u(\tau,\phi)=0$ defined on the unit sphere $\mathbb{S}^2$. The neural network architecture consists of a shallow network with 100 neurons in the hidden layer, employing ReLU activation functions. The training data comprises 100 sampling points uniformly distributed on the sphere. The network is trained using stochastic gradient descent (SGD) with a learning rate of $1e{-3}$ over 10,000 epochs.

\textbf{Case 1: Complete Adherence to Frequency Principle}
\begin{figure}
    \centering
    \includegraphics[width=6cm, height =4cm]{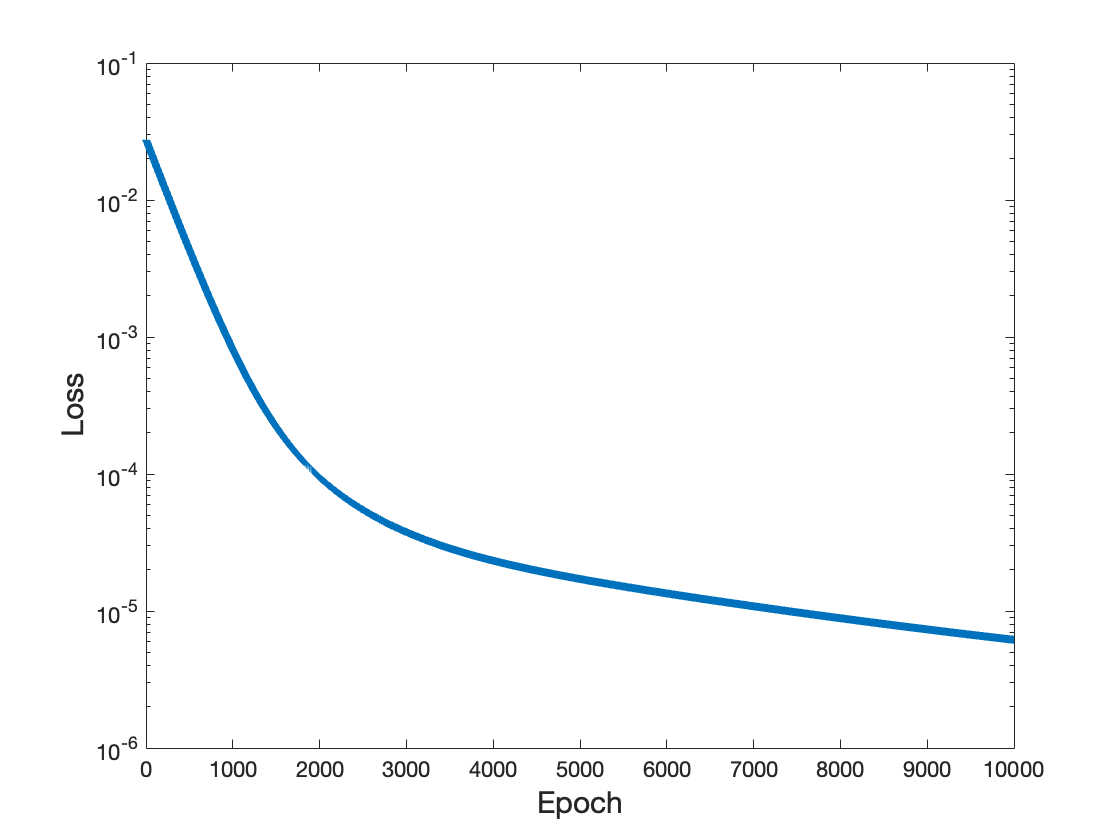}
    \caption{Image of Loss of the first test (Best Loss: 6.16e-6).}
    \label{1test}
\end{figure}
\begin{figure}
    \centering
    \subfigure[$l = 1,2,\ldots, 5$, $j=0$]{\includegraphics[width=6cm, height =4cm]{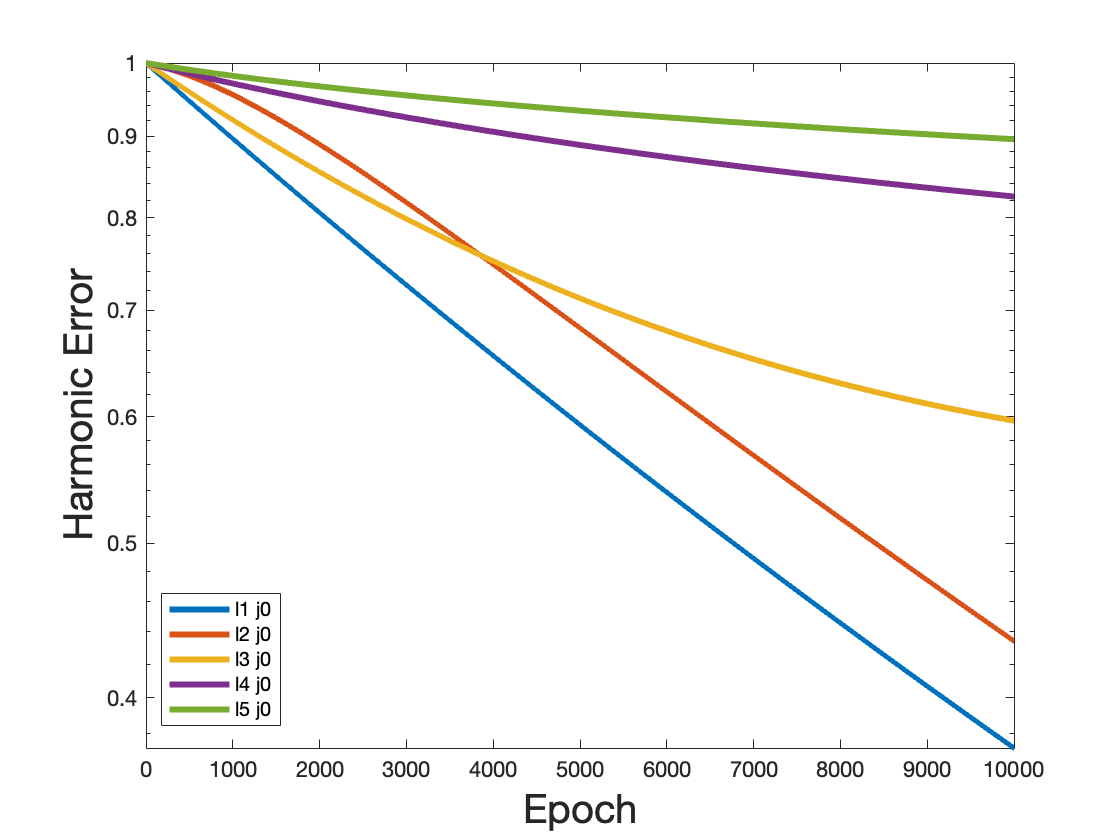}}
    \subfigure[$l = 6,7,\ldots, 10$, $j=0$]{\includegraphics[width=6cm, height =4cm]{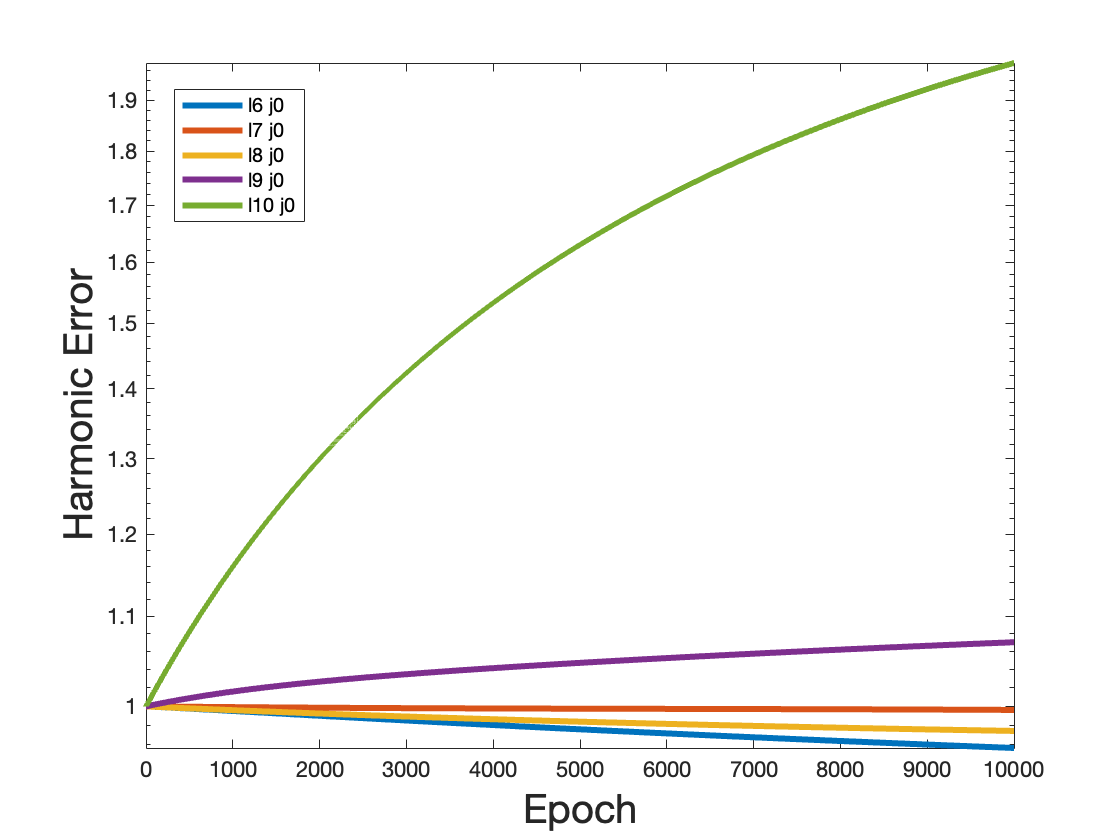}}
    \caption{Harmonic Errors of $c_{lj}$ When $l = 1,2,\ldots, 10$, $j=0$.}
    \label{1testalphabetam0}
\end{figure}

\begin{figure}
    \centering
    \subfigure[Image of the target function]{\includegraphics[width=6cm, height =4cm]{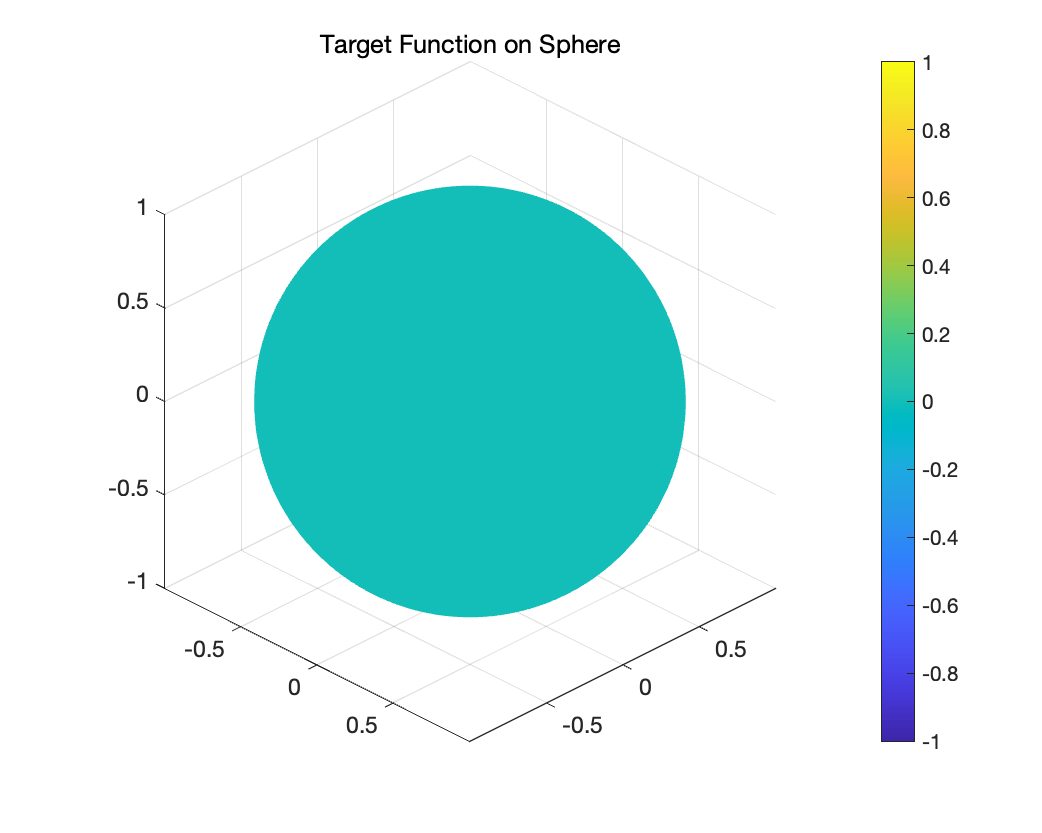}}
    \subfigure[Image of the SNN output]{\includegraphics[width=6cm, height =4cm]{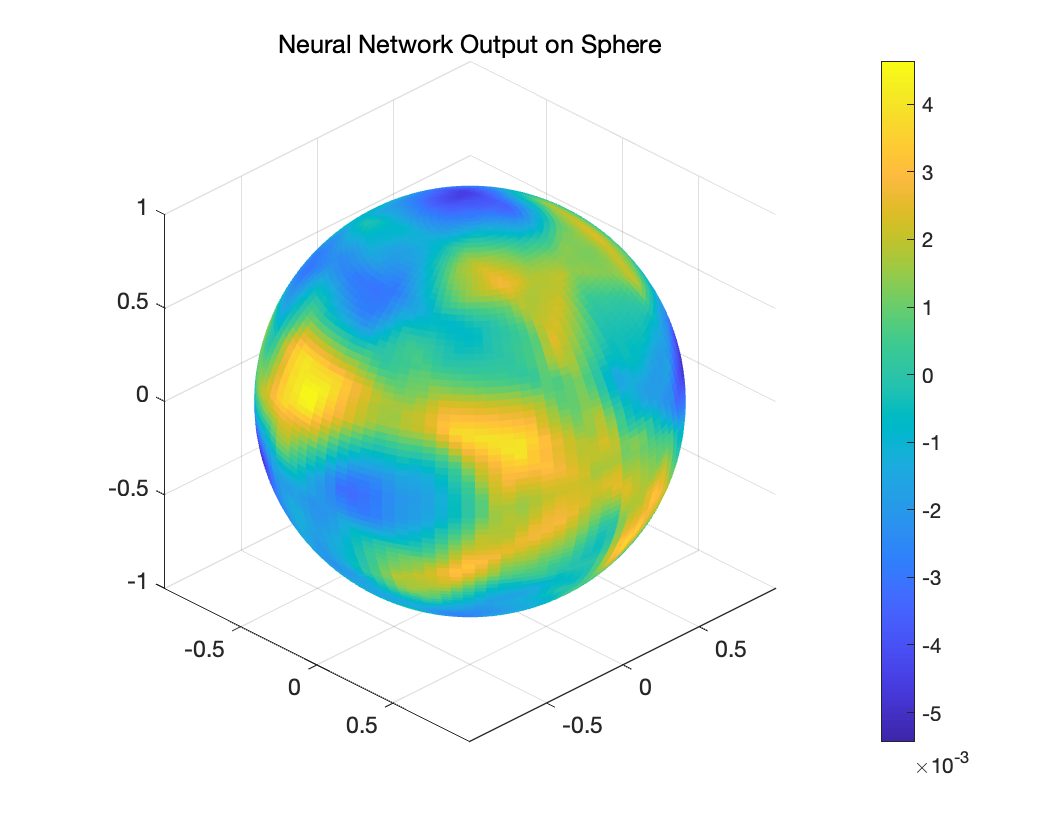}}

   \subfigure[Image of the error between the target function and SNN]{\includegraphics[width=6cm, height =4cm]{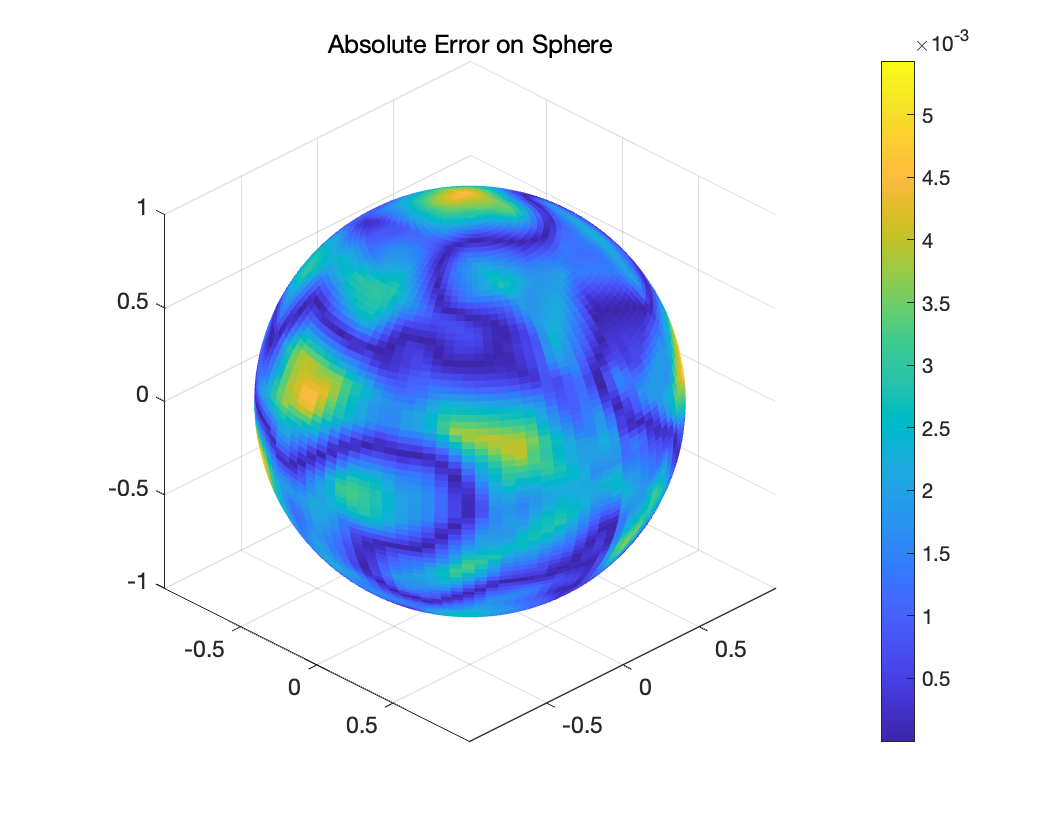}}
    \caption{The target function, SNN output, and error of the first test.}
    \label{1testdnntarget}
\end{figure}

From Figure \ref{1test}, we can see that the loss goes down to 6.16e-6, and Figure \ref{1testalphabetam0} shows the harmonic error of the first 10 frequencies of $j=0$. In this case, the neural network's learning behavior strictly follows the frequency principle. The network first learns low-frequency components of the target function before progressing to higher frequencies. This hierarchical learning pattern is evident in the spherical harmonic coefficient errors, which show a clear decremental pattern from lower to higher degrees.

Visualizing the target function, network output, and absolute error demonstrates the high accuracy of the approximation (Figure \ref{1testdnntarget}). The error distribution is notably uniform across the sphere, indicating that the network has successfully captured both low and high-frequency components of the target function.

\textbf{Case 2: Partial Adherence to Frequency Principle}
\begin{figure}
    \centering
    \includegraphics[width=6cm, height =4cm]{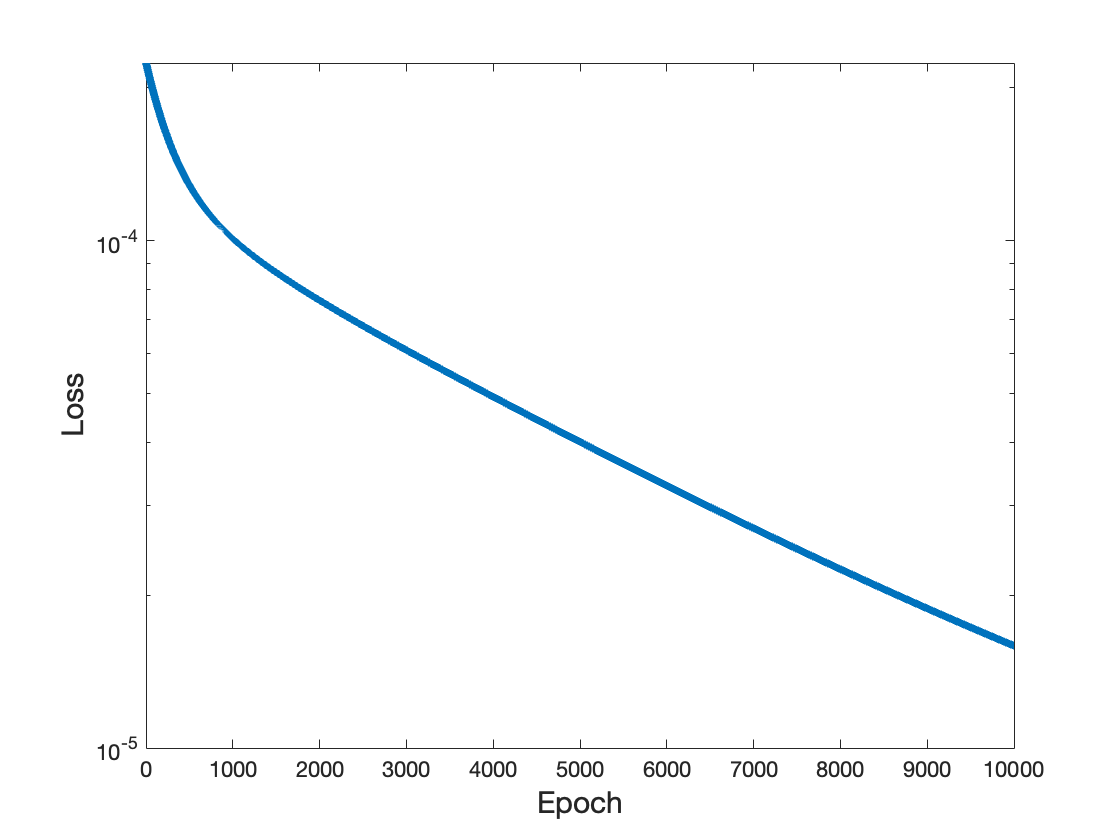}
    \caption{Image of Loss of the second test (Best Loss: 1.59e-5).}
    \label{2test}
\end{figure}
\begin{figure}
    \centering
    \subfigure[$l = 1,2,\ldots, 5$, $j=0$]{\includegraphics[width=6cm, height =4cm]{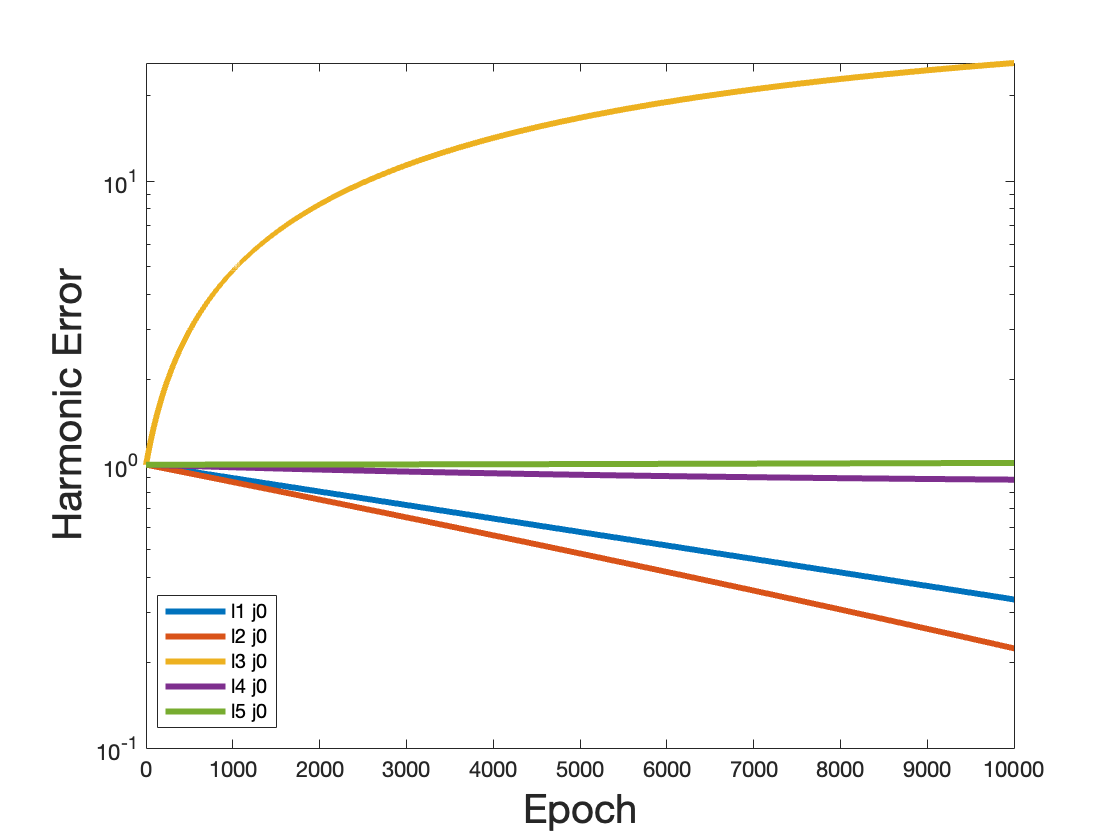}}
    \subfigure[$l = 6,7,\ldots, 10$, $j=0$]{\includegraphics[width=6cm, height =4cm]{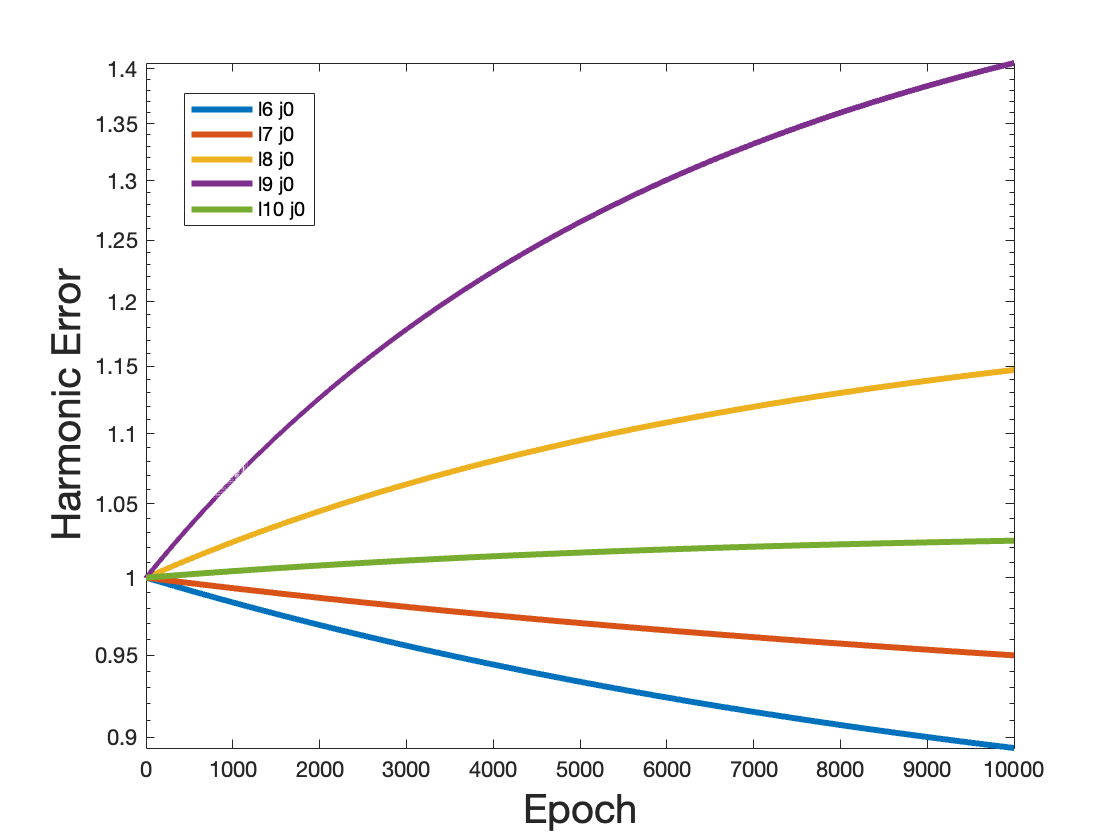}}
    \caption{Harmonic Errors of $c_{lj}$ When $l = 1,2,\ldots, 10$, $j=0$.}
    \label{2testalphabeta}
\end{figure}

\begin{figure}
    \centering
    \subfigure[Image of the target function]{\includegraphics[width=6cm, height =4cm]{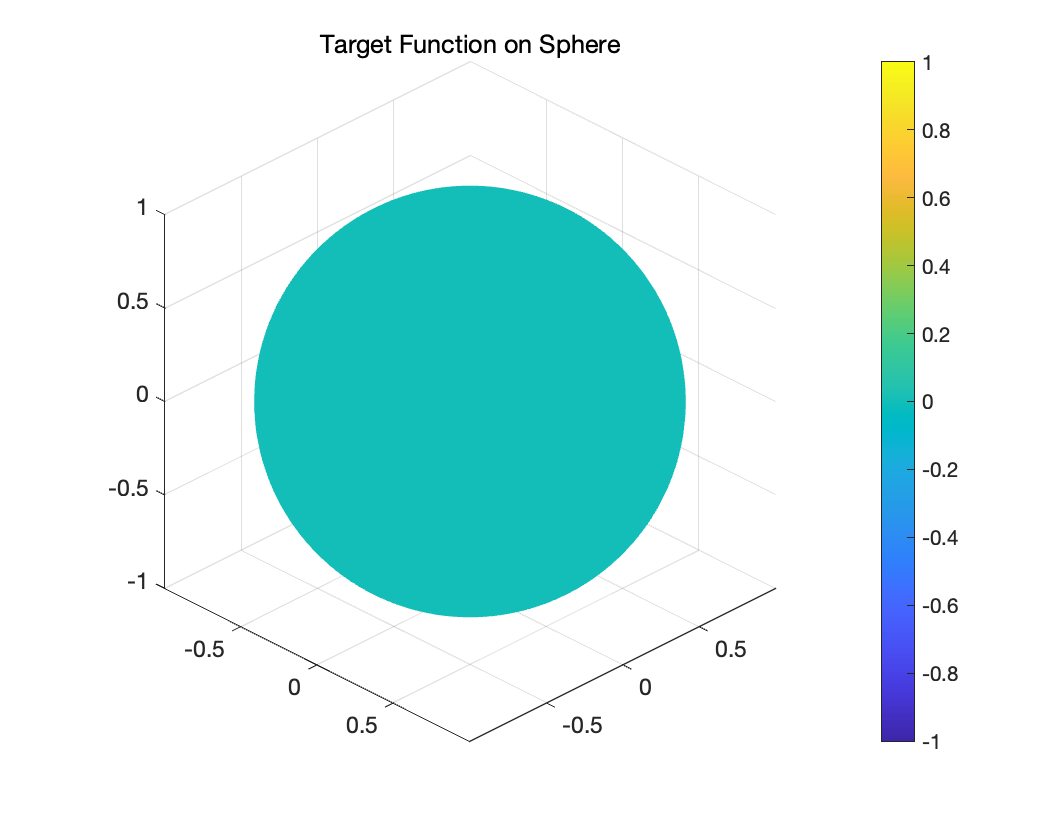}}
    \subfigure[Image of the SNN output]{\includegraphics[width=6cm, height =4cm]{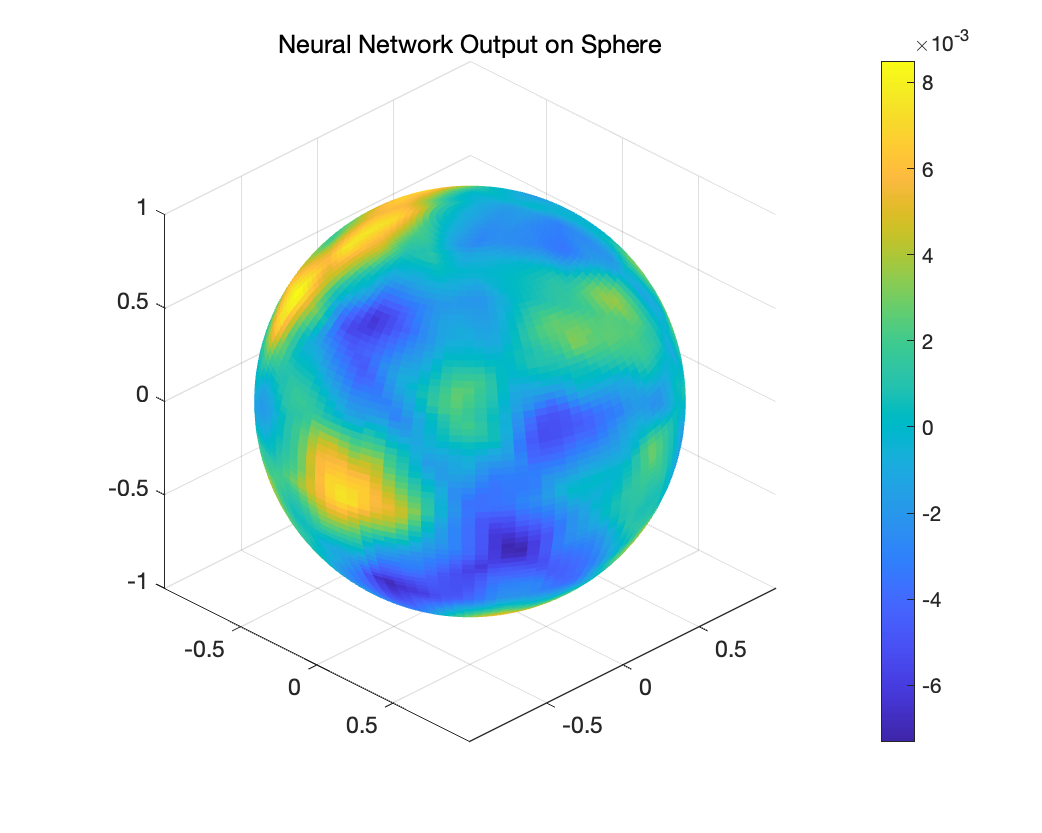}}

   \subfigure[Image of the error between the target function and SNN]{\includegraphics[width=6cm, height =4cm]{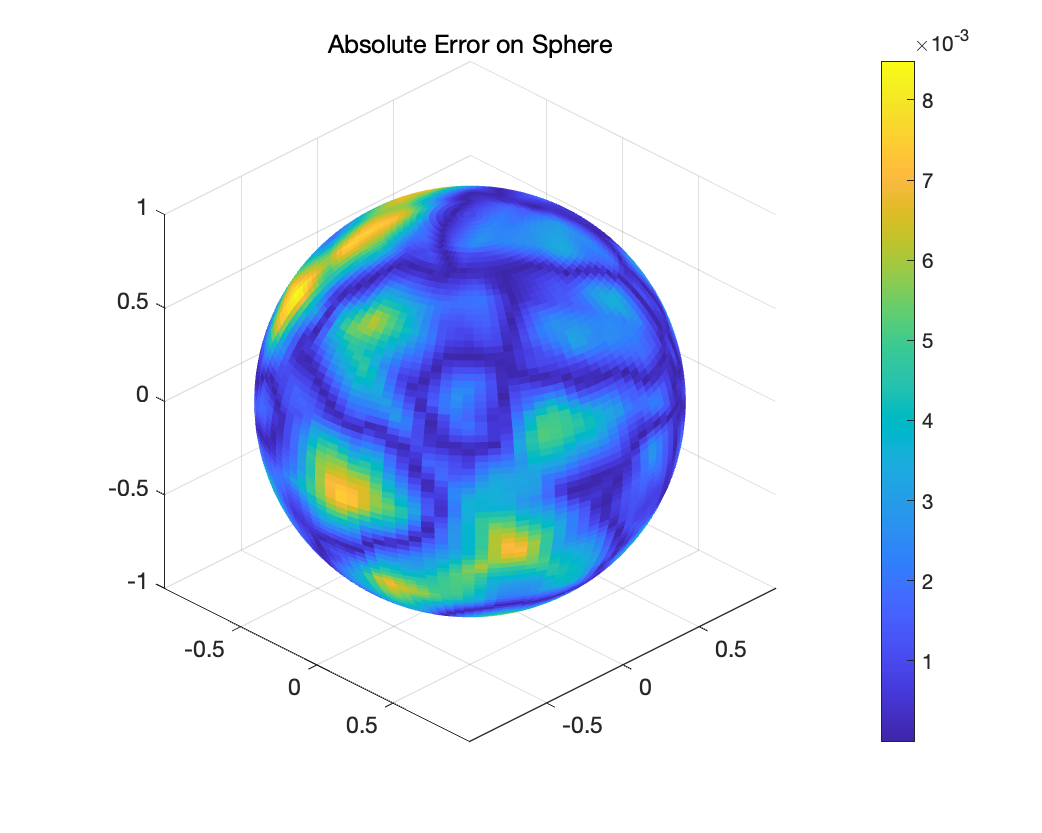}}
    \caption{The target function, SNN output, and error of the second test.}
    \label{2testdnntarget}
\end{figure}

Figure \ref{2test} shows that the loss goes down to 1.59e-5, and Figure \ref{2testalphabeta} shows the harmonic error of the first 10 frequencies of $j=0$. We observe a mixed learning pattern in this intermediate case where the network partially adheres to the frequency principle. While lower-frequency components are generally learned earlier, some higher-frequency components are learned simultaneously or even earlier than certain low-frequency components. This behavior is reflected in the non-monotonic pattern of spherical harmonic coefficient errors across different degrees.

The spatial distribution of the error shows localized regions of higher discrepancy, suggesting that certain frequency components remain challenging for the network to approximate accurately. This partial adherence may be attributed to the specific initialization scheme and the optimization trajectory taken during training.

\textbf{Case 3: Contradiction to Frequency Principle}
\begin{figure}
    \centering
    \includegraphics[width=6cm, height =4cm]{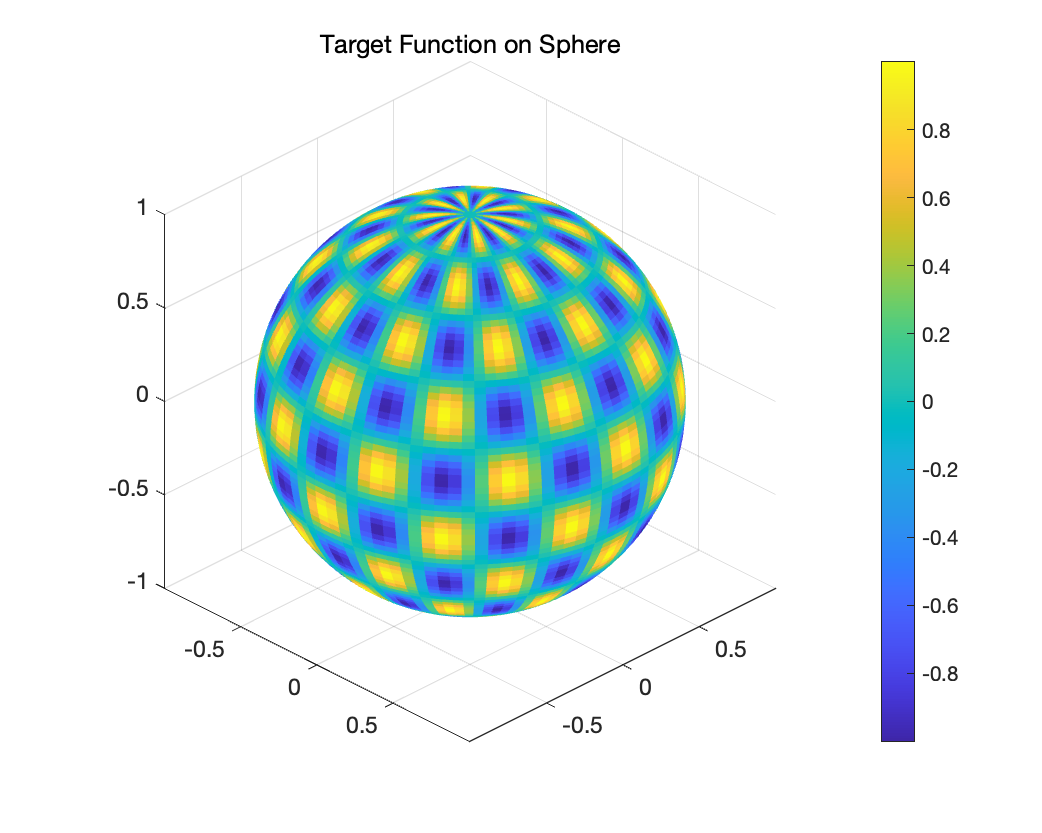}
    \caption{Image of the initialization of the third test.}
    \label{3testhighinitial}
\end{figure}
\begin{figure}
    \centering
    \includegraphics[width=6cm, height =4cm]{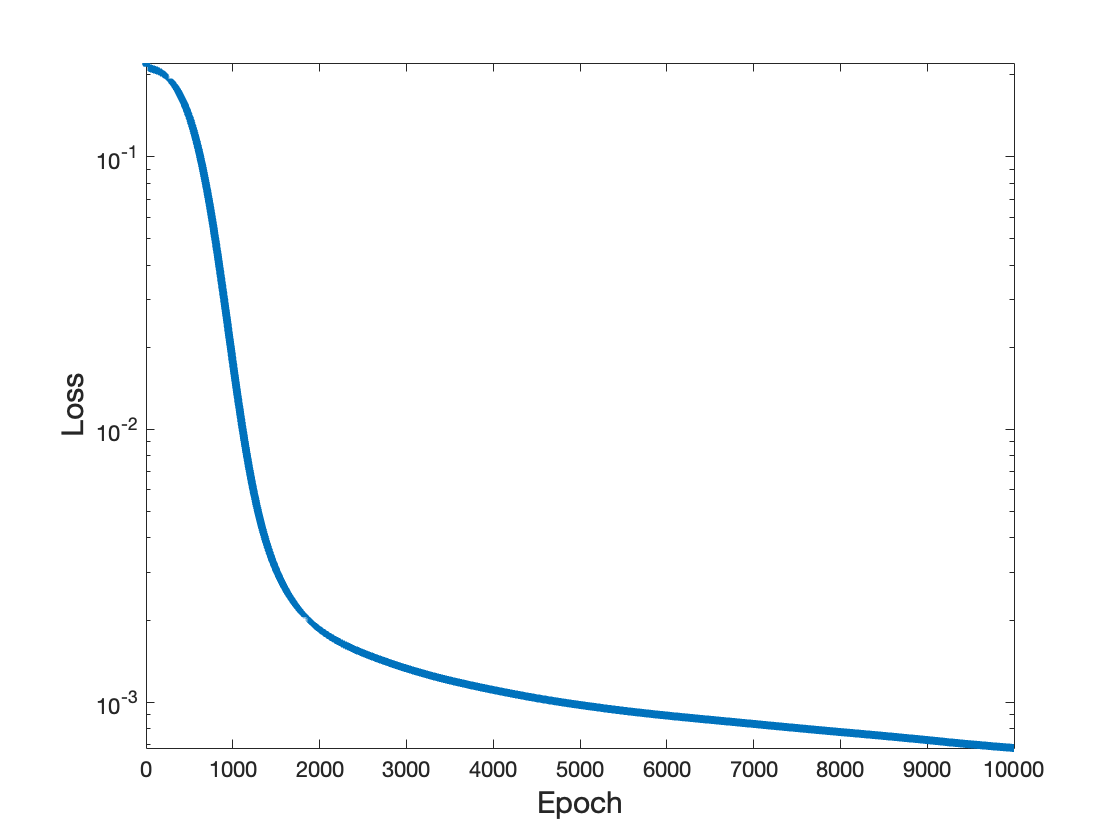}
    \caption{Image of Loss of the third test (Best Loss: 6.79e-4).}
    \label{3test}
\end{figure}
\begin{figure}
    \centering
    \subfigure[$l = 1,2,\ldots, 5$, $j=0$]{\includegraphics[width=6cm, height =4cm]{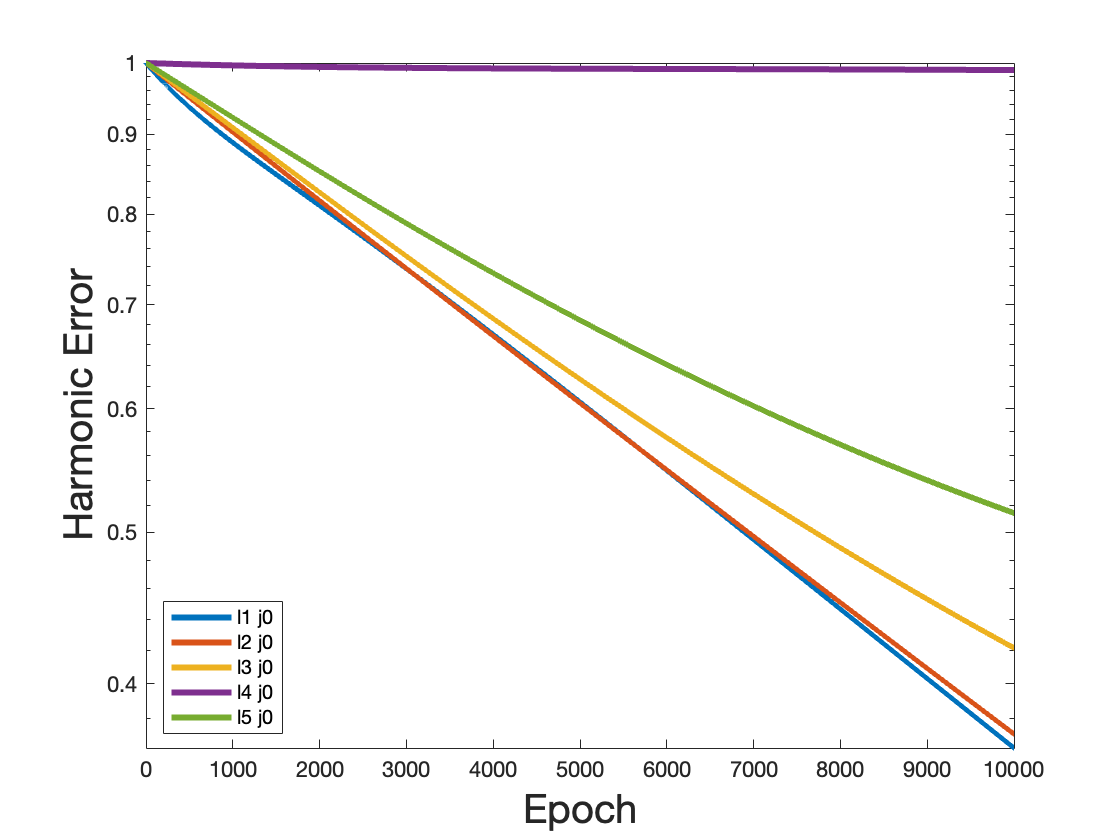}}
    \subfigure[$l = 6,7,\ldots, 10$, $j=0$]{\includegraphics[width=6cm, height =4cm]{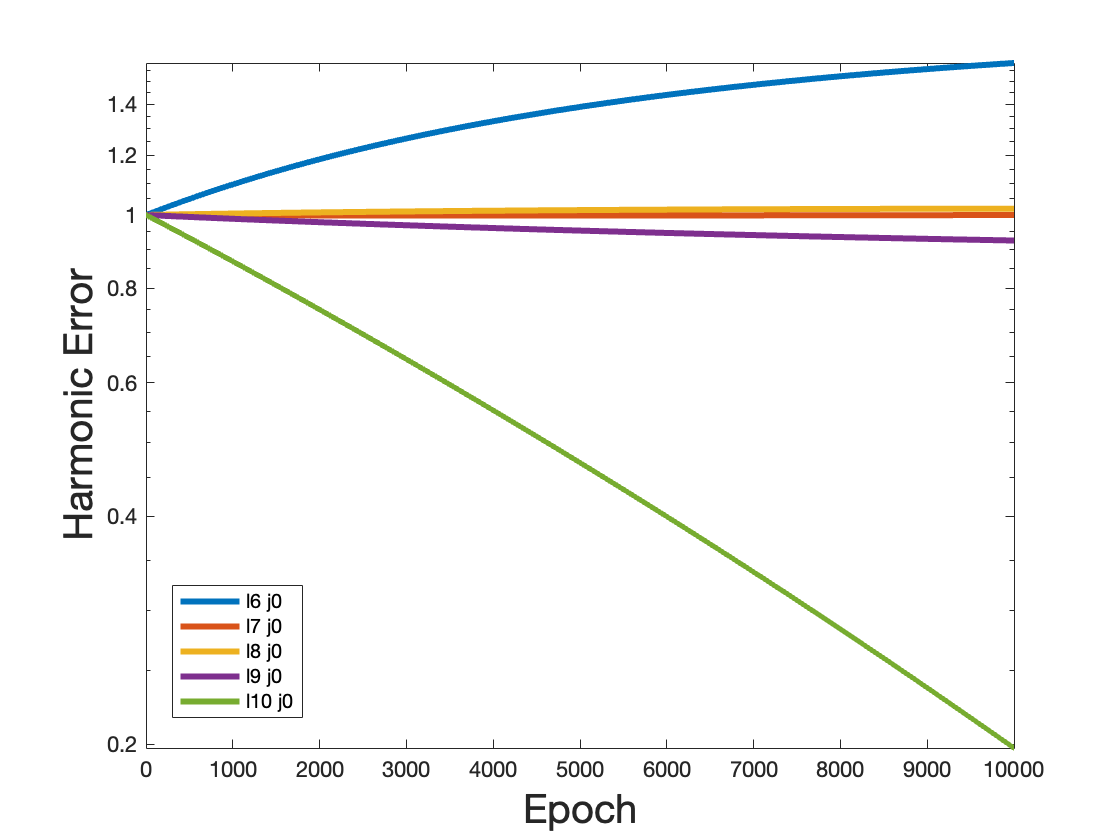}}
    \caption{Harmonic Errors of $c_{lj}$ When $l = 1,2,\ldots, 10$, $j=0$.}
    \label{3testalphabeta}
\end{figure}

\begin{figure}
    \centering
    \subfigure[Image of the target function]{\includegraphics[width=6cm, height =4cm]{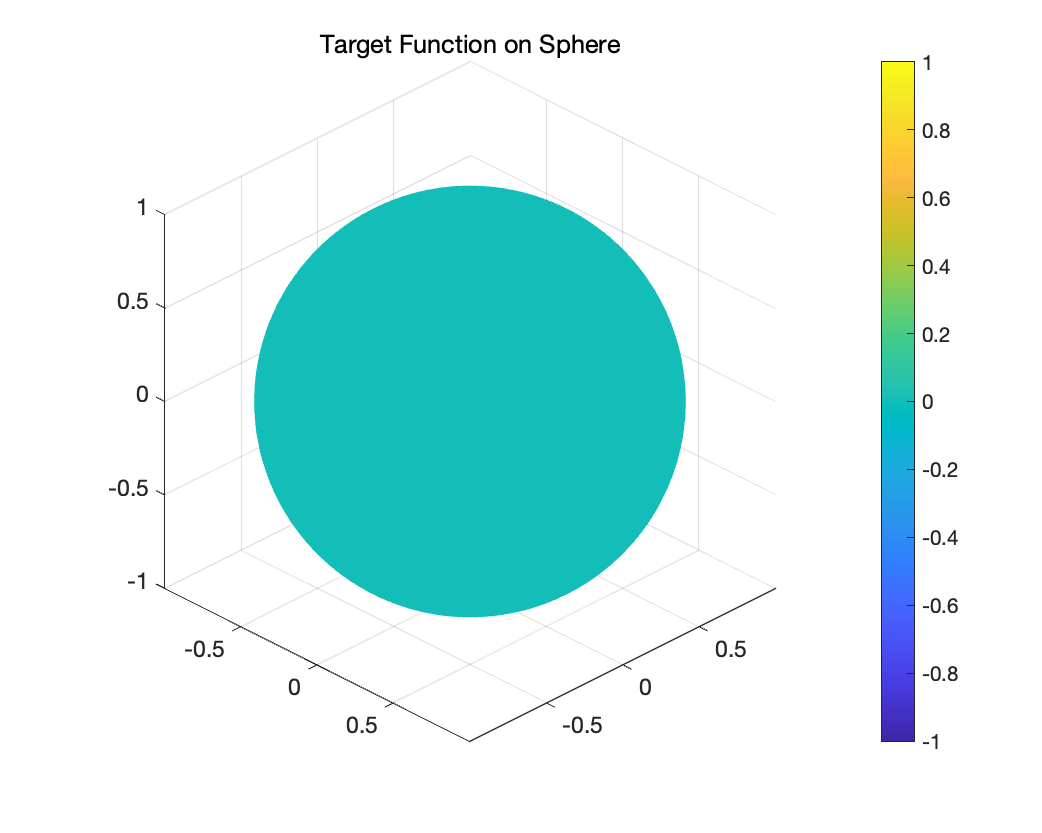}}
    \subfigure[Image of the SNN output]{\includegraphics[width=6cm, height =4cm]{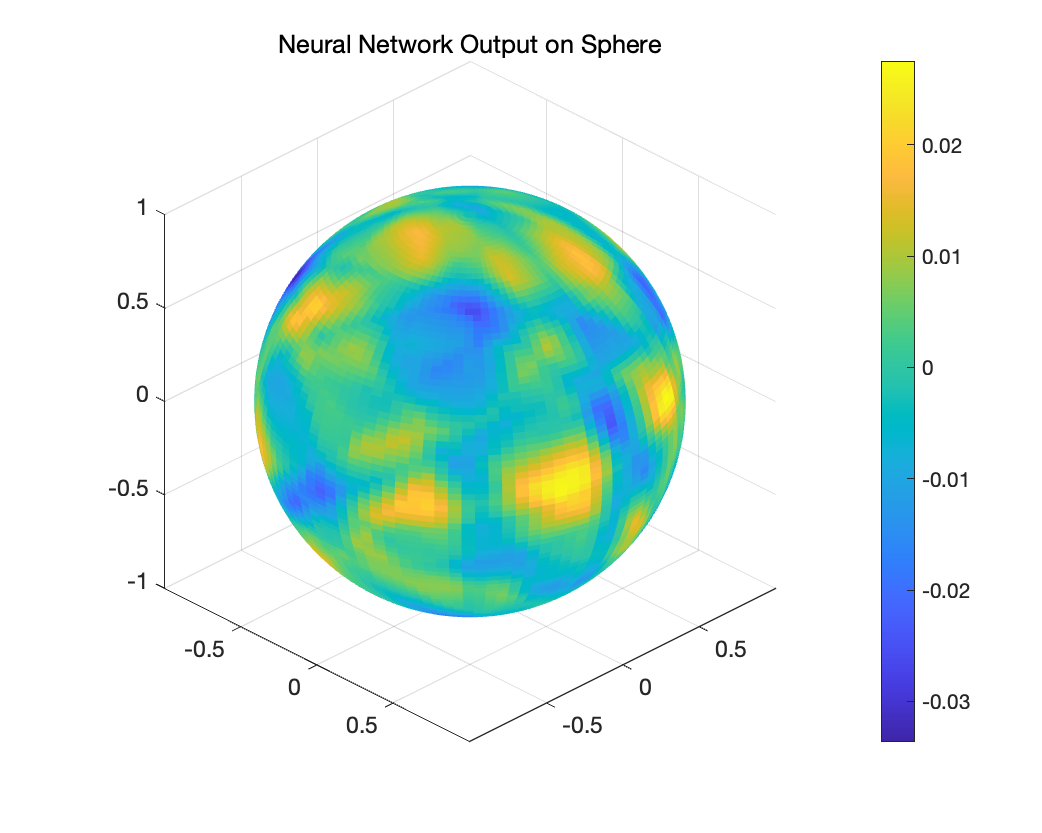}}

   \subfigure[Image of the error between the target function and SNN]{\includegraphics[width=6cm, height =4cm]{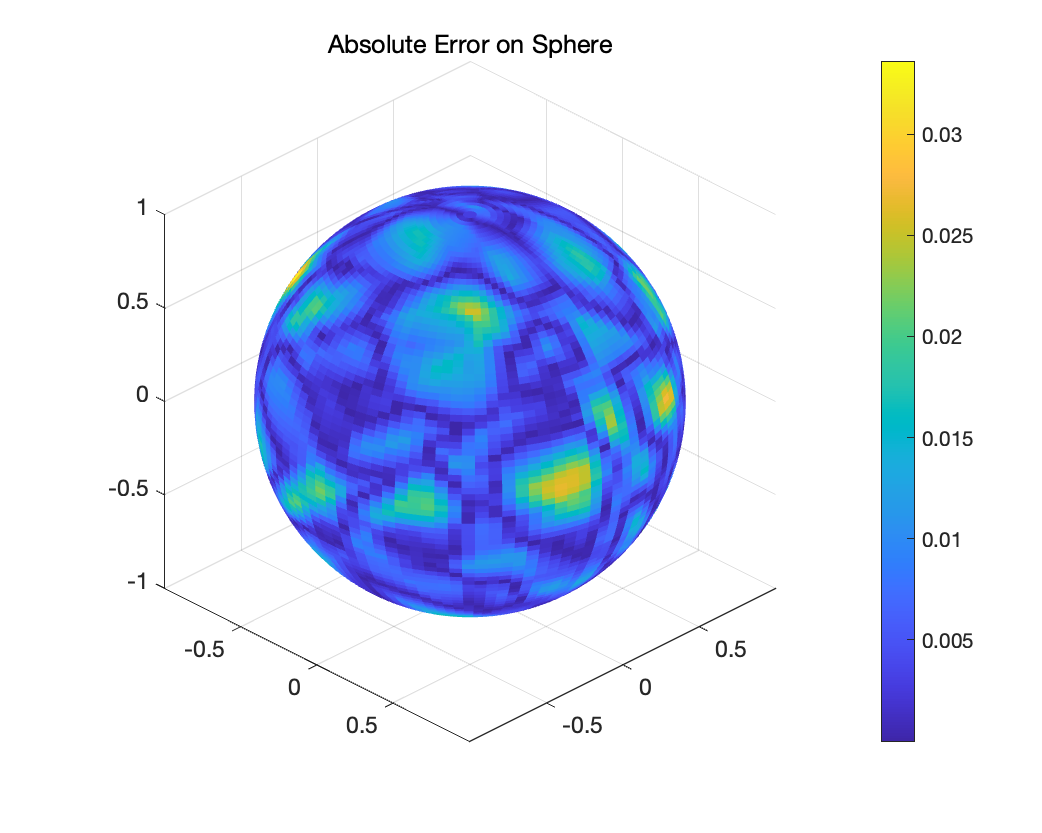}}
    \caption{The target function, SNN output, and error of the third test.}
    \label{3testdnntarget}
\end{figure}

In this case, we give a high-frequency initialization $\sin(10\tau)\cos(10\phi)$ as Figure \ref{3testhighinitial} shown. Then, the network exhibits learning behavior that contradicts the frequency principle. Surprisingly, higher-frequency components are learned before lower-frequency components, as evidenced by the spherical harmonic coefficient error plots. This unexpected behavior demonstrates that while the frequency principle is often observed in neural networks, it is not universal and can be violated under certain conditions.

The error distribution shows significant variations across the sphere, with particularly large errors in regions corresponding to low-frequency components. This case provides valuable insights into the limitations of the frequency principle and suggests that network architecture, initialization, and optimization parameters play crucial roles in determining learning patterns.

\textbf{Comparative Analysis}
The three cases presented above demonstrate the diverse learning behaviors possible in shallow neural networks on the sphere. While the frequency principle often serves as a useful framework for understanding neural network learning dynamics, our experiments show that its applicability can vary significantly. These variations may be attributed to several factors:
\begin{itemize}
    \item The choice of network initialization, which influences the initial distribution of frequency components in the network's output
    \item The optimization trajectory determined by the SGD algorithm and learning rate
    \item  The specific architecture of the shallow network and its capacity to represent different frequency components
\end{itemize}

These findings contribute to our understanding of neural network learning dynamics on manifolds and suggest that the frequency principle should be considered as a general tendency rather than a universal law.

\subsubsection{$u(\tau,\phi) = \sin(\tau)\cos(3\phi) + \sin(3\tau)\cos(5\phi)$}

We examine a more complex target function defined on the unit sphere $\mathbb{S}^2$ that combines multiple frequency components. The function $u(\tau,\phi) = \sin(\tau)\cos(3\phi) + \sin(3\tau)\cos(5\phi)$ incorporates both low and high-frequency terms in both $\tau$ and $\phi$ directions. The neural network architecture remains consistent with our previous analysis, employing 100 neurons in the hidden layer with ReLU activation functions. Training is conducted using SGD with a learning rate of $1e{-3}$ over 100,000 epochs on 100 uniformly distributed sampling points.

\textbf{Case 1: Partial Adherence to Frequency Principle}
\begin{figure}
    \centering
    \includegraphics[width=6cm, height=4cm]{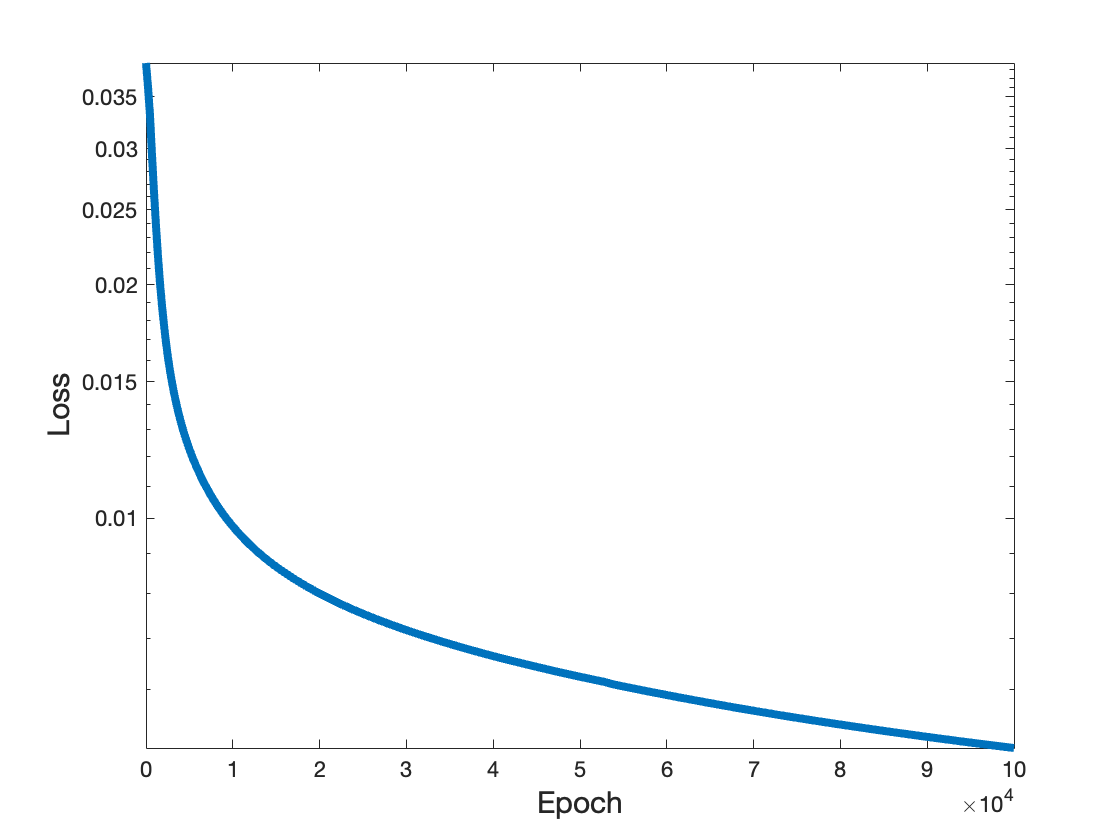}
    \caption{Image of Loss of the fourth test (Best Loss: 5.7e-3).}
    \label{trig1test}
\end{figure}

\begin{figure}
    \centering
    \subfigure[$l = 1,3,5$, $j=0$]{\includegraphics[width=6cm, height=4cm]{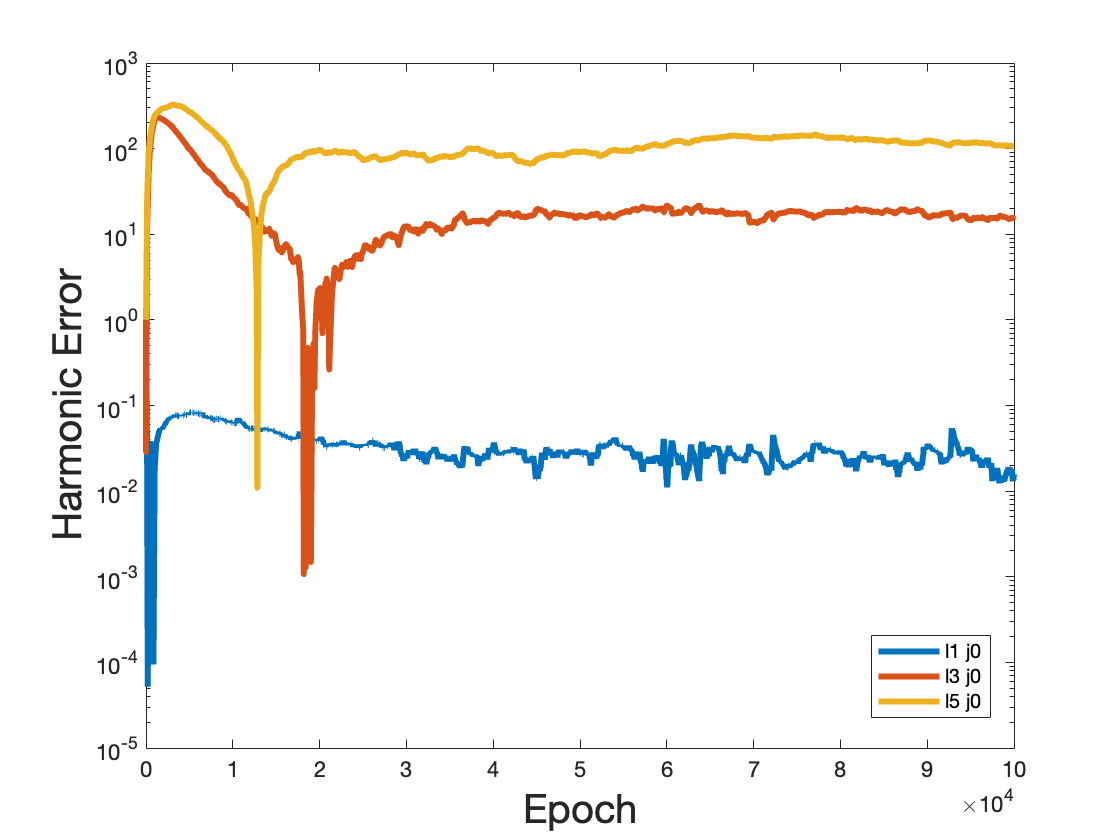}}
    
    \subfigure[$l = 1,2,\ldots,5$, $j=0$]{\includegraphics[width=6cm, height=4cm]{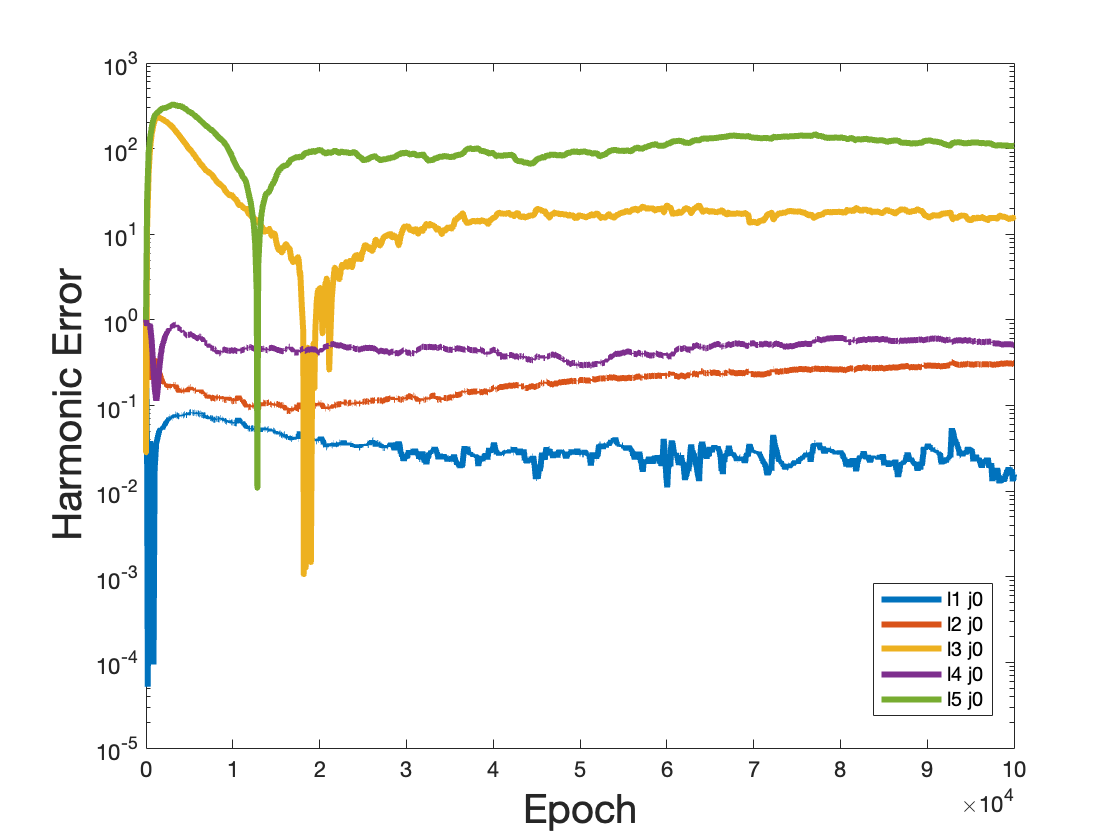}}
    \subfigure[$l = 6,7,\ldots,10$, $j=0$]{\includegraphics[width=6cm, height=4cm]{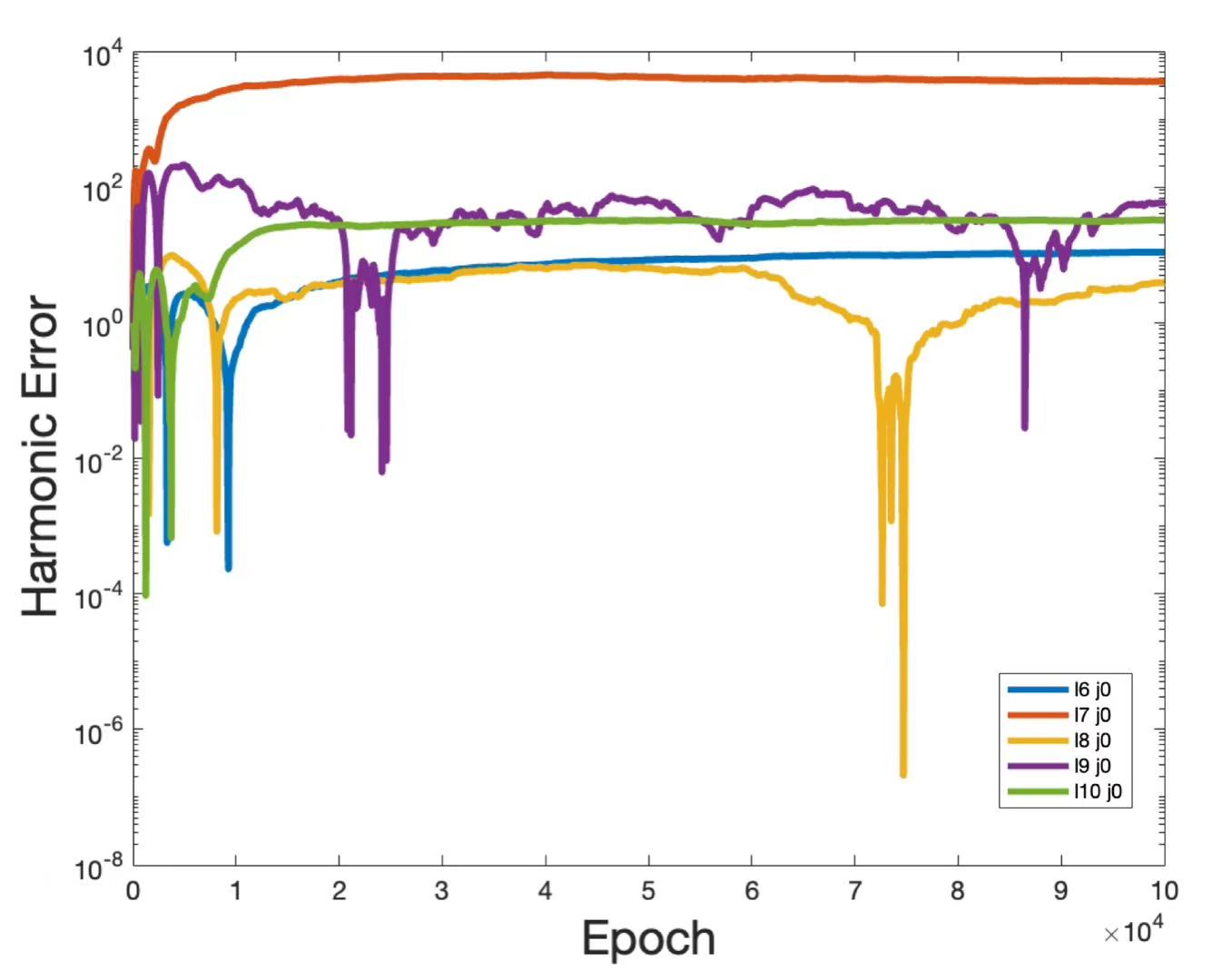}}
    \caption{Harmonic Errors showing mixed learning patterns.}
    \label{trig1testharmonic}
\end{figure}

\begin{figure}
    \centering
    \subfigure[Target function]{\includegraphics[width=6cm, height=4cm]{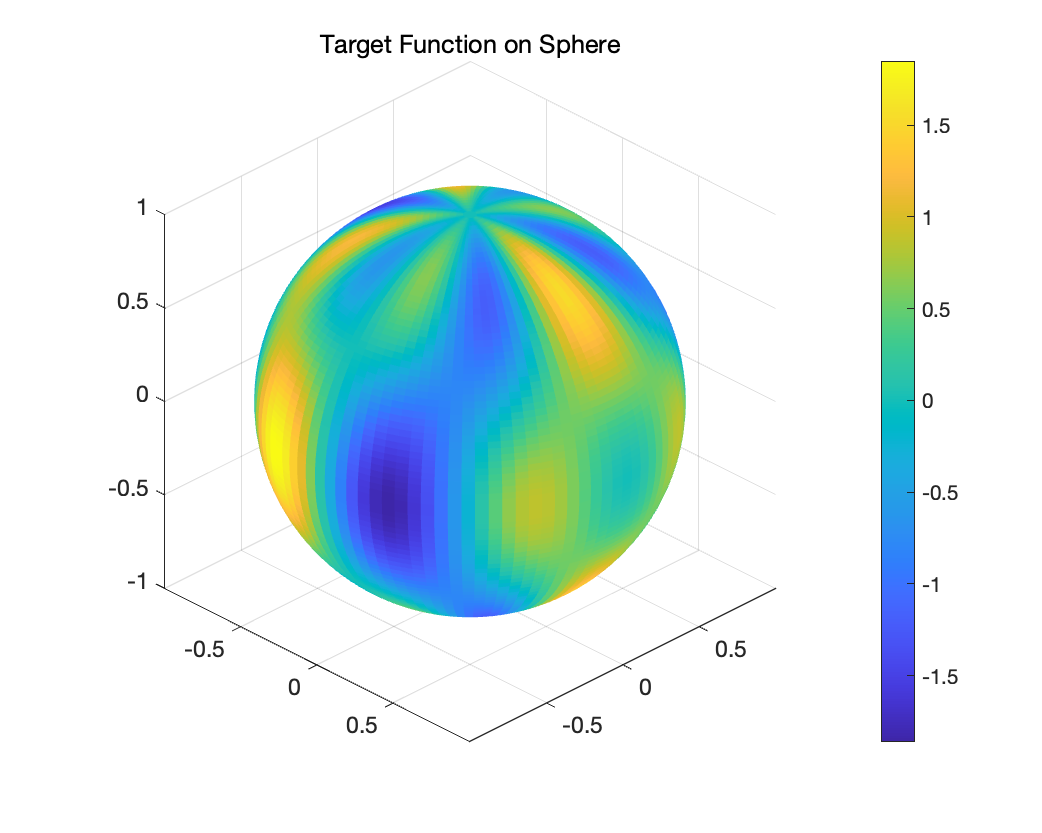}}
    \subfigure[Neural network output]{\includegraphics[width=6cm, height=4cm]{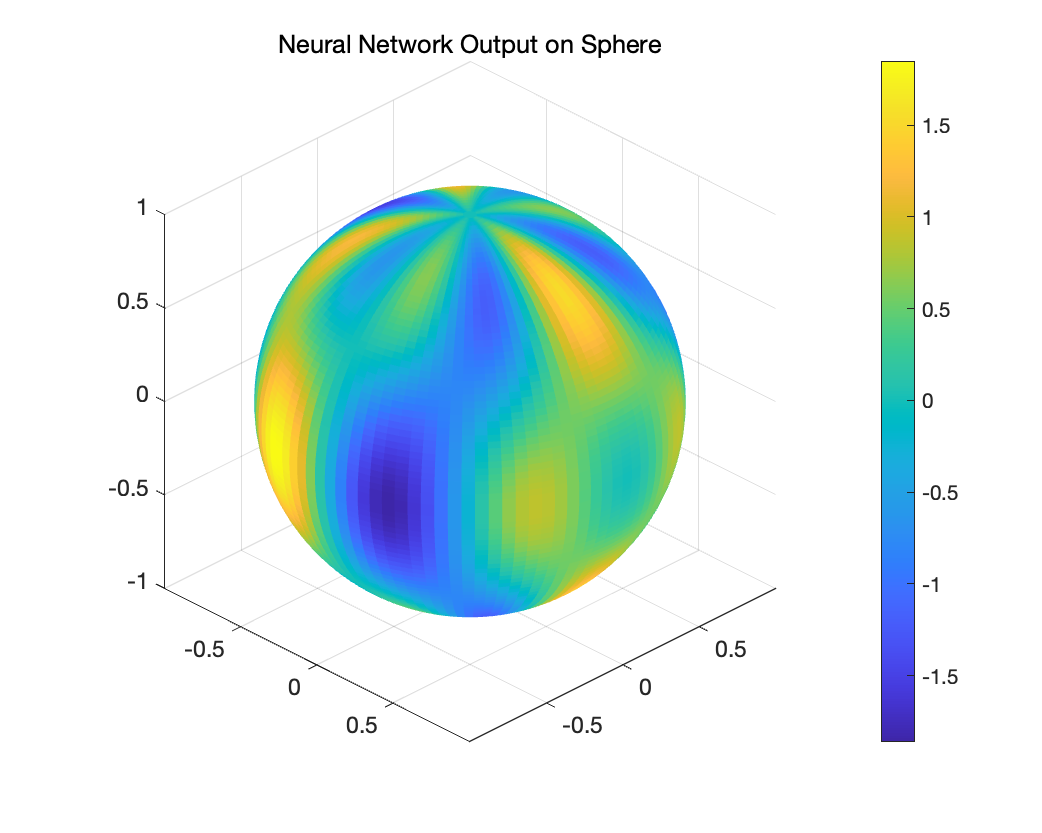}}
    \subfigure[Absolute error distribution]{\includegraphics[width=6cm, height=4cm]{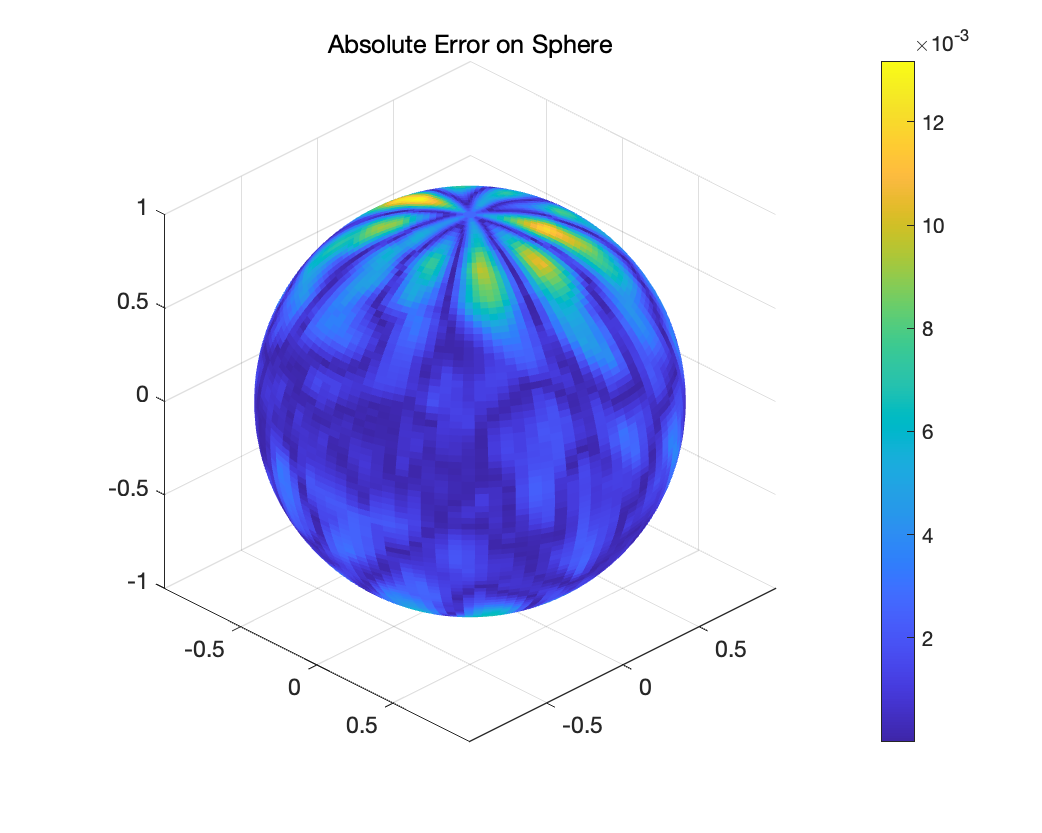}}
    \caption{Spatial comparison revealing partial frequency learning.}
    \label{trig1testvisual}
\end{figure}

The network exhibits complex learning dynamics with a final loss of 5.7e-3 (Figure \ref{trig1test}). The harmonic error analysis in Figure \ref{trig1testharmonic} reveals simultaneous learning of multiple frequency components, particularly evident in the coupling between the $\sin(\tau)$ and $\cos(3\phi)$ terms. This behavior suggests that the natural coupling of frequency components in the target function influences the network's learning trajectory.

The spatial distribution of errors (Figure \ref{trig1testvisual}) shows localized regions of a higher discrepancy, particularly where the frequency components of $\tau$ and $\phi$ interact strongly. This pattern indicates that the network struggles to fully decouple the mixed-frequency components during the learning process, leading to a partial adherence to the frequency principle.

\textbf{Case 2: Contradiction to Frequency Principle}
\begin{figure}
    \centering
    \includegraphics[width=6cm, height=4cm]{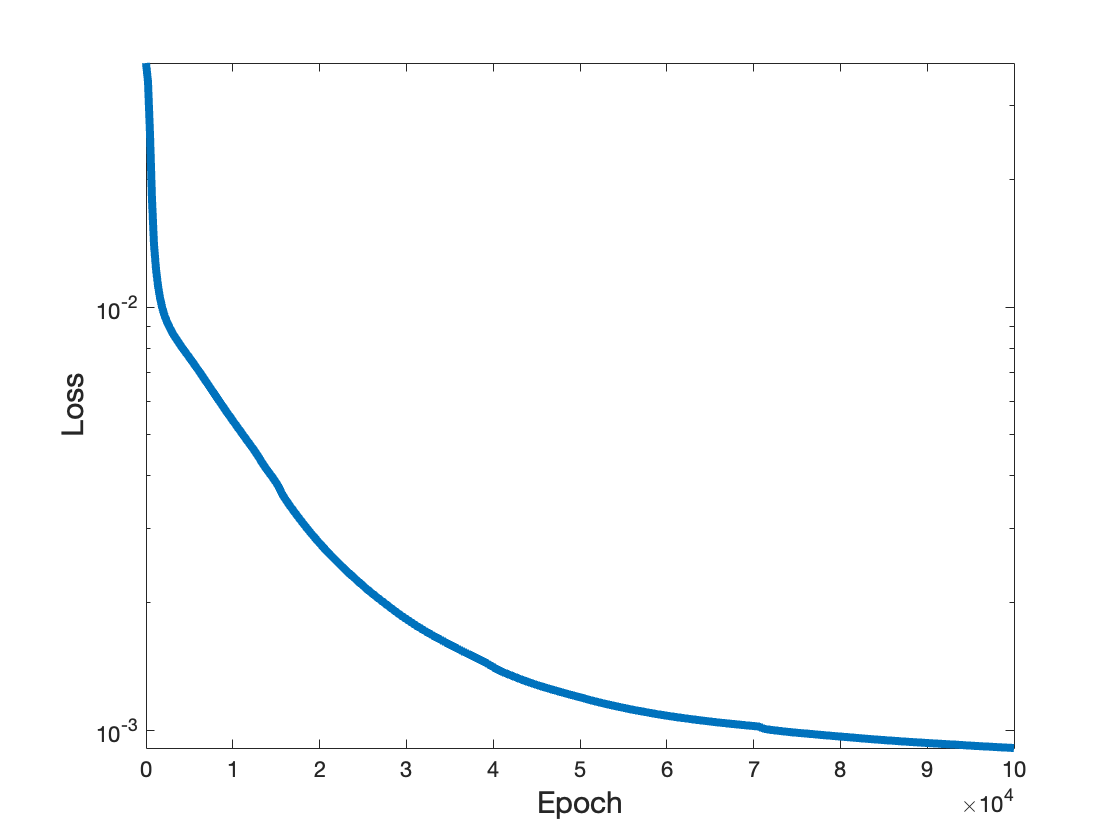}
    \caption{Image of Loss of the fifth test (Best Loss: 9.1e-4).}
    \label{trig3test}
\end{figure}

\begin{figure}
    \centering
    \subfigure[$l = 1,3,5$, $j=0$]{\includegraphics[width=6cm, height=4cm]{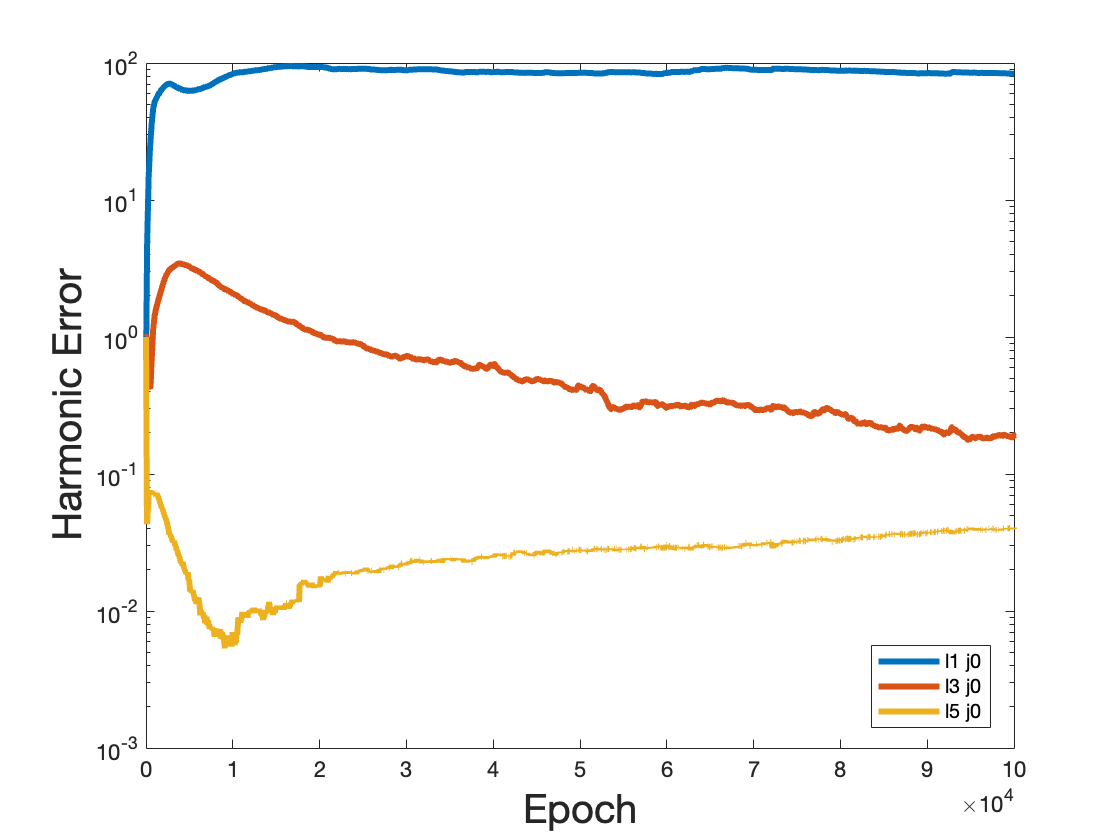}}
    
    \subfigure[$l = 1,2,\ldots,5$, $j=0$]
    {\includegraphics[width=6cm, height=4cm]{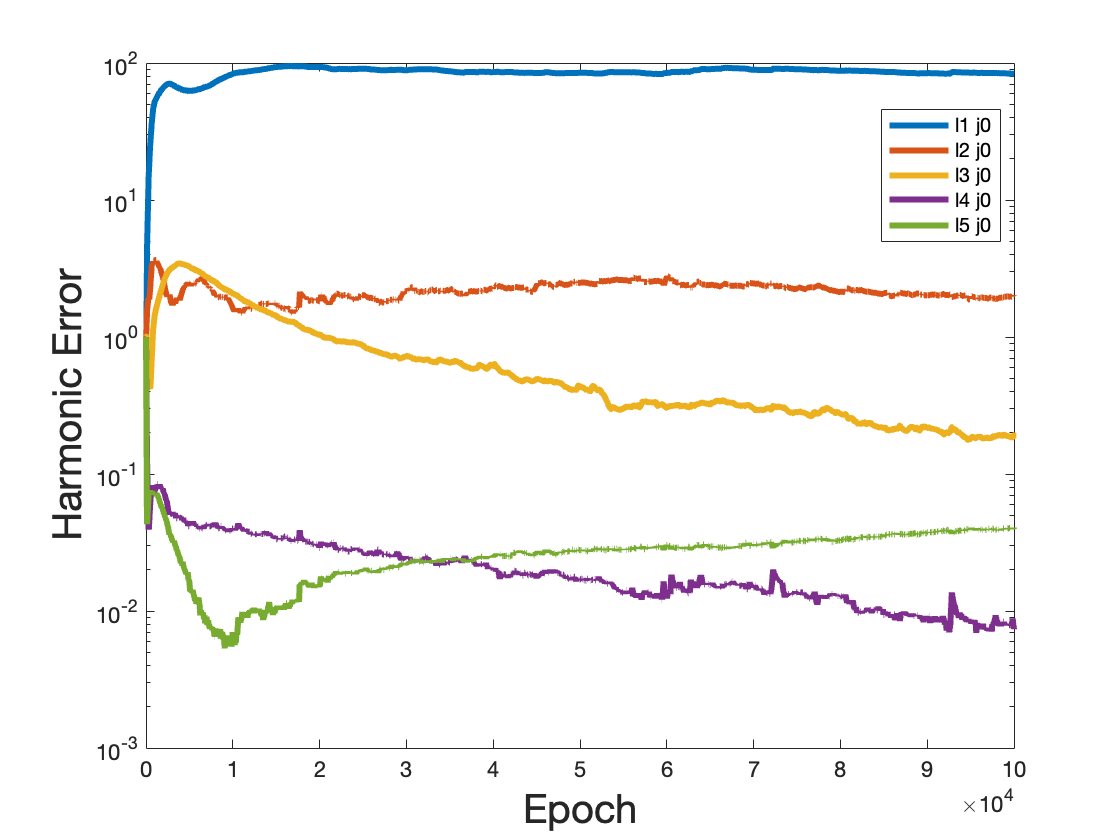}}
    \subfigure[$l = 6,7,\ldots,10$, $j=0$]{\includegraphics[width=6cm, height=4cm]{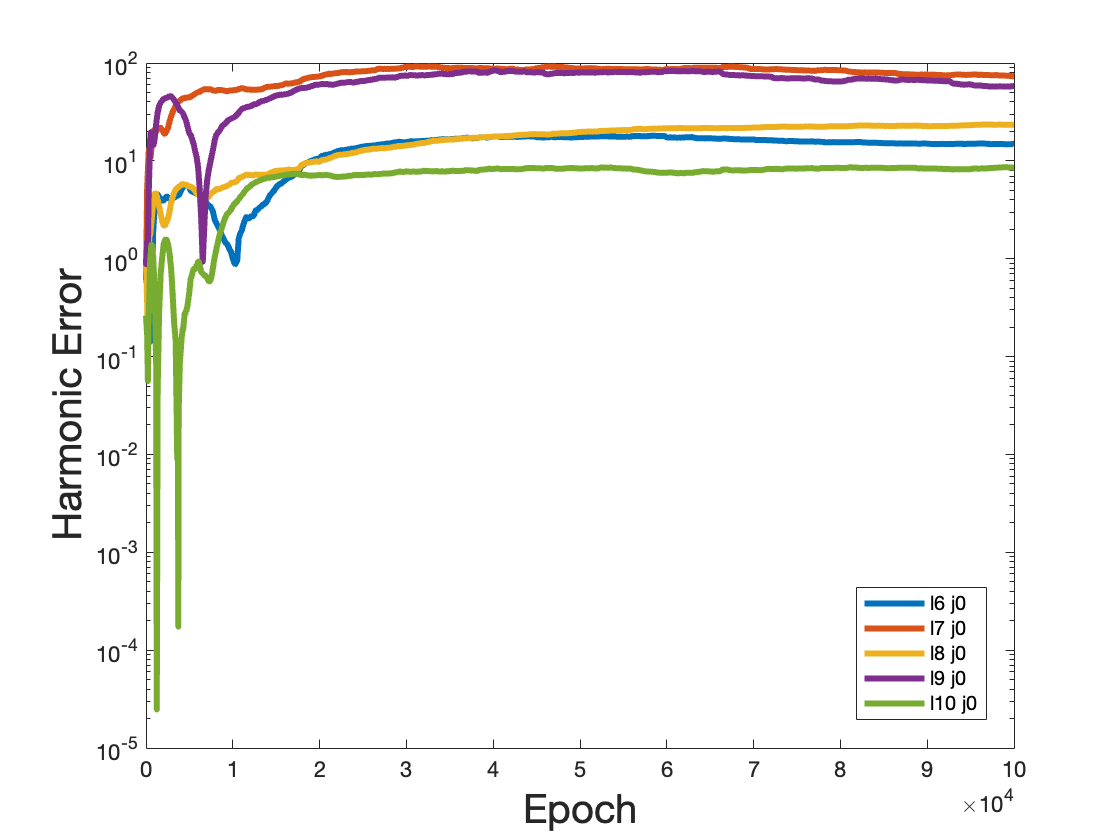}}
    \caption{Harmonic Errors demonstrating inverse frequency learning}
    \label{trig3testharmonic}
\end{figure}

\begin{figure}
    \centering
    \subfigure[Target function]{\includegraphics[width=6cm, height=4cm]{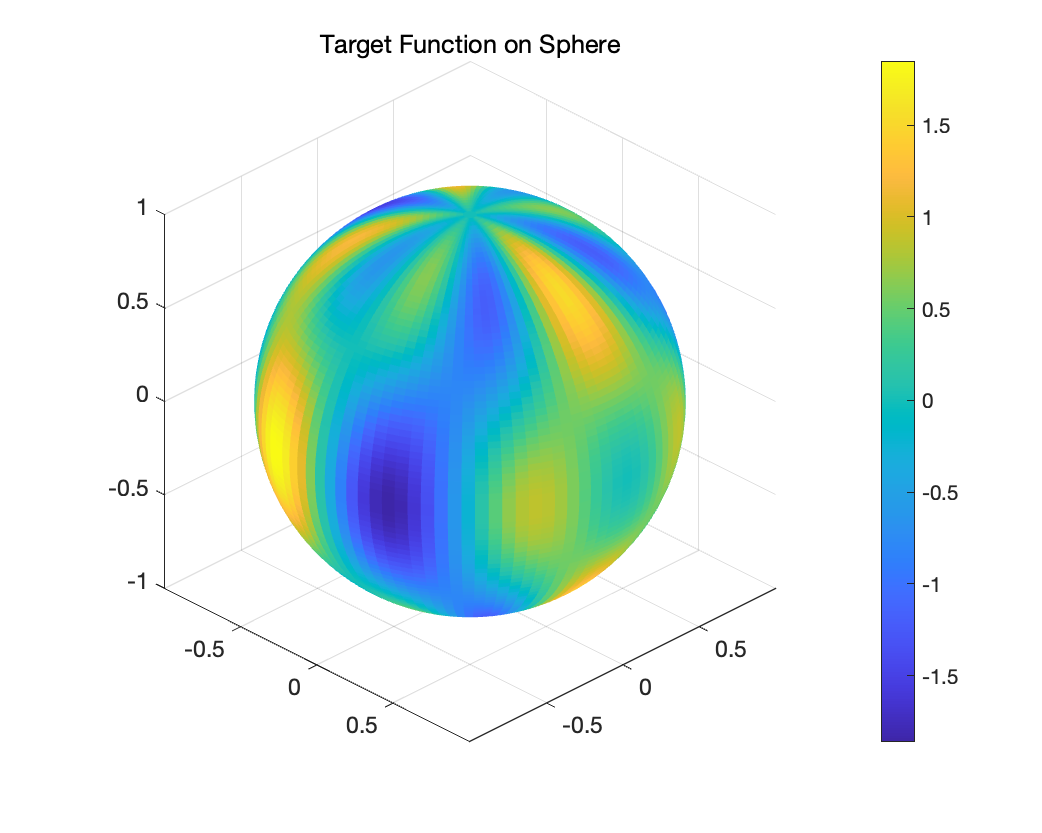}}
    \subfigure[Neural network output]{\includegraphics[width=6cm, height=4cm]{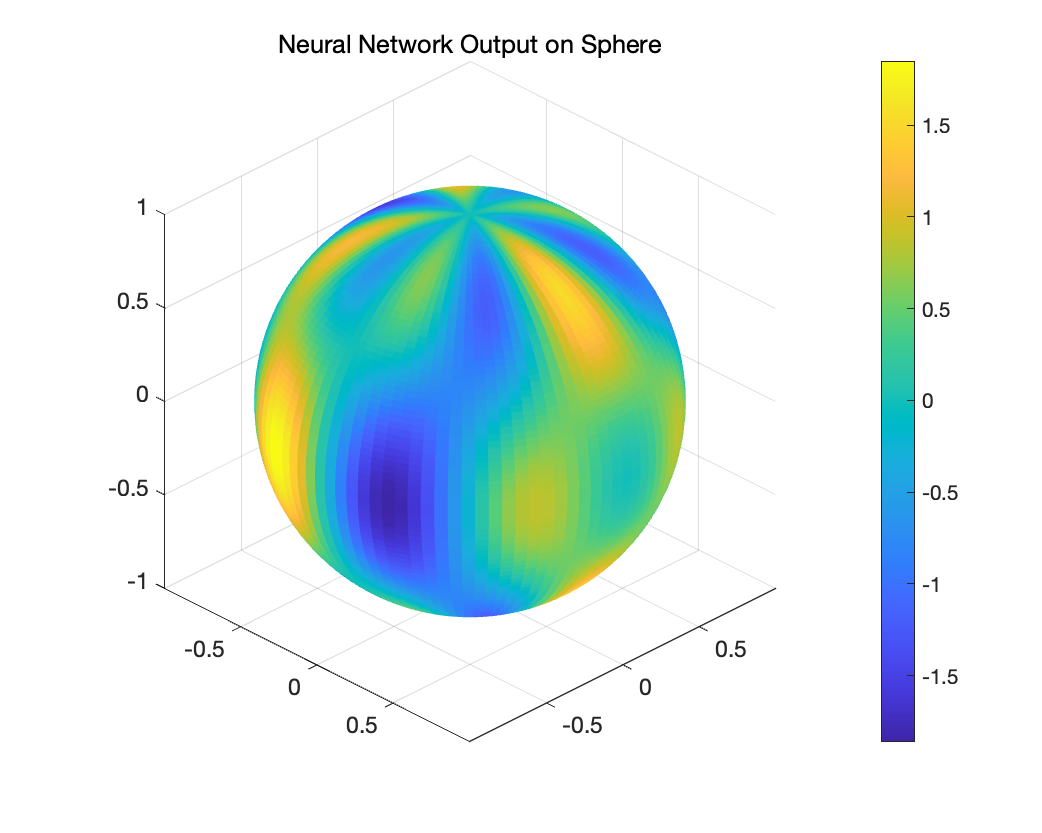}}
    \subfigure[Absolute error distribution]{\includegraphics[width=6cm, height=4cm]{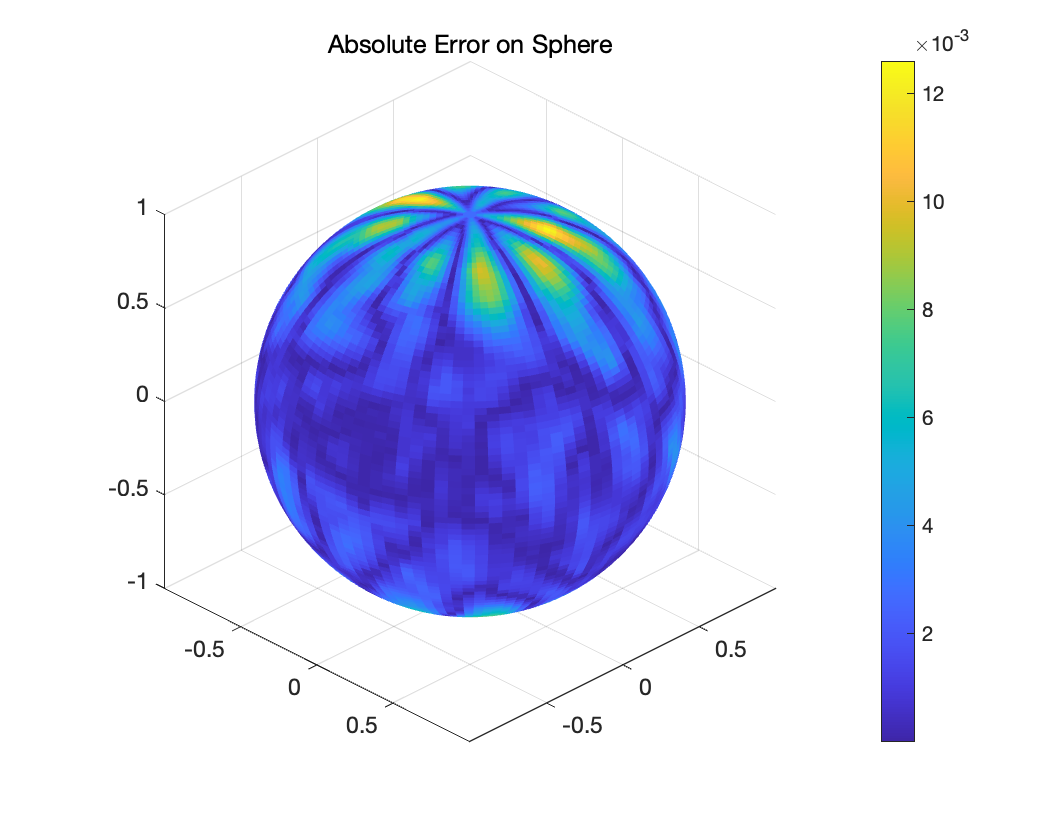}}
    \caption{Spatial comparison showing inverse frequency learning patterns}
    \label{trig3testvisual}
\end{figure}

In this case, we also give a high-frequency initialization $\sin(10\tau)\cos(10\phi)$ as Figure \ref{3testhighinitial} shown and we observe a striking contradiction to the frequency principle, with the network achieving a final loss of 9.1e-4 (Figure \ref{trig3test}). The harmonic error analysis (Figure \ref{trig3testharmonic}) reveals that higher-frequency components, particularly those associated with $\sin(3\tau)$ and $\cos(5\phi)$, are learned before their lower-frequency counterparts. This inverse learning pattern persists throughout the training process, demonstrating a clear violation of the traditional frequency principle.

The spatial error distribution (Figure \ref{trig3testvisual}) exhibits a distinctive pattern where errors are concentrated in regions dominated by low-frequency components, while high-frequency features are captured with surprising accuracy. This behavior suggests that the network's initialization and optimization trajectory have fundamentally altered the usual frequency-dependent learning dynamics.

\textbf{Comparative Analysis}
The analysis of this trigonometric target function yields several significant insights:
\begin{itemize}
    \item The coupling of different frequency components in $\tau$ and $\phi$ introduces complex learning dynamics that can disrupt the traditional frequency principle
    \item Network initialization plays a crucial role in determining whether the learning process follows, partially adheres to, or contradicts the frequency principle
    \item The presence of mixed frequencies in the target function may create local optima in the loss landscape that influence the learning trajectory
\end{itemize}

These findings extend our understanding of neural network learning dynamics in the sphere, particularly for functions with coupled frequency components. The observed behaviors suggest that the frequency principle, while often observed in simpler cases, may not fully capture the learning dynamics of neural networks when approximating functions with intricate frequency interactions. This observation has important implications for both theoretical analysis and practical applications of neural networks on spherical domains.

\section{Frequency Analysis on the Unit Sphere with Trained Weights}
\label{Sec:SphericalAnalysisTW}

In this section, we analyze the frequency behavior of shallow neural networks (SNNs) defined on the unit sphere \( S^2 \) when both the weights \( w_i \) and the coefficients \( a_i \) can change during training. As before, the unit sphere provides a natural domain for spherical harmonics, allowing us to investigate the frequency components of ReLU-activated functions without the complications introduced by periodic extensions in one-dimensional analyses.

We also consider SNNs of the form \eqref{neuralnetwork-sphere}, but now \(\theta = ((a_i)_{i=1}^m, (w_i)_{i=1}^m)\) represents both the output weights \( a_i \in \mathbb{R} \) and the direction vectors \( w_i \in S^2 \). Unlike the fixed-weight scenario, here we allow \( w_i \) to evolve during training. This introduces a new layer of complexity, as changes in \( w_i \) can alter the spherical harmonic representation dynamically.

\subsection{Expansion of a Single Neuron with Variable Direction}
\label{subsec:NeuronExpansionVariableDir}

We now turn to the scenario in which the direction \(w_i \in S^2\) of each ReLU neuron 
in \eqref{nnrepresent-sphere} \emph{changes} during training.  In this case, the local axis of symmetry 
for each neuron evolves over time.  Below, we first recap the axisymmetric expansion in the \emph{local} 
coordinates aligned with \(w_i\), and then relate it to the \emph{global} coordinates \((\tau,\phi)\).

Recall from Section~\ref{Sec:SphericalAnalysisFW} that if \(w_i\) is fixed and aligned with the polar axis in 
some local coordinate frame, then for a point \(\tilde{x}(\tau',\phi)\) on the sphere (with 
\(\tau'\in[0,\pi]\), \(\phi\in[0,2\pi)\)), we have
\[
w_i^\top \tilde{x}
\;=\;
\cos(\tau'),
\quad
\mathrm{ReLU}(w_i^\top \tilde{x})
\;=\;
\max\{0,\cos(\tau')\}.
\]

Because \(\max\{0,\cos(\tau')\}\) depends only on \(\tau'\) (and is independent of the azimuth \(\phi\)), 
its spherical harmonic expansion in that \emph{local} frame involves only the \(j=0\) modes:
\[
\mathrm{ReLU}(w_i^\top \tilde{x})
\;=\;
\sum_{\ell=0}^{\infty}
c_{\ell}\,
Y_\ell^0(\tau',\phi),
\]
where \(\bigl\{c_{\ell}\bigr\}\) is given explicitly by
\[
c_{\ell}
\;=\;
\int_{S^2}
\max\!\bigl(0,\cos(\tau')\bigr)
\,\overline{Y_\ell^0(\tau',\phi)}
\;d\Omega(\tau',\phi).
\]

Carrying out this integral over \(\tau'\in[0,\tfrac{\pi}{2}]\) and \(\phi\in[0,2\pi]\) 
results in the known closed-form expressions:
\[
c_{\ell}
=
\begin{cases}
\displaystyle
\frac{\pi}{2}, & \ell=0,\\[0.75em]
\displaystyle
\frac{\sqrt{3\pi}}{3}, & \ell=1,\\[0.75em]
\displaystyle
\frac{\sqrt{\pi}}{24}\,\bigl(\tfrac12\bigr)^\ell\,\sqrt{2\ell+1}\,\bigl(\ell^2 + 3\ell +2\bigr),
& \ell\ge2.
\end{cases}
\]

In particular, \(\bigl|c_{\ell}\bigr|\approx O(\ell^{5/2}/2^\ell)\) for large \(\ell\), indicating that 
\(\mathrm{ReLU}\bigl(\cos(\tau')\bigr)\) is dominated by low-frequency modes in the local frame.

When \(w_i\) \emph{varies} during training, we cannot assume a single, static local frame.  Instead, we fix a 
\emph{global} coordinate system on the sphere, described by \((\tau,\phi)\) with \eqref{spherecordinate}. If $w_i$ points in an arbitrary direction, define $R_{w_i}$ to be the rotation (in $\mathbb{R}^3$) that 
takes the north pole $(0,0,1)$ to $w_i$.  In practice, if the direction \( w_i \) evolves during training, we cannot assume a single, fixed coordinate system. Instead, let \(R_{w_i}\) be the rotation that takes the global frame's north pole to align with \(w_i\). For \(w_i = (w_x, w_y, w_z)\) on the unit sphere \(S^2\), we can explicitly construct \(R_{w_i}\) using Rodrigues' rotation formula:
\[
R_{w_i} = I + \sin\theta K + (1-\cos\theta)K^2,
\]
where \(\theta = \arccos(w_z)\) is the angle between \(w_i\) and the north pole \((0,0,1)\), and \(K\) is the skew-symmetric matrix formed from the normalized axis of rotation \(k = (-w_y, w_x, 0)/\sqrt{w_x^2 + w_y^2}\):
\[
K = \begin{pmatrix} 
0 & -k_z & k_y \\
k_z & 0 & -k_x \\
-k_y & k_x & 0
\end{pmatrix}.
\]

Then any point $x(\tau,\phi)$ can be mapped into the local 
axis-aligned frame via
\[
\tilde{x}(\tau,\phi) 
\;=\;
R_{w_i}^{-1}\,x(\tau,\phi).
\]

In that local frame, $\mathrm{ReLU}\bigl(w_i^\top \tilde{x}\bigr)$ has the pure $j=0$ expansion 
discussed above.  However, from the \emph{global} viewpoint, the rotation $R_{w_i}$ \emph{mixes} the 
spherical harmonics within each degree $\ell$.  Concretely, via the Wigner $D$-matrices, we have
\[
Y_{\ell}^0\bigl(\tilde{x}(\tau,\phi)\bigr)
\;=\;
\sum_{j=-\ell}^{\ell}
D_{j0}^{(\ell)}\bigl(R_{w_i}\bigr)
\,Y_{\ell}^j\bigl(x(\tau,\phi)\bigr),
\]
so we have
\begin{equation}\label{reluglobalform}
    \mathrm{ReLU}\bigl(w_i^\top x(\tau,\phi)\bigr)
\;=\;
\sum_{\ell=0}^{\infty}\sum_{j=-\ell}^{\ell}
\Bigl[
c_{\ell}
D_{j0}^{(\ell)}\!\bigl(R_{w_i}\bigr)
\,Y_{\ell}^j(\tau,\phi)
\Bigr].
\end{equation}

The training implications are threefold. \textbf{First, invariant local structure:} In the coordinate system where $w_i$ defines the polar axis, the activation function $\mathrm{ReLU}\!\bigl(w_i^\top x\bigr)$ exhibits axial symmetry, yielding a spherical harmonic expansion containing only $j=0$ terms. The expansion coefficients $c_{\ell}$ are determined purely by the geometry of $\max\{0,\cos(\tau')\}$ and remain independent of the specific orientation of $w_i$. \textbf{Second, global mode coupling:} When expressed in fixed global coordinates $(\tau,\phi)$, the local $j=0$ mode couples to multiple global angular momentum states $\{Y_{\ell}^j\}$. As $w_i$ evolves during training, the corresponding rotation matrix $R_{w_i}$ varies, causing the energy distribution across global $j$ indices to change dynamically. \textbf{Third, dynamic frequency control:} Through optimization of $w_i\in S^2$, the network repositions the axisymmetric $\mathrm{ReLU}$ activation pattern across the sphere $S^2$. This mechanism enables selective amplification of specific global harmonics $Y_{\ell}^j(\tau,\phi)$ at different training stages, extending beyond the low-frequency dominance observed with fixed weight vectors.

The fundamental insight is that while individual ReLU neurons possess a fixed 
spectral decay profile $\{c_{\ell}\}_{\ell\ge0}$ (where $|c_{\ell}|\approx O(\ell^{5/2}/2^\ell)$) 
in their local reference frame, the global harmonic content becomes orientation-dependent. 
The evolution of $w_i$ induces corresponding changes in the rotation operator $R_{w_i}$, 
which transforms the local expansion into global coordinates. This transformation 
provides the network with the capacity to redistribute spectral energy across 
different global frequencies throughout the training process.

\subsection{From Single Neurons to the Entire Network}

Consider the shallow neural network \eqref{nnrepresent-sphere},  from Section~\ref{subsec:NeuronExpansionVariableDir} \eqref{reluglobalform}, each single-neuron term 
\(\mathrm{ReLU}\!\bigl(w_i^\top x(\tau,\phi)\bigr)\) can be expanded in \emph{global} spherical harmonics by mixing the 
local axisymmetric modes (indexed by \(\ell\) with \(j=0\)) via the rotation \(R_{w_i}\).

Substituting these single-neuron expansions into $u(\tau,\phi;\theta)$, we obtain:
\[
\begin{aligned}
u\bigl(\tau,\phi;\theta\bigr)
&=\;
\sum_{i=1}^{m} 
a_i 
\sum_{\ell=0}^{\infty}
c_{\ell}
\sum_{j=-\ell}^{\ell} 
D_{j0}^{(\ell)}\bigl(R_{w_i}\bigr)\,
Y_{\ell}^j(\tau,\phi)
\\[5pt]
&=\;
\sum_{\ell=0}^{\infty}
\sum_{j=-\ell}^{\ell}
\Bigl[
\sum_{i=1}^{m} 
\bigl(a_i\,c_{\ell}\,D_{j0}^{(\ell)}(R_{w_i})\bigr)
\Bigr]
\,Y_{\ell}^j(\tau,\phi).
\end{aligned}
\]
Hence, the coefficient in front of each global harmonic $Y_{\ell}^j(\tau,\phi)$ is 
\(\sum\limits_{i=1}^m a_i\,c_{\ell}\,D_{j0}^{(\ell)}(R_{w_i})\).  Because \(c_{\ell}\approx O(\ell^{5/2}/2^\ell)\) 
indicates rapid decay for large \(\ell\) in the \emph{local} axis frame, the network inherits a natural bias toward 
lower-frequency modes.  However, the rotation terms $R_{w_i}$ (and thus $D_{j0}^{(\ell)}(R_{w_i})$) can evolve 
during training, altering how that local bias manifests \emph{globally}.

When each $w_i$ is free to move on $S^2$ during training, the local axis alignment (itself tied to $R_{w_i}$) changes over time. This means the same intrinsic mode $c_{\ell}\,Y_{\ell}^0(\tilde{x})$ may be mapped differently to global modes $Y_{\ell}^j(\tau,\phi)$ at different training steps. This dynamic remapping operates through two complementary mechanisms. \textbf{Each neuron maintains its intrinsic low-frequency bias:} In any local coordinate system aligned with $w_i$, the expansion of $\max\{0,\cos(\tilde{\tau})\}$ heavily favors small $\ell$ values due to the inherent spectral decay. \textbf{However, collective rotational effects enable global frequency redistribution:} By updating each $w_i$, the network can redirect how that low-frequency local pattern is projected onto the global $\{\ell,j\}$ modes. A direction $w_i$ that initially favored one set of global modes could, upon rotation, enhance or suppress different harmonic components.

In some regimes, this flexibility may \emph{still} preserve an overall low-frequency bias, as all neurons collectively 
reinforce lower modes.  In others, certain neurons can rotate in a way that boosts high-frequency modes 
(i.e., $j\ne0$ or large $\ell$), partially overriding the intrinsic decay of $\ell^{5/2}/2^\ell$.  
Thus, to fully appreciate the frequency behavior of a trained network, one must analyze:

\begin{enumerate}
\item The \emph{intrinsic} coefficients \(\{c_{\ell}\}\) of each neuron’s axis-aligned ReLU expansion, and
\item The evolving orientation \(\{w_i(t)\}\) (or equivalently the rotation \(\{R_{w_i}(t)\}\)) under training.
\end{enumerate}

By considering both factors---the natural low-frequency emphasis of each ReLU \emph{and} the ability to reorient 
that emphasis globally---one gains a clearer picture of how (and under what conditions) the network might deviate 
from purely low-frequency-first learning and sustain or amplify higher-frequency components.

\subsection{Temporal Evolution of the Error with Trained Weights}
\label{subsec:TemporalEvolutionTrainedWeights}

We now analyze the case in which \emph{both} the output-layer coefficients $\{a_i\}$ and the directions $\{w_i\}\subset S^2$ evolve under gradient descent.   To track how $D(\tau,\phi)$  defined in \eqref{errorxi-sphere} changes over time, note that
\[
\frac{\partial D(\tau,\phi)}{\partial t}
\;=\;
\frac{\partial}{\partial t}
\bigl[u(\tau,\phi;\theta) - h(\tau,\phi)\bigr]
\;=\;
\frac{\partial u(\tau,\phi;\theta)}{\partial \theta}
\;\cdot\;
\frac{d\theta}{dt}.
\]
Since $\theta$ now includes both $a_i$ and $w_i$, we partition:
\[
\frac{\partial D(\tau,\phi)}{\partial t}
\;=\;
\sum_{i=1}^m
\Bigl[
  \underbrace{\frac{\partial u(\tau,\phi;\theta)}{\partial a_i}
  \;\frac{\partial L(\theta)}{\partial a_i}}_{\text{update in }a_i}
  \;+\;
  \underbrace{\frac{\partial u(\tau,\phi;\theta)}{\partial w_i}
  \;\frac{\partial L(\theta)}{\partial w_i}}_{\text{update in }w_i}
\Bigr].
\]
Recall
\[
\frac{\partial u(\tau,\phi;\theta)}{\partial a_i}
\;=\;
\mathrm{ReLU}\bigl(w_i^\top x(\tau,\phi)\bigr),
\quad
\frac{\partial u(\tau,\phi;\theta)}{\partial w_i}
\;=\;
a_i\,\nabla_{w_i}\!\mathrm{ReLU}\bigl(w_i^\top x(\tau,\phi)\bigr).
\]
In turn, the loss gradients are
\[
\frac{\partial L(\theta)}{\partial a_i}
\;=\;
-\int_{0}^{2\pi}\!\!\int_{0}^{\pi}
D(\tau',\phi')\,
\mathrm{ReLU}\bigl(w_i^\top x(\tau',\phi')\bigr)
\,\sin\tau'\,d\tau'\,d\phi',
\]
and
\[
\frac{\partial L(\theta)}{\partial w_i}
\;=\;
-\int_{0}^{2\pi}\!\!\int_{0}^{\pi}
D(\tau',\phi')\,a_i
\,\nabla_{w_i}\!\mathrm{ReLU}\bigl(w_i^\top x(\tau',\phi')\bigr)
\,\sin\tau'\,d\tau'\,d\phi'.
\]
Hence,
\[
\begin{aligned}
    \frac{\partial D(\tau,\phi)}{\partial t}
&\;=\;
-\sum_{i=1}^m
\Bigl\{
  \mathrm{ReLU}\bigl(w_i^\top x(\tau,\phi)\bigr)
  \int_{S^2} D(\tau',\phi')\,\mathrm{ReLU}\bigl(w_i^\top x(\tau',\phi')\bigr)\,d\Omega \\
&\;+\;
  a_i^2\,\nabla_{w_i}\!\mathrm{ReLU}\bigl(w_i^\top x(\tau,\phi)\bigr)
  \int_{S^2} D(\tau',\phi')\,\nabla_{w_i}\!\mathrm{ReLU}\bigl(w_i^\top x(\tau',\phi')\bigr)\,d\Omega
\Bigr\}.
\end{aligned}
\]

Observe that
\[
\nabla_{w_i}\!\mathrm{ReLU}\bigl(w_i^\top x\bigr)
\;=\;
\begin{cases}
x(\tau,\phi), & \text{if } w_i^\top x(\tau,\phi)>0,\\
0, & \text{if } w_i^\top x(\tau,\phi)\le 0.
\end{cases}
\]
Thus the term 
\(\mathrm{ReLU}\bigl(w_i^\top x(\tau,\phi)\bigr)\)
affects how $a_i$ updates (the scalar “output weights”), whereas
\(\nabla_{w_i}\!\mathrm{ReLU}\bigl(w_i^\top x(\tau,\phi)\bigr)\)
affects how $w_i$ itself rotates or shifts.  Concretely:
\begin{itemize}
\item 
\emph{Direct Output-Weight Adjustment}:
\(\int_{S^2}D(\tau',\phi')\,\mathrm{ReLU}(w_i^\top x(\tau',\phi'))\,d\Omega\)
measures how the error $D$ correlates with the half-space $w_i^\top x(\tau',\phi')\ge0$.  
\item 
\emph{Direction Rotation}:
\(\int_{S^2}D(\tau',\phi')\,\nabla_{w_i}\!\mathrm{ReLU}(w_i^\top x(\tau',\phi'))\,d\Omega\) 
captures changes in $w_i$ can reorient which region $\{w_i^\top x\ge0\}$ is active, thus “chasing” the parts of $S^2$ where $D$ is large.
\end{itemize}

Focus, for instance, on the integral
\(\int_{S^2}D(\tau',\phi')\,\mathrm{ReLU}(w_i^\top x(\tau',\phi'))\,d\Omega\).
Defining the upper hemisphere 
\(\Omega_i=\bigl\{(\tau',\phi'): w_i^\top x(\tau',\phi')\ge0\bigr\}\), one can explicitly evaluate or 
invoke a 2D Mean Value Theorem argument on \(\tau'\in[0,\tfrac{\pi}{2}]\).  Analogous ideas apply to the 
gradient term $\nabla_{w_i}\!\mathrm{ReLU}$.  In each case, one obtains a representative point 
\((\tau_i^*,\phi_i^*)\) or \((\tau_i^\dagger,\phi_i^\dagger)\) in the active region.  Summarizing, 
\[
\int_{\Omega_i}
D(\tau',\phi')\cos(\tau')\sin(\tau')\,d\tau'\,d\phi'
\;=\;
\pi\,D(\tau_i^*,\phi_i^*)\,\sin\bigl(2\,\tau_i^*\bigr),
\]
and
\[
\int_{\Omega_i}
D(\tau',\phi')\,y(\tau',\phi')
\,d\tau'\,d\phi'
\;=\;
\pi\,D(\tau_i^\dagger,\phi_i^\dagger)\,\bar{y}_i,
\]
for some average vector $\bar{y}_i = \Bigl(\sin(2\tau_i^\dagger),
2\sin^2(\tau_i^\dagger)\cos(\phi_i^\dagger),
2\sin^2(\tau_i^\dagger)\sin(\phi_i^\dagger)\Bigr)$ in $\Omega_i$.  

Putting these pieces together, one obtains a schematic evolution law:
\[
\begin{aligned}
\frac{\partial D(\tau,\phi)}{\partial t}
&=
-\sum_{i=1}^m
\Bigl[
  \mathrm{ReLU}\bigl(w_i^\top x(\tau,\phi)\bigr)\,
  \Bigl(\pi\,D\bigl(\tau_i^*,\phi_i^*\bigr)\,\sin\bigl(2\tau_i^*\bigr)\Bigr)
\\[-3pt]
&\qquad\qquad
+\;
  a_i^2\,
  \Bigl(\pi\,D\bigl(\tau_i^\dagger,\phi_i^\dagger\bigr)\,\bar{y}_i\Bigr)
  \,\cdot\,
  \nabla_{w_i}\!\mathrm{ReLU}\bigl(w_i^\top x(\tau,\phi)\bigr)
\Bigr].
\end{aligned}
\]
The first term updates the scalar coefficients $\{a_i\}$, while the second (through $\nabla_{w_i}$) reorients $w_i$.  
Hence the network can reconfigure both the amplitude and the direction of each neuron’s half-space activation.  
Compared to the fixed-weight scenario, this added flexibility in rotating $w_i$ can drastically reshape how 
the error $D(\tau,\phi)$ develops in frequency space.  Higher modes may be selectively enhanced or suppressed 
depending on how the $w_i$ moves, providing a potential pathway for either preserving (or breaking) a 
low-frequency bias in the SNN’s training dynamics.

\subsection{Implications for the Frequency Principle with Trained Weights}
\label{Subsec:FPTrainedWeights}

When \emph{both} the output coefficients \(\{a_i\}\) and the directions \(\{w_i\}\subset S^2\) evolve under gradient descent, the time evolution of the error 
\[
D(\tau,\phi)
\;=\;
u\bigl(\tau,\phi;\theta\bigr)\;-\;h(\tau,\phi),
\]
becomes more intricate for the reason that the additional ``directional updates'' in \(\nabla_{w_i}\mathrm{ReLU}\bigl(w_i^\top x(\tau,\phi)\bigr)\).  From Section~\ref{subsec:TemporalEvolutionTrainedWeights}, the partial derivative of \(D(\tau,\phi)\) under gradient descent is
\[
\begin{aligned}
\frac{\partial D(\tau,\phi)}{\partial t}
&\;=\;
-\sum_{i=1}^{m}
\Bigl(
  \mathrm{ReLU}\bigl(w_i^\top x(\tau,\phi)\bigr)\,\underbrace{\pi\,D\bigl(\tau_i^*,\phi_i^*\bigr)\,\sin\bigl(2\tau_i^*\bigr)}_{\text{from } \int D(y)\,\mathrm{ReLU}(w_i^\top y)\,dy}
\\[-3pt]
&\qquad\qquad
+\;a_i^2\,\nabla_{w_i}\,\mathrm{ReLU}\!\bigl(w_i^\top x(\tau,\phi)\bigr)\,
\underbrace{\pi\,D\bigl(\tau_i^\dagger,\phi_i^\dagger\bigr)\,\bar{y}_i}_{\text{from } \int D(y)\,\nabla_{w_i}\mathrm{ReLU}(w_i^\top y)\,dy}
\Bigr),
\end{aligned}
\]
where the points \(\bigl(\tau_i^*,\phi_i^*\bigr)\) and \(\bigl(\tau_i^\dagger,\phi_i^\dagger\bigr)\) arise from a 2D Mean Value Theorem argument over the hemisphere \(\{(\tau',\phi'): w_i^\top x(\tau',\phi')\ge0\}\).  Additionally, each rotation \(R_{w_i}\) mapping the north pole to \(w_i\) can be parametrized by angles \((\tau,\phi)\).

Recalling the \emph{global} expansion of each neuron’s ReLU term (see \eqref{reluglobalform}) \eqref{reluglobalform}, we substitute this into the gradient-based update for $\tfrac{\partial D}{\partial t}$.  Consequently, we collect the resulting terms into spherical harmonic modes \(Y_{\ell}^j(\tau,\phi)\).  For example, define the \emph{direct neuron-output} contribution:
\begin{equation}\label{Cellj}
    C_{\ell}^j 
    \;:=\; 
    -\sum_{i=1}^m 
    \Bigl[\pi\,D\bigl(\tau_i^*,\phi_i^*\bigr)\,\sin\bigl(2\tau_i^*\bigr)\Bigr]\;
    c_{\ell}
    \,D_{j0}^{(\ell)}\bigl(R_{w_i}\bigr),
\end{equation}
and the \emph{weight rotation} contribution:
\begin{equation}\label{Gellj}
    G_{\ell}^j 
    \;:=\; 
    -\sum_{i=1}^m 
    a_i^2
    \Bigl[\pi\,D\bigl(\tau_i^\dagger,\phi_i^\dagger\bigr)\,\bar{y}_i\Bigr]
    \;\cdot\;
    \nabla_{w_i}\Bigl(
      c_{\ell}\,D_{j0}^{(\ell)}\!\bigl(R_{w_i}\bigr)
    \Bigr).
\end{equation}

 To make $G_{\ell}^j$ explicit, we need to calculate
\[
\nabla_{w_i}\Bigl(c_{\ell} D_{j0}^{(\ell)}(R_{w_i})\Bigr),
\]
which expands as
$$
\begin{aligned}
 \nabla_{w_i}\Bigl(c_{\ell} D_{j0}^{(\ell)}(R_{w_i})\Bigr) =\; c_{\ell}\;\nabla_{w_i}\Bigl(D_{j0}^{(\ell)}(R_{w_i})\Bigr) =\; c_{\ell}\Bigl(\frac{\partial D_{j0}^{(\ell)}(R_{w_i})}{\partial R_{w_i}}\;\frac{\partial R_{w_i}}{\partial w_i}\Bigr).
\end{aligned}
$$

The rotation \(R_{w_i}\) depends on \((x_i,y_i,z_i)\) through angles \(\tau\) (polar) and \(\phi\) (azimuthal), with \(\cos(\tau)=z_i/\|w_i\|\) and \(\tan(\phi)=y_i/x_i\).  Thus
$$
\frac{\partial R_{w_i}}{\partial w_i} 
\;=\; 
\Bigl(
  \frac{\partial R_{w_i}}{\partial x_i},\;
  \frac{\partial R_{w_i}}{\partial y_i},\;
  \frac{\partial R_{w_i}}{\partial z_i}
\Bigr)^\top,
$$
and the Wigner \(D\)-matrix elements \(D_{j0}^{(\ell)}\bigl(R_{w_i}\bigr)\) can be written as 
\(\sqrt{\tfrac{4\pi}{2\ell+1}}\,Y_{\ell}^j(\tau,\phi)\). Consequently,
$$
\begin{aligned}
 \nabla_{w_i}\Bigl(c_{\ell} D_{j0}^{(\ell)}(R_{w_i})\Bigr)&=\; c_{\ell}\sqrt{\frac{4\pi}{2\ell+1}}\Biggl(
\frac{\partial Y_{\ell}^j(\tau,\phi)}{\partial \tau}\,\frac{\partial \tau}{\partial w_i} \;+\;
\frac{\partial Y_{\ell}^j(\tau,\phi)}{\partial \phi}\,\frac{\partial \phi}{\partial w_i}
\Biggr)
\end{aligned}
$$
with
$$
\frac{\partial \tau}{\partial w_i}
=
\begin{pmatrix}
-\frac{\cos\tau\sin\tau\cos\phi}{\|w_i\|^2\sqrt{1-\tfrac{\cos^2\tau}{\|w_i\|^2}}}\\[4pt]
-\frac{\cos\tau\sin\tau\sin\phi}{\|w_i\|^2\sqrt{1-\tfrac{\cos^2\tau}{\|w_i\|^2}}}\\[4pt]
\frac{\sin\tau}{\|w_i\|^2\sqrt{1-\tfrac{\cos^2\tau}{\|w_i\|^2}}}
\end{pmatrix},
\qquad
\frac{\partial \phi}{\partial w_i}
=
\begin{pmatrix}
-\frac{\sin\phi}{\sin\tau},\\[3pt]
\frac{\cos\phi}{\sin\tau},\\[3pt]
0
\end{pmatrix}.
$$

Recall that 
\[
Y_{\ell}^j(\tau,\phi) 
\;=\; 
N_{\ell}^j\;P_{\ell}^j\!\bigl(\cos\tau\bigr)\,e^{ij\phi},
\]
where \(N_{\ell}^j=\sqrt{\tfrac{2\ell+1}{4\pi}\,\tfrac{(\ell-j)!}{(\ell+j)!}}\) is the normalization and \(P_{\ell}^j\) are the associated Legendre polynomials.  Then we have
$$
\begin{aligned}
\frac{\partial Y_{\ell}^j(\tau,\phi)}{\partial \tau}
&=\; N_{\ell}^j e^{ij\phi}(-\sin\tau)\Bigl[
  \tfrac{\ell\cos\tau\,P_{\ell}^j(\cos\tau)-(\ell+j)\,P_{\ell-1}^j(\cos\tau)}{\sin^2\tau}
\Bigr],\\
\frac{\partial Y_{\ell}^j(\tau,\phi)}{\partial \phi}
&=\; ij\,Y_{\ell}^j(\tau,\phi).
\end{aligned}
$$

Substituting these derivatives back yields
\begin{equation}\label{gradiectclidlrw}
    \begin{aligned}
& \nabla_{w_i}\Bigl(c_{\ell} D_{j0}^{(\ell)}(R_{w_i})\Bigr) \\
&=\; c_{\ell}\sqrt{\frac{4\pi}{2\ell+1}}
\Biggl\{
  N_{\ell}^j e^{ij\phi}\sin\tau
  \Bigl[
    \frac{\ell\cos\tau\,P_{\ell}^j(\cos\tau)}{\sin^2\tau}
    \;-\;
    \frac{(\ell+j)\,P_{\ell-1}^j(\cos\tau)}{\sin^2\tau}
  \Bigr]\\
&\quad \begin{pmatrix}
-\frac{\cos\tau\sin\tau\cos\phi}{\|w_i\|^2\sqrt{1-\tfrac{\cos^2\tau}{\|w_i\|^2}}}\\[4pt]
-\frac{\cos\tau\sin\tau\sin\phi}{\|w_i\|^2\sqrt{1-\tfrac{\cos^2\tau}{\|w_i\|^2}}}\\[4pt]
\frac{\sin\tau}{\|w_i\|^2\sqrt{1-\tfrac{\cos^2\tau}{\|w_i\|^2}}}
\end{pmatrix}
+\;ij\,N_{\ell}^j P_{\ell}^j(\cos\tau)\,e^{ij\phi}
\begin{pmatrix}
-\frac{\sin\phi}{\sin\tau},\\[3pt]
\frac{\cos\phi}{\sin\tau},\\[3pt]
0
\end{pmatrix}
\Biggr\}.
\end{aligned}
\end{equation}

Hence the gradient of the spherical harmonic expansion depends on the angles \(\tau,\phi\) and the spherical harmonic degrees \(\ell\) and orders \(j\).  The first part (derivative wrt.\ \(\tau\)) captures how changes in the polar angle affect associated Legendre polynomials, while the second part (derivative wrt.\ \(\phi\)) tracks azimuthal dependence via the factor \(e^{ij\phi}\). Putting~\eqref{gradiectclidlrw} into \eqref{Gellj} and including the dot product with \(\bar{y}_i\), one obtains the final explicit form of \(G_{\ell}^j\):

\begin{equation}\label{Gelljexact}
\begin{aligned}
    G_{\ell}^j &:= -2\pi\sum_{i=1}^m a_i^2\pi\,D\bigl(\tau_i^\dagger,\phi_i^\dagger\bigr)\,c_{\ell}\sqrt{\frac{4\pi}{2\ell+1}} 
    \Biggl\{
    N_{\ell}^j e^{ij\phi_i^\dagger}\sin\tau_i^\dagger \Bigl[\frac{\ell \cos\tau_i^\dagger\,P_{\ell}^j(\cos\tau_i^\dagger)}{\sin^2\tau_i^\dagger}
    \;-\;\\&\quad
    \frac{ (\ell+j)\,P_{\ell-1}^j(\cos\tau_i^\dagger)}{\sin^2\tau_i^\dagger}\Bigr] 
 \bar{y}_i \;\cdot\;
\begin{pmatrix}
-\frac{\cos\tau_i^\dagger\sin\tau_i^\dagger\cos\phi_i^\dagger}{\|w_i\|^2\sqrt{1-\tfrac{\cos^2\tau_i^\dagger}{\|w_i\|^2}}}\\[4pt]
-\frac{\cos\tau_i^\dagger\sin\tau_i^\dagger\sin\phi_i^\dagger}{\|w_i\|^2\sqrt{1-\tfrac{\cos^2\tau_i^\dagger}{\|w_i\|^2}}}\\[4pt]
\frac{\sin\tau_i^\dagger}{\|w_i\|^2\sqrt{1-\tfrac{\cos^2\tau_i^\dagger}{\|w_i\|^2}}}
\end{pmatrix}\\
&\quad +\,ij\,N_{\ell}^j P_{\ell}^j(\cos\tau_i^\dagger)e^{ij\phi_i^\dagger}\,\bar{y}_i \cdot
\begin{pmatrix}
-\frac{\sin\phi_i^\dagger}{\sin\tau_i^\dagger},\\[3pt]
\frac{\cos\phi_i^\dagger}{\sin\tau_i^\dagger},\\[3pt]
0
\end{pmatrix}
\Biggr\}.
\end{aligned}
\end{equation}

Finally, collecting all terms, the error’s spherical harmonic components evolve according to
\begin{equation}\label{dDelljt}
    \frac{\partial D(x)}{\partial t}
    \;=\;
    \sum_{\ell = 0}^{\infty}\sum_{j = -\ell}^{\ell}\bigl(C_{\ell}^j + G_{\ell}^j\bigr)\;Y_{\ell}^j(x).
\end{equation}

This shows how both the direct amplitude updates (via \(C_{\ell}^j\)) and the rotation effects (via \(G_{\ell}^j\)) shape the frequency distribution in the trained-weights regime.

Unlike the fixed-weight setting (Section~\ref{Sec:SphericalAnalysisFW}), each coefficient $G_{\ell}^j$ now provides 
a mechanism for rotating $w_i$ to activate different regions of $S^2$.  As $w_i$ changes, the rotation $R_{w_i}$ in 
\eqref{reluglobalform} continuously re-mixes the local axisymmetric modes across the global $j$ modes.  This adds 
substantial flexibility, allowing higher-frequency components to be selectively enhanced or suppressed based on the 
error distribution $D(\tau,\phi)$ over time.  Whether a low-frequency bias persists (the standard Frequency Principle) 
thus depends on:
\begin{itemize}
    \item The \emph{intrinsic} decay of the single-neuron coefficients $c_{\ell}$ in $\ell$, and
    \item How each neuron’s direction $w_i$ evolves to redirect these local expansions into the global harmonics.
\end{itemize}
In certain regimes, the network may still predominantly favor lower frequencies, while in others, strategic rotations 
of $w_i$ can shift energy into higher modes, possibly violating the usual FP.  The precise outcome depends on the 
relative magnitudes of $C_{\ell}^j$ and $G_{\ell}^j$ as well as the dynamics of $w_i(t)$.  

Overall, \emph{both} amplitude updates ($C_{\ell}^j$) and rotation updates ($G_{\ell}^j$) jointly shape the error’s 
frequency content in the fully-trained-weight scenario.  This underscores the importance of examining how directional 
degrees of freedom permit further reconfiguration of the SNN’s spherical harmonic representation beyond the simpler 
fixed-weight analyses.

Having established the updated expressions for \(C_{\ell}^j\) and \(G_{\ell}^j\), we now turn to analyze how they shape the Frequency Principle in the trained-weights setting.  Below, we present a series of theorems and corollaries that clarify when and why low-frequency dominance is preserved, and under what circumstances it may fail.  We begin by formalizing the \emph{decay rates} of these harmonic coefficients, followed by a result on \emph{immediate violation} of the FP if certain initial error conditions are satisfied.  We then conclude with two corollaries highlighting the precise conditions for FP preservation or violation when both \(a_i\) and \(w_i\) are trained.

\begin{lemma}[Decay Rates of Error Evolution Terms]
\label{thm:decayrates-trainedweights}
Consider the spherical harmonic coefficients \(C_{\ell}^j\) and \(G_{\ell}^j\) in the error evolution equation \eqref{dDelljt}, where \(C_{\ell}^j\) and \(G_{\ell}^j\) are defined by \eqref{Cellj} and \eqref{Gellj}, respectively.  Under the assumption that the geometric/error factors remain bound, the following estimates hold:
\begin{enumerate}
\item 
\(\bigl|C_{\ell}^j\bigr| = O\!\Bigl(\dfrac{\ell^{5/2}}{2^\ell}\Bigr)\).
\item 
\(\bigl|G_{\ell}^j\bigr| = O\!\Bigl(\dfrac{\ell^{7/2}}{2^\ell}\Bigr)\).
\end{enumerate}
\end{lemma}

\begin{proof}
\textbf{Estimate for \(\bigl|C_{\ell}^j\bigr|\).}
From \eqref{Cellj}, each coefficient \(C_{\ell}^j\) is a linear combination of the single-neuron coefficients \(c_{\ell}\), multiplied by bounded factors \(\sin\bigl(2\tau_i^*\bigr)\), \(\pi\,D(\tau_i^*,\phi_i^*)\), and the Wigner element \(D_{j0}^{(\ell)}\bigl(R_{w_i}\bigr)\).  Since \(\bigl|c_{\ell}\bigr|\) decays on the order of \(O(\ell^{5/2}/2^\ell)\) and all other terms are \(O(1)\), we deduce
\[
\bigl|C_{\ell}^j\bigr|
\;=\;
O\!\Bigl(\tfrac{\ell^{5/2}}{2^\ell}\Bigr).
\]

\noindent
\textbf{Estimate for \(\bigl|G_{\ell}^j\bigr|\).}
From \eqref{Gellj} and \eqref{Gelljexact}, the term \(G_{\ell}^j\) involves an additional derivative 
\(\nabla_{w_i}\bigl(c_{\ell}D_{j0}^{(\ell)}(R_{w_i})\bigr)\).  
Inspecting \eqref{gradiectclidlrw}, we see that partial derivatives with respect to the angles \(\tau\) and \(\phi\) introduce an extra factor of \(O(\ell)\).  
Hence multiplying this additional factor by the intrinsic \(\bigl|c_{\ell}\bigr|\approx O(\ell^{5/2}/2^\ell)\) yields 
\[
\bigl|G_{\ell}^j\bigr|
\;=\;
O\!\Bigl(\tfrac{\ell^{7/2}}{2^\ell}\Bigr).
\]
Thus, although \(G_{\ell}^j\) still decays exponentially with \(\ell\), it carries an extra factor of \(\ell\) relative to \(C_{\ell}^j\).
\end{proof}

\begin{corollary}
    [Immediate FP Violation in the Trained-Weights Scenario]
\label{thm:fp-fails-trainedweights}
Let $\{w_i\}_{i=1}^m\subset S^2$ be directions updated by gradient descent (so $w_i=w_i(t)$) and $\{a_i\}$ be scalar output weights (also evolving in time).  Consider the shallow ReLU network.
\[
u(\tau,\phi;\theta)\;=\;\sum_{i=1}^m a_i\,\mathrm{ReLU}\!\bigl(w_i^\top x(\tau,\phi)\bigr),
\]
on the sphere, and define the error function
\[
D(\tau,\phi)\;=\;u(\tau,\phi;\theta)\;-\;h(\tau,\phi)
\]
 and suppose $D(\tau,\phi)$ is continuous on $S^2$.  Let
\[
(\tau_i^*,\phi_i^*)\in \Bigl[0,\tfrac{\pi}{2}\Bigr]\times[0,2\pi]
\quad\text{and}\quad
(\tau_i^\dagger,\phi_i^\dagger)\in \Bigl[0,\tfrac{\pi}{2}\Bigr]\times[0,2\pi]
\]
be the representative points for the integrals 
\(\int D(\tau',\phi')\,\mathrm{ReLU}\!\bigl(w_i^\top x(\tau',\phi')\bigr)\,d\Omega\)
and 
\(\int D(\tau',\phi')\,\nabla_{w_i}\mathrm{ReLU}\!\bigl(w_i^\top x(\tau',\phi')\bigr)\,d\Omega\),
respectively.   From \eqref{Cellj} and \eqref{Gellj}, the error’s spherical harmonic components evolve according to
\[
C_{\ell}^j
\;=\;
-\sum_{i=1}^m 
\Bigl[
  \pi\,D\bigl(\tau_i^*,\phi_i^*\bigr)\,\sin\bigl(2\,\tau_i^*\bigr)
\Bigr]\,
c_{\ell}\,
D_{j0}^{(\ell)}\!\bigl(R_{w_i}\bigr),
\]\[
G_{\ell}^j
\;=\;
-\sum_{i=1}^m 
a_i^2\,\Bigl[\pi\,D\bigl(\tau_i^\dagger,\phi_i^\dagger\bigr)\,\bar{y}_i\Bigr]
\,\cdot\,
\nabla_{w_i}\Bigl(
  c_{\ell} D_{j0}^{(\ell)}\!\bigl(R_{w_i}\bigr)
\Bigr).
\]

\noindent
\emph{If at any training time $t$,} 
\[
D\bigl(\tau_i^*,\phi_i^*\bigr)\;=\;0
\quad\text{and}\quad
D\bigl(\tau_i^\dagger,\phi_i^\dagger\bigr)\;=\;0
\quad
\forall\,i=1,\dots,m,
\]
\emph{then}
\[
C_{\ell}^j=0
\quad
\text{and}
\quad
G_{\ell}^j=0
\quad
\text{for every }\ell\in\{0,1,2,\dots\},\; j=-\ell,\dots,\ell.
\]
\end{corollary}


\begin{remark}
If $D\bigl(\tau_i^*,\phi_i^*\bigr)=D\bigl(\tau_i^\dagger,\phi_i^\dagger\bigr)=0$ at precisely those mean-value points driving ReLU and rotation signals, the initialization kills all updates at every $\ell,j$.  Thus no preference for low frequencies arises at training onset, violating the FP in a degenerate scenario.
\end{remark}

\begin{remark}
    Let 
\(
u(\tau,\phi; \theta) 
= 
\sum_{i=1}^{m} a_i\,\mathrm{ReLU}\!\bigl(w_i^\top x(\tau,\phi)\bigr)
\)
be a shallow ReLU network on the unit sphere \(S^2\), where both \(\{a_i\}\) and \(\{w_i\}\) are updated by gradient descent.  Suppose:
\begin{itemize}
\item 
For each \(i\), the terms \(D\bigl(\tau_i^*,\phi_i^*\bigr)\) and \(D\bigl(\tau_i^\dagger,\phi_i^\dagger\bigr)\) remain bounded and do not oscillate in sign over training (no sign-flips cause systematic cancellation).
\item 
The coefficients \(\bigl|c_{\ell}\bigr|\) and their derivatives follow the exponential decay property in Lemma~\ref{thm:decayrates-trainedweights}.
\end{itemize}
Then for sufficiently large \(\ell\), the Frequency Principle holds: namely, lower-frequency modes (smaller \(\ell\)) are learned faster than higher-frequency modes (\(\ell\gg1\)).
\end{remark}




\begin{remark}
    Under the same setting and notation as Corollary~\ref{thm:fp-fails-trainedweights}, assume:
\begin{itemize}
\item 
The sign of \(D\bigl(\tau_i^*,\phi_i^*\bigr)\) and \(D\bigl(\tau_i^\dagger,\phi_i^\dagger\bigr)\) remains consistent over training, but
\item 
For certain lower \(\ell\), the rotation-driven term \(\bigl|G_{\ell}^j\bigr|\) can exceed (or be comparable to) \(\bigl|C_{\ell}^j\bigr|\), and can grow in a way that competes with or surpasses some higher \(\ell'\).  
\end{itemize}
Then there exists a regime in which the error update for \(\ell\)-modes is at least as large as that for certain higher \(\ell'\)-modes, thereby violating the usual low-frequency dominance.  In other words, the Frequency Principle can fail for those modes.
\end{remark}




\begin{figure}
    \centering
    \subfigure[$|C_\ell^{j}|$]{\includegraphics[width=6cm, height =4cm]{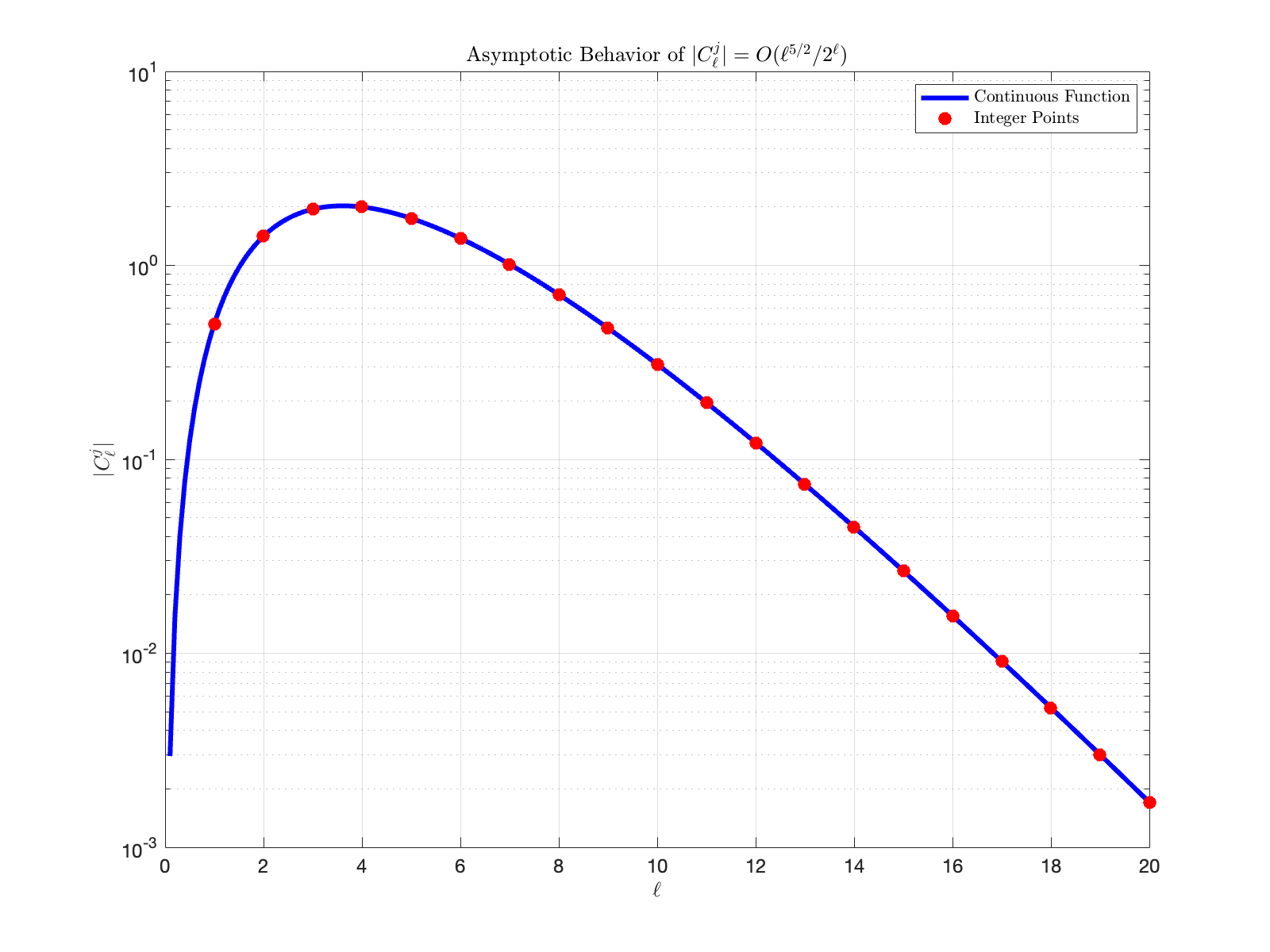}}
    \subfigure[$|G_\ell^{j}|$]{\includegraphics[width=6cm, height =4cm]{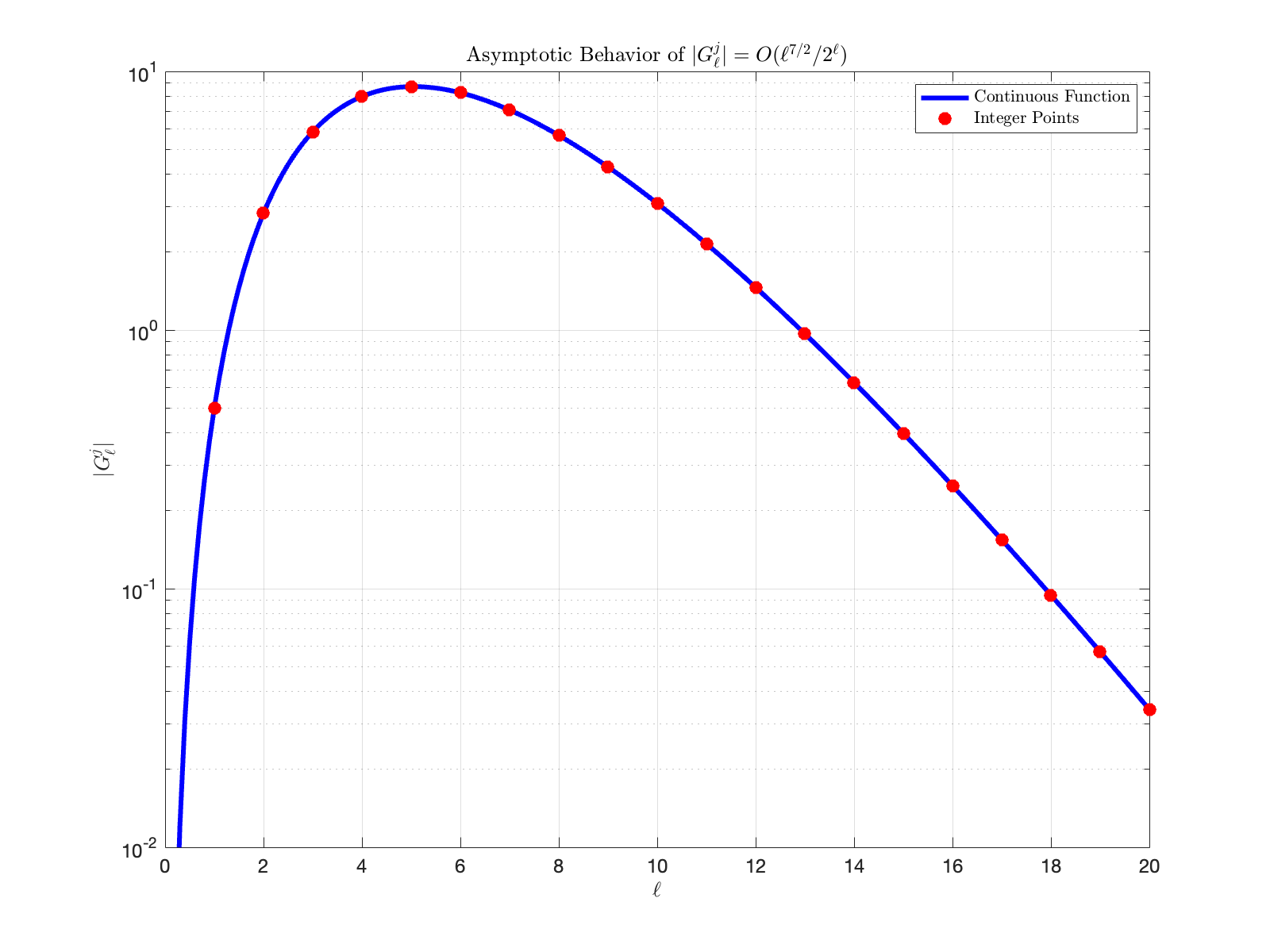}}
    \caption{Asymptotic Behavior of $|C_\ell^{j}|$ and $|G_\ell^{j}|$}
    \label{AsymptoticBehaviorofCG}
\end{figure}

In summary, these results clarify how weight rotation adds a polynomial factor \(\ell\) to the exponential decay but does not remove it.  For very large \(\ell\), both \(C_{\ell}^j\) and \(G_{\ell}^j\) (showed in Figure \ref{AsymptoticBehaviorofCG}) remain exponentially small in \(\ell\), so the FP is preserved.  However, at lower or moderate frequencies, geometric effects in \(G_{\ell}^j\) can become significant, potentially allowing certain modes to increase faster, leading to FP violations.  This framework provides a detailed account of when (and how) trained weights alter frequency-dependent learning dynamics on the sphere.

\subsection{Implications for the Frequency Principle: Numerical Example}
Like Section \ref{NumericalExampleFixedweight}, our investigation employs a shallow neural network architecture consisting of 100 neurons in the hidden layer to approximate the zero function $u(\tau,\phi)$ on $\mathbb{S}^2$. The network is trained using SGD with a learning rate of $1e{-3}$ over 10,000 epochs, for numerical example $u(\tau,\phi)=0$ and a learning rate of $1e{-3}$ over 100,000 epochs for numerical example $u(\tau,\phi) = \sin(\tau)\cos(3\phi) + \sin(3\tau)\cos(5\phi)$. Training data comprises 100 uniformly distributed sampling points on the sphere. 

\subsubsection{$u(\tau,\phi)=0$}
\textbf{Case 1: Weight-Induced Frequency Learning: Canonical Case}
\begin{figure}
    \centering
    \includegraphics[width=6cm, height =4cm]{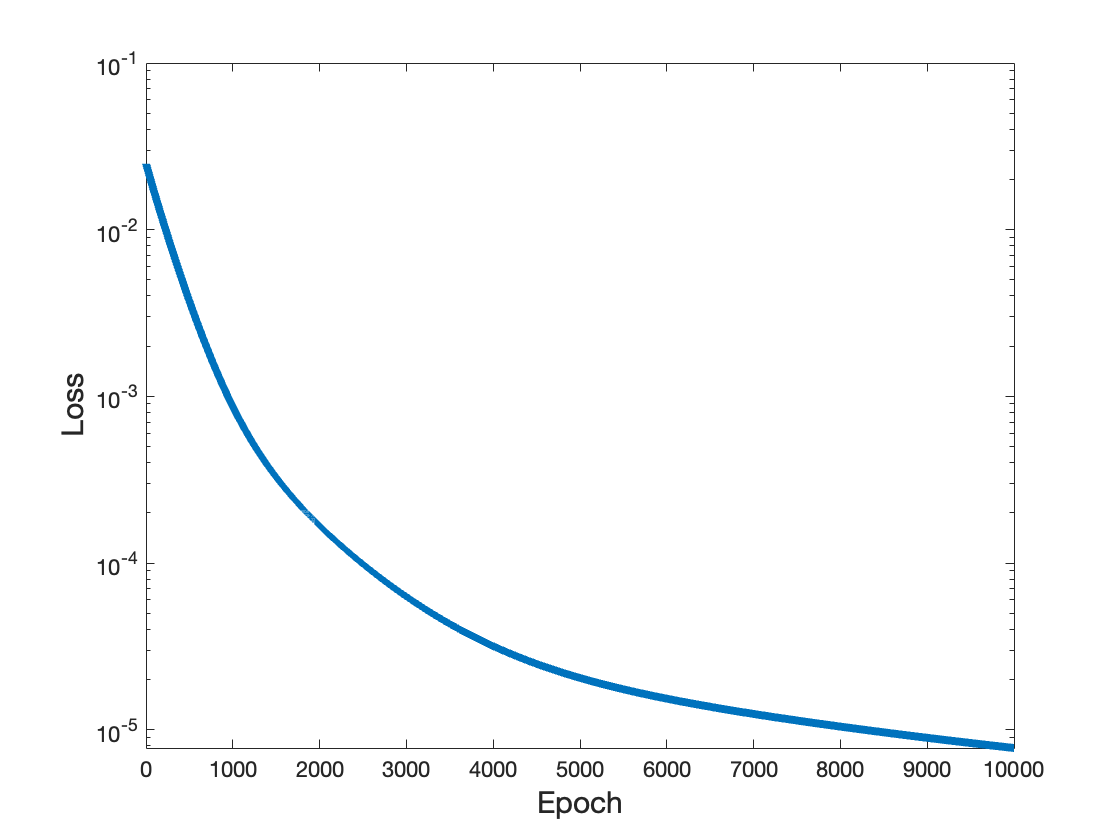}
    \caption{Image of Loss of the sixth test (Best Loss: 7.76e-6).}
    \label{6test}
\end{figure}
\begin{figure}
    \centering
    \subfigure[$l = 1,2,\ldots, 5$, $j=0$]{\includegraphics[width=6cm, height =4cm]{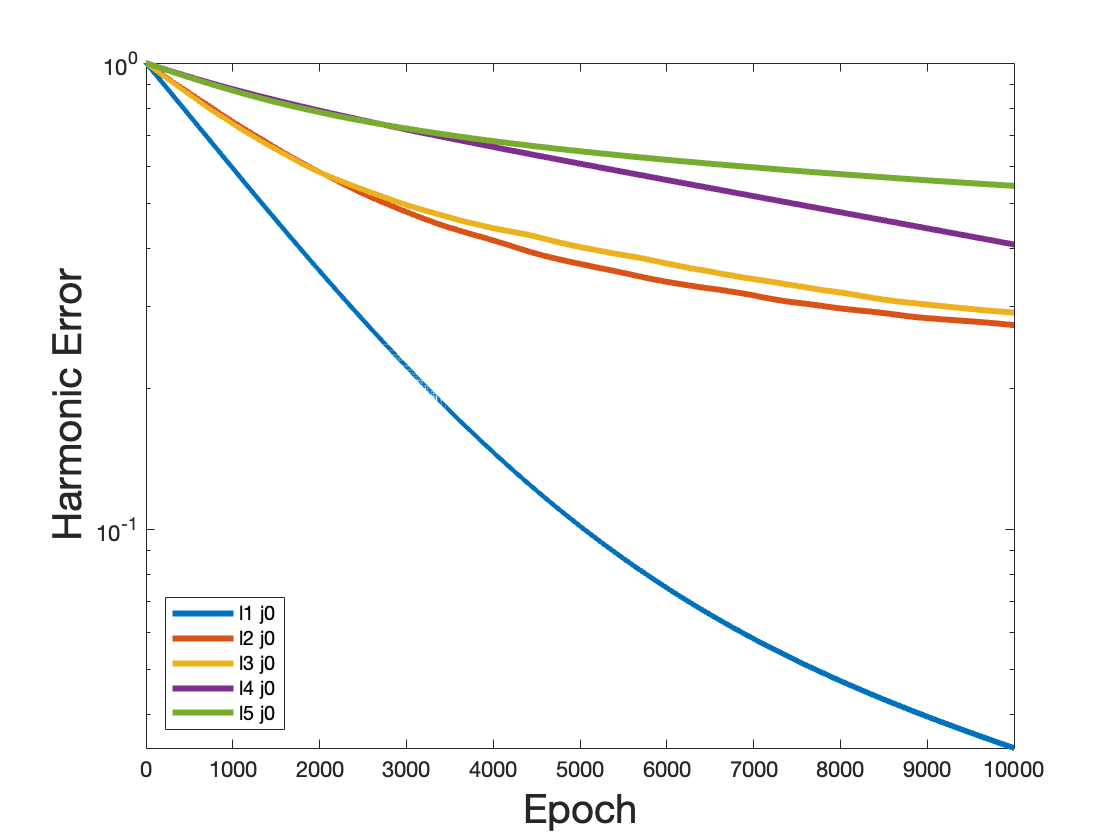}}
    \subfigure[$l = 6,7,\ldots, 10$, $j=0$]{\includegraphics[width=6cm, height =4cm]{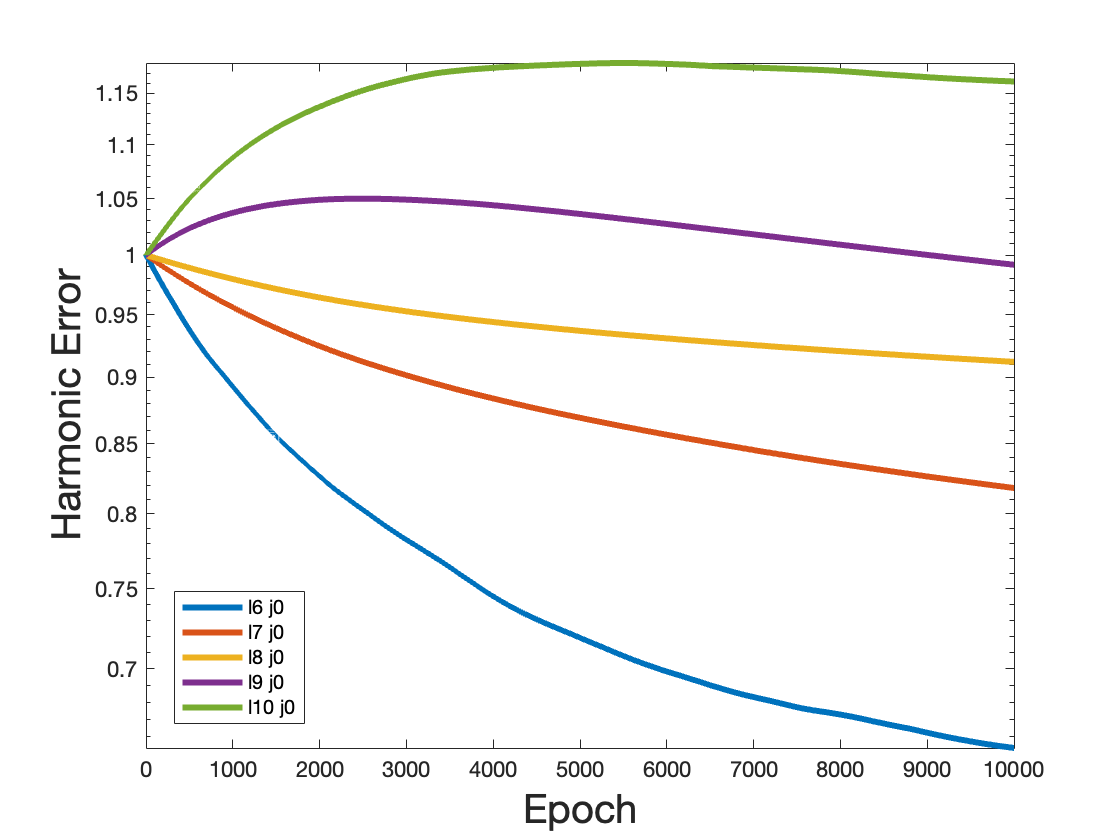}}
    \caption{Harmonic Errors of $c_{lj}$ When $l = 1,2,\ldots, 10$, $j=0$}
    \label{6testalphabeta}
\end{figure}
\begin{figure}
    \centering
    \subfigure[$l = 1,2,\ldots, 5$, $j=1$]{\includegraphics[width=6cm, height =4cm]{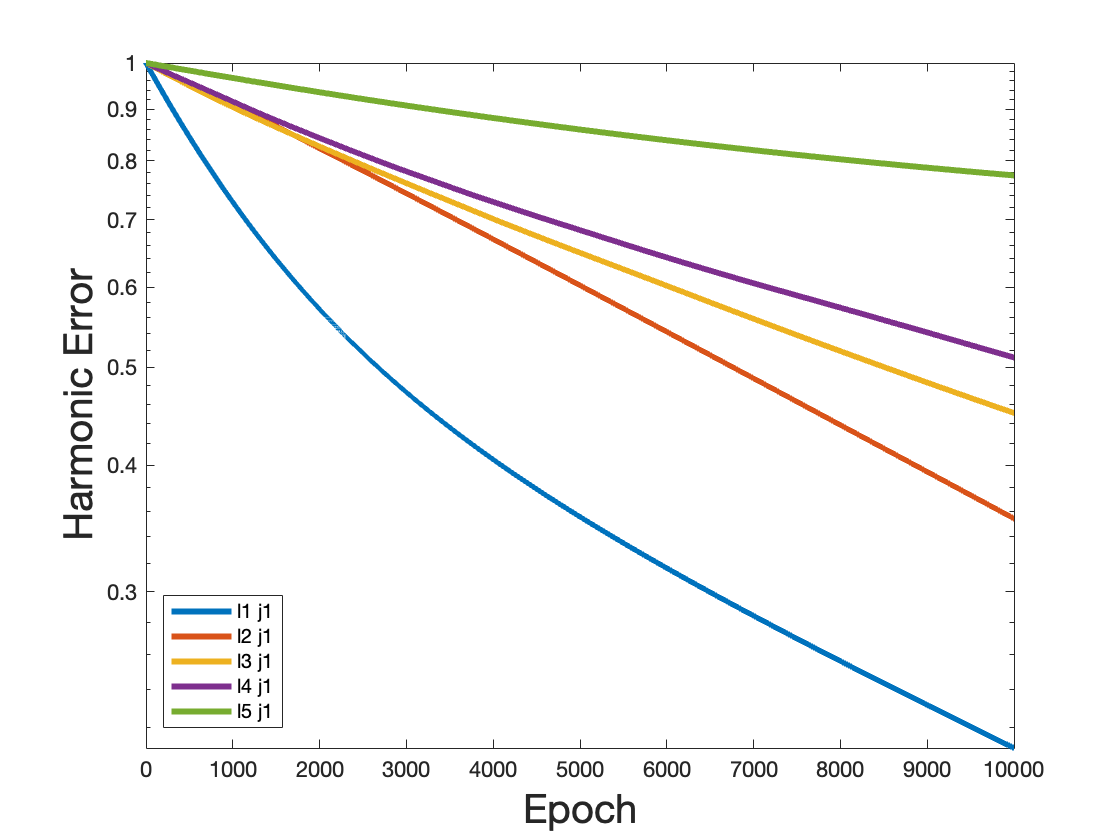}}
    \subfigure[$l = 6,7,\ldots, 10$, $j=1$]{\includegraphics[width=6cm, height =4cm]{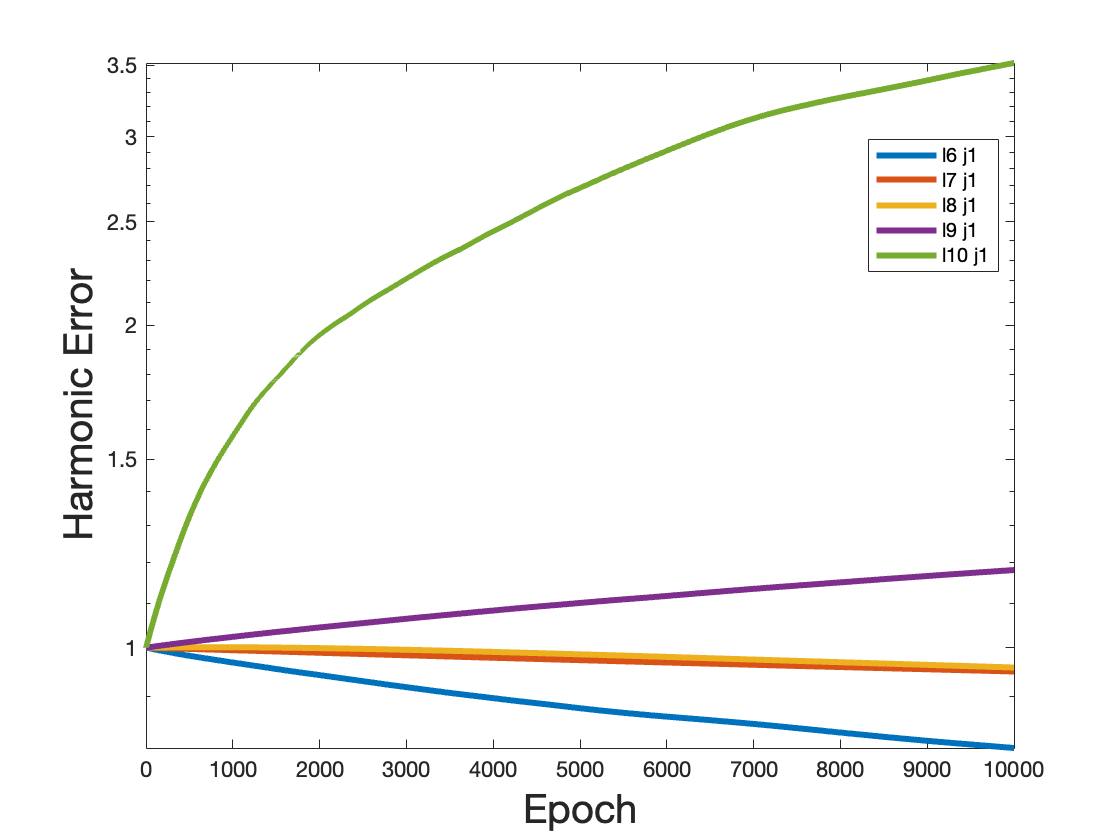}}
    \caption{Harmonic Errors of $c_{lj}$ When $l = 1,2,\ldots, 10$, $j=1$.}
    \label{6testalphabetam1}
\end{figure}

\begin{figure}
    \centering
    \subfigure[Image of the target function]{\includegraphics[width=6cm, height =4cm]{fqsphere/target_function_u0fp.png}}
    \subfigure[Image of the SNN output]{\includegraphics[width=6cm, height =4cm]{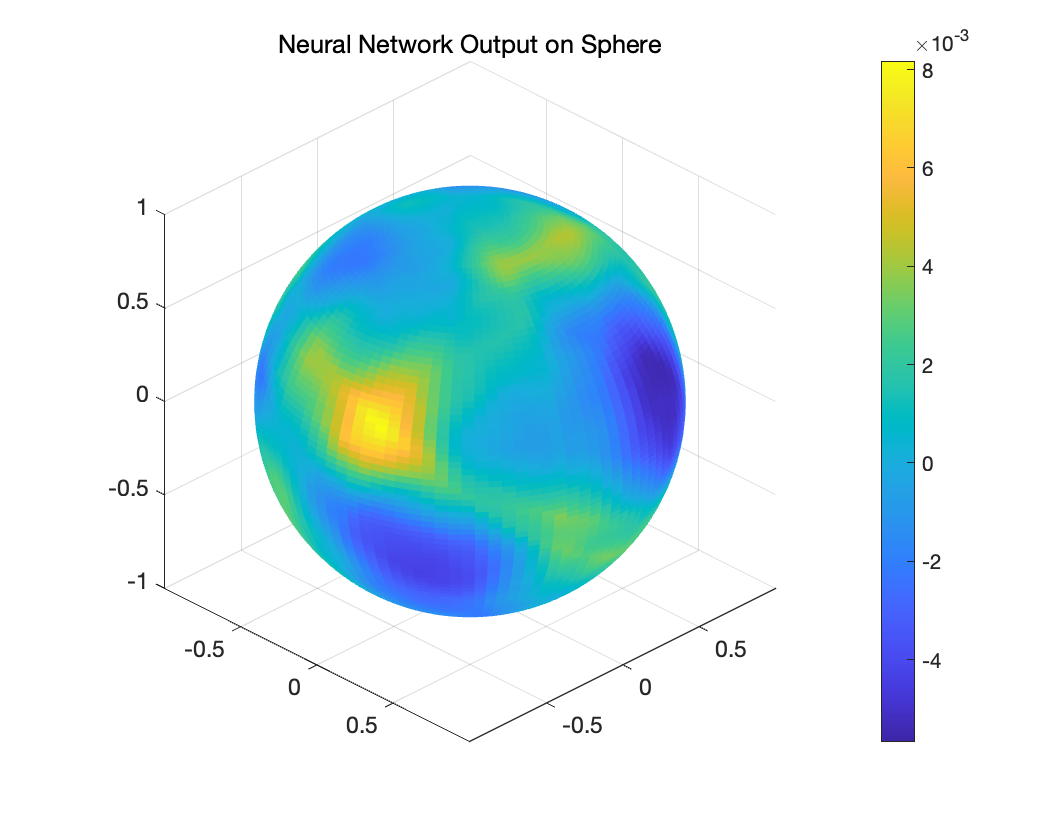}}

   \subfigure[Image of the error between the target function and SNN]{\includegraphics[width=6cm, height =4cm]{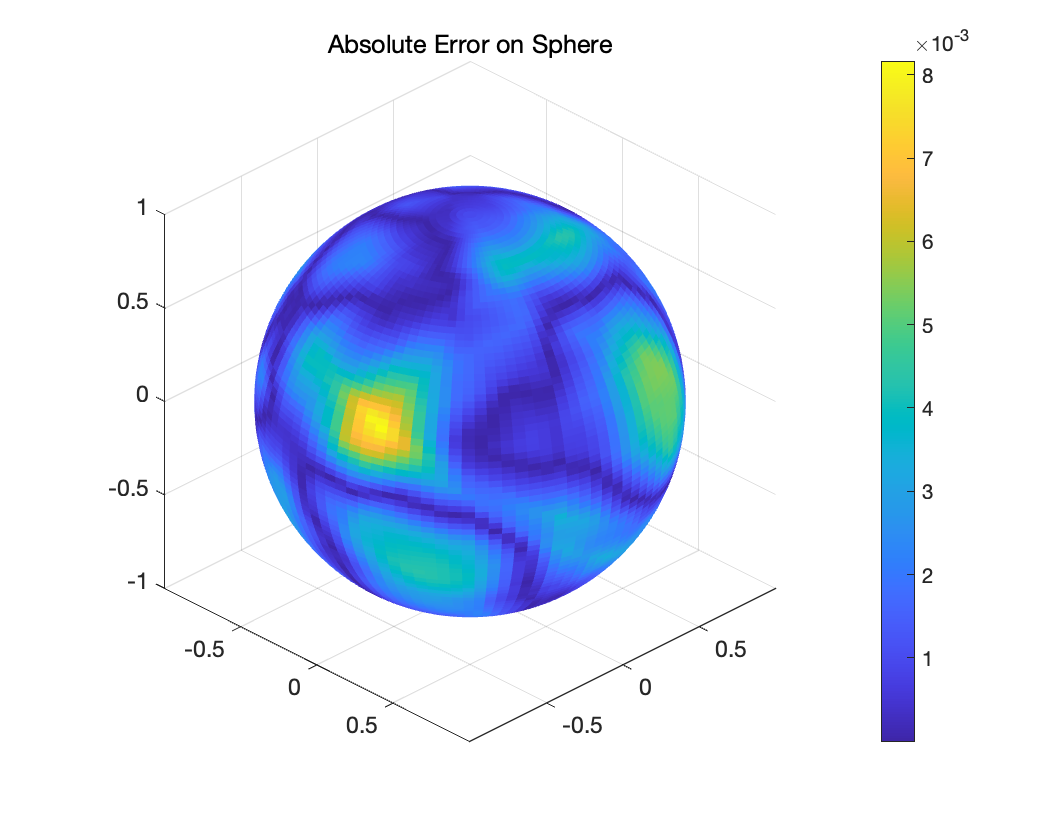}}
    \caption{The target function, SNN output, and error of the sixth test.}
    \label{6testdnntarget}
\end{figure}

In the canonical case, analysis of the trained weights reveals a direct correlation with the frequency principle. The weight distribution in the first layer exhibits a hierarchical structure that naturally favors lower-frequency components. Specifically, the spherical harmonic decomposition of the network output shows (Figure \ref{6testalphabeta} and Figure \ref{6testalphabetam1}):
\begin{itemize}
    \item  The magnitude of weights corresponding to low-frequency components dominates the spectrum.
    \item A systematic decay in weight magnitudes as the frequency increases.
    \item Clear separation between frequency bands in the learning process.
\end{itemize}
This weight configuration results in optimal approximation properties, as evidenced by the uniform distribution of residual errors across the sphere's surface (Figure \ref{6testdnntarget}).

\textbf{Case 2: Hybrid Weight Distribution Pattern}
\begin{figure}
    \centering
    \includegraphics[width=6cm, height =4cm]{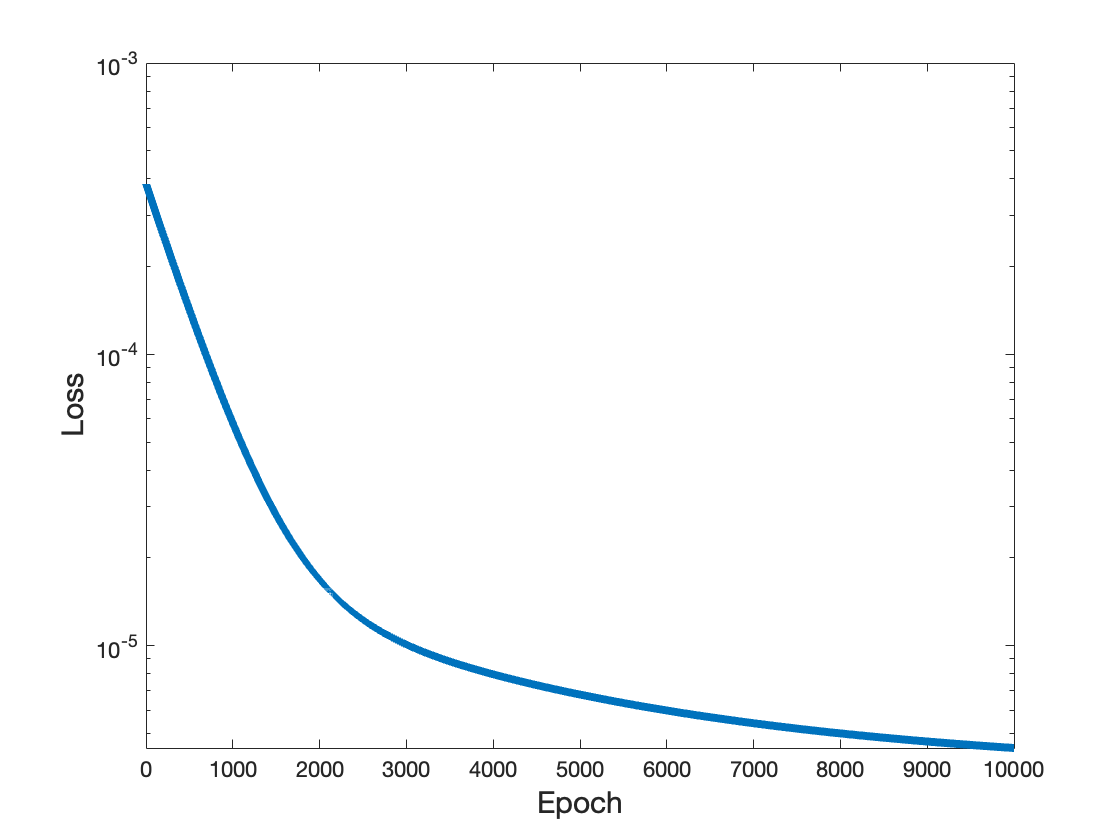}
    \caption{Image of Loss of the seventh test (Best Loss: 4.45e-6).}
    \label{7test}
\end{figure}
\begin{figure}
    \centering
    \subfigure[$l = 1,2,\ldots, 5$, $j=0$]{\includegraphics[width=6cm, height =4cm]{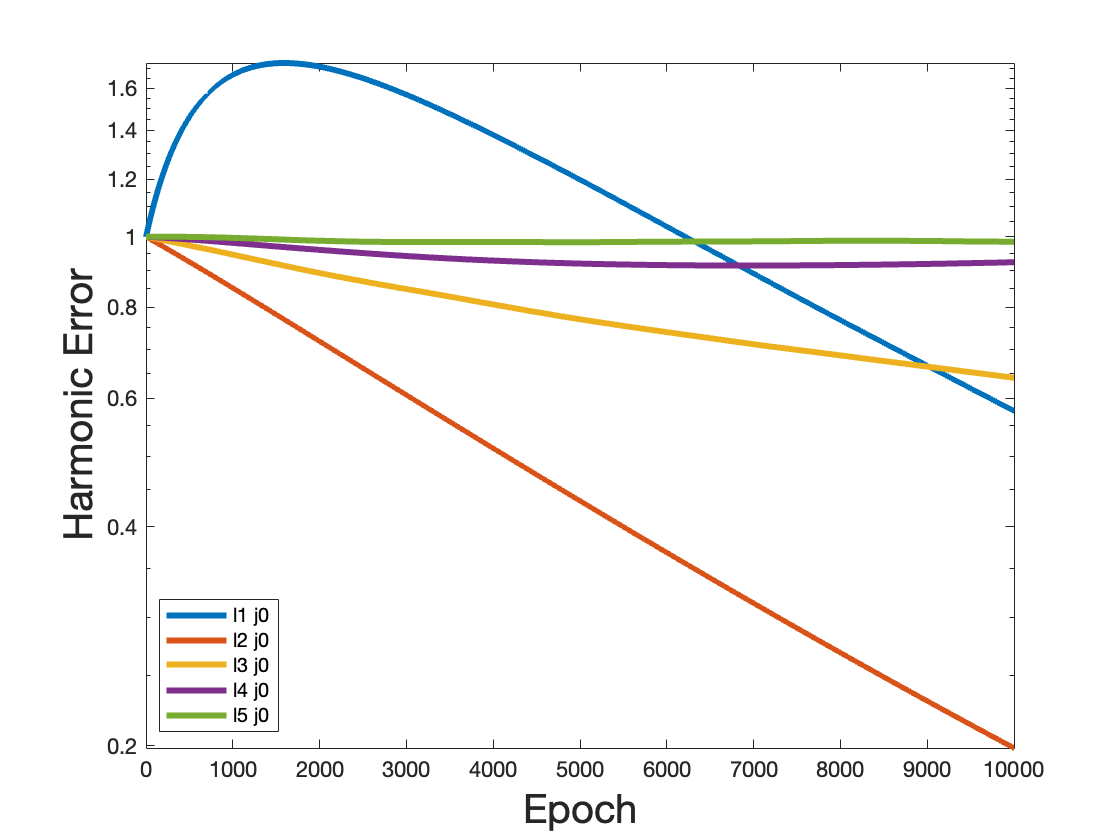}}
    \subfigure[$l = 6,7,\ldots, 10$, $j=0$]{\includegraphics[width=6cm, height =4cm]{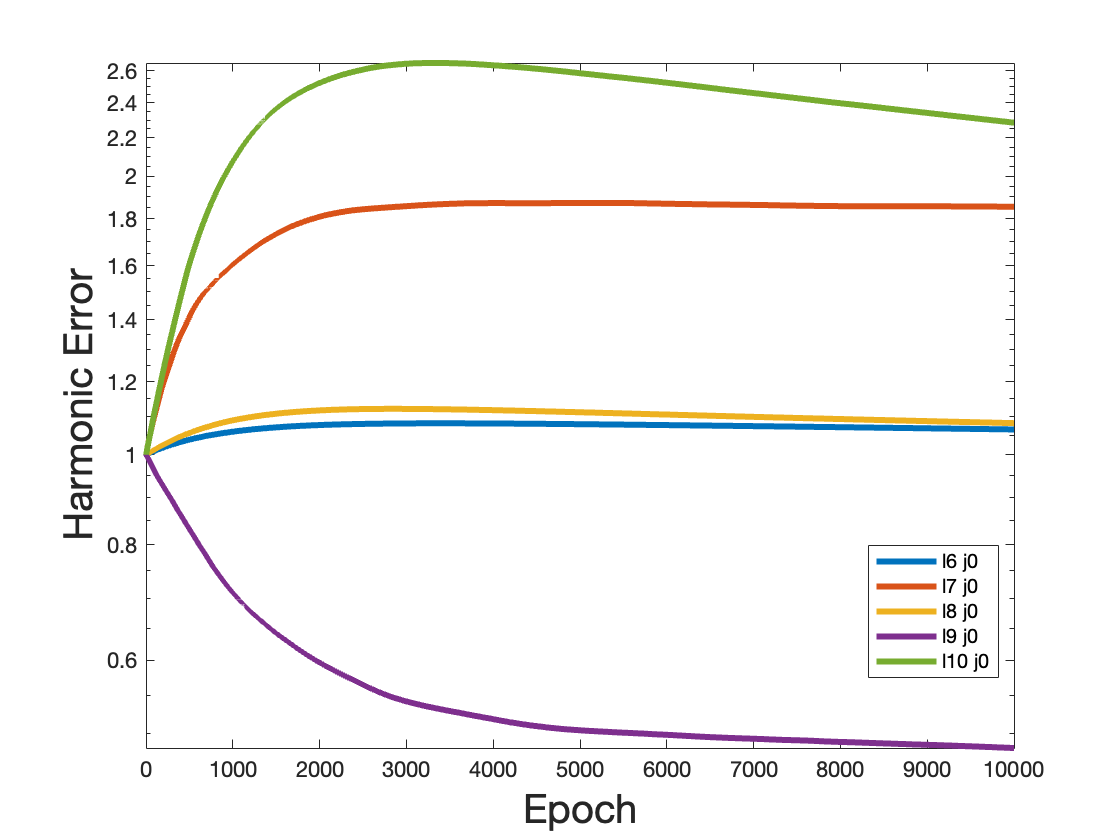}}
    \caption{Harmonic Errors of $c_{lj}$ When $l = 1,2,\ldots, 10$, $j=0$.}
    \label{7testalphabeta}
\end{figure}
\begin{figure}
    \centering
    \subfigure[$l = 1,2,\ldots, 5$, $j=1$]{\includegraphics[width=6cm, height =4cm]{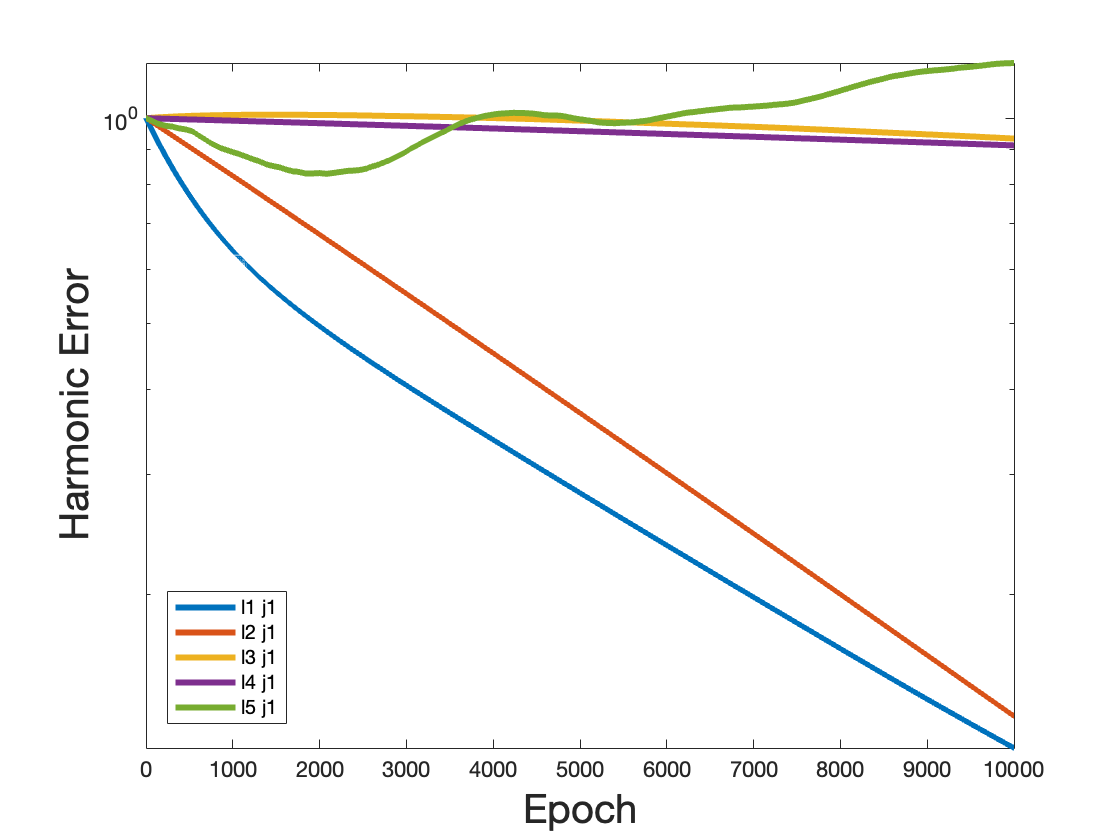}}
    \subfigure[$l = 6,7,\ldots, 10$, $j=1$]{\includegraphics[width=6cm, height =4cm]{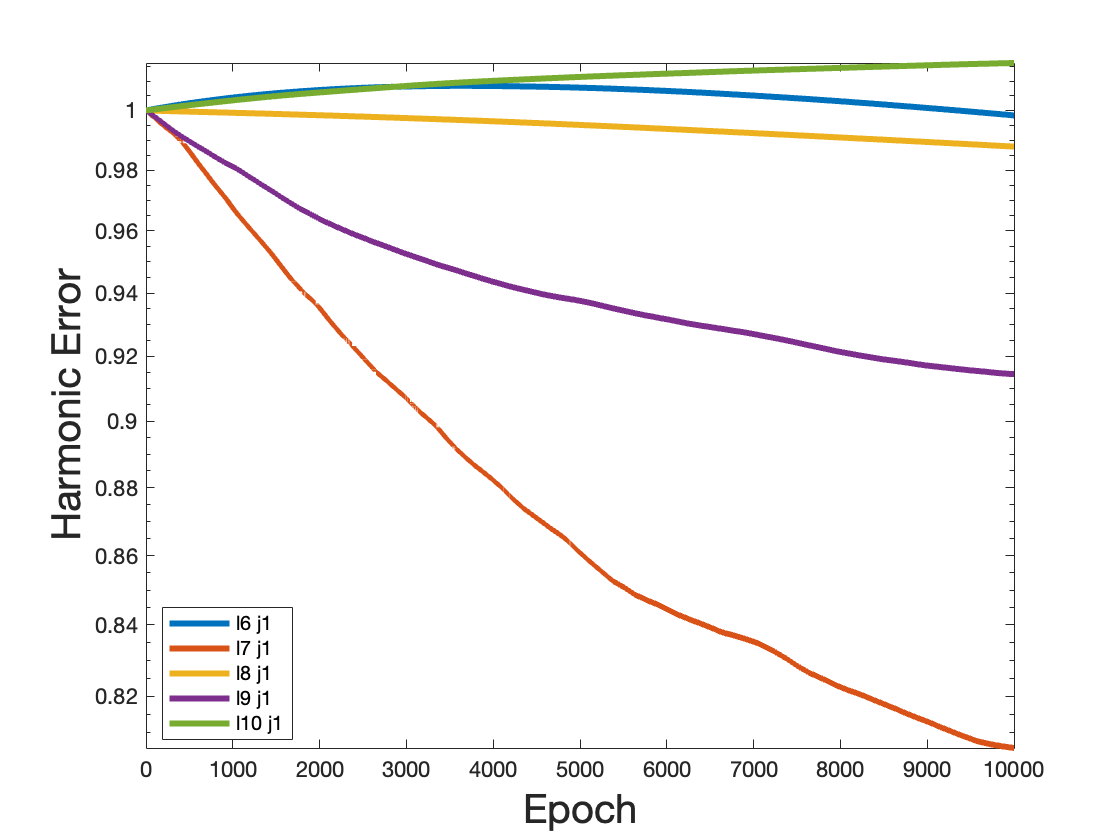}}
    \caption{Harmonic Errors of $c_{lj}$ When $l = 1,2,\ldots, 10$, $j=1$.}
    \label{7testalphabetam1}
\end{figure}

\begin{figure}
    \centering
    \subfigure[Image of the target function]{\includegraphics[width=6cm, height =4cm]{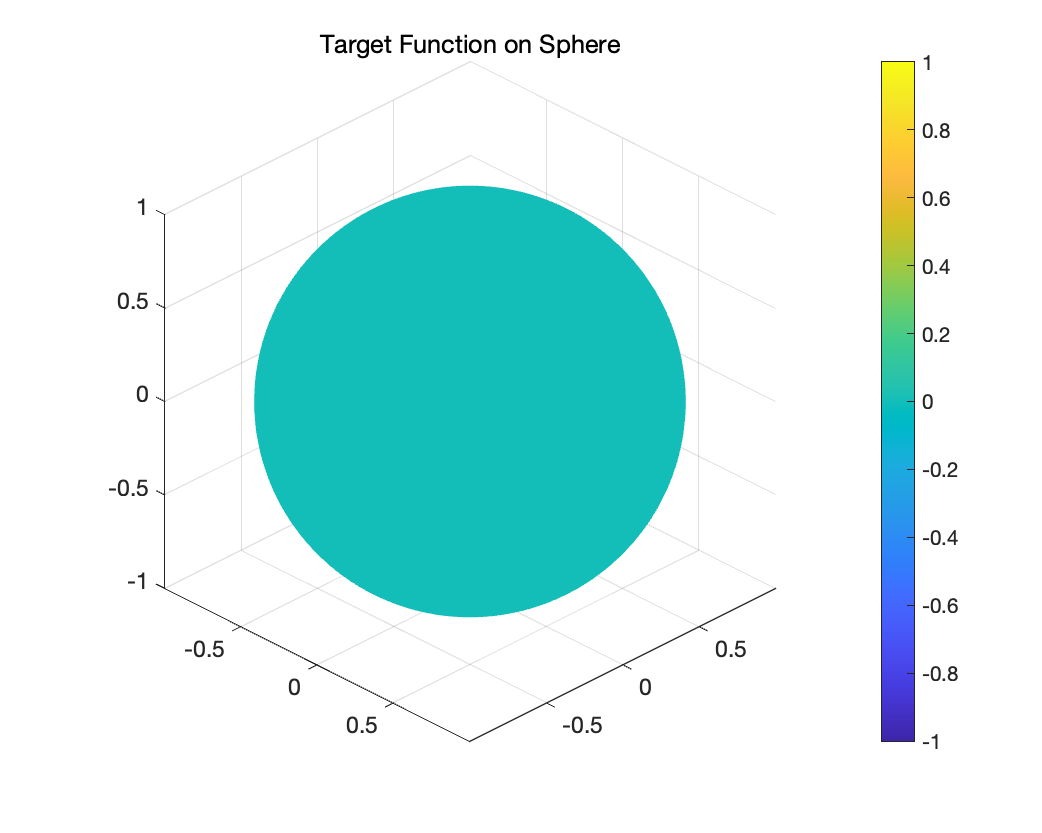}}
    \subfigure[Image of the SNN output]{\includegraphics[width=6cm, height =4cm]{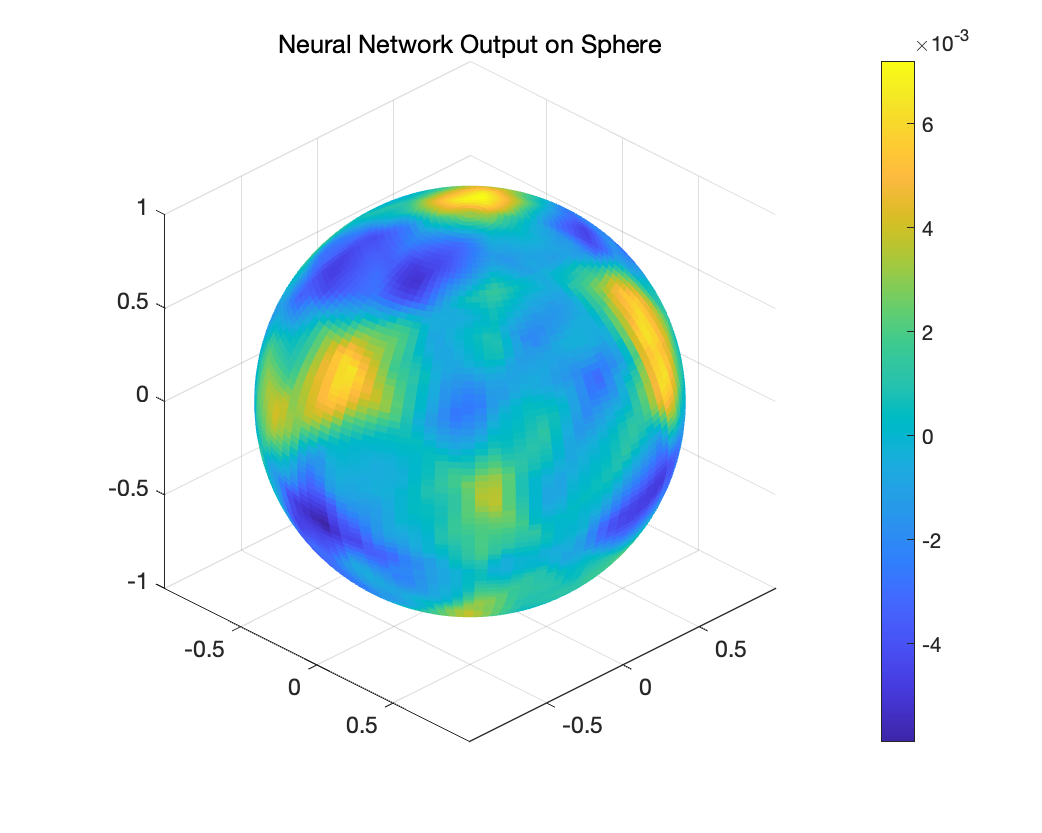}}

   \subfigure[Image of the error between the target function and SNN]{\includegraphics[width=6cm, height =4cm]{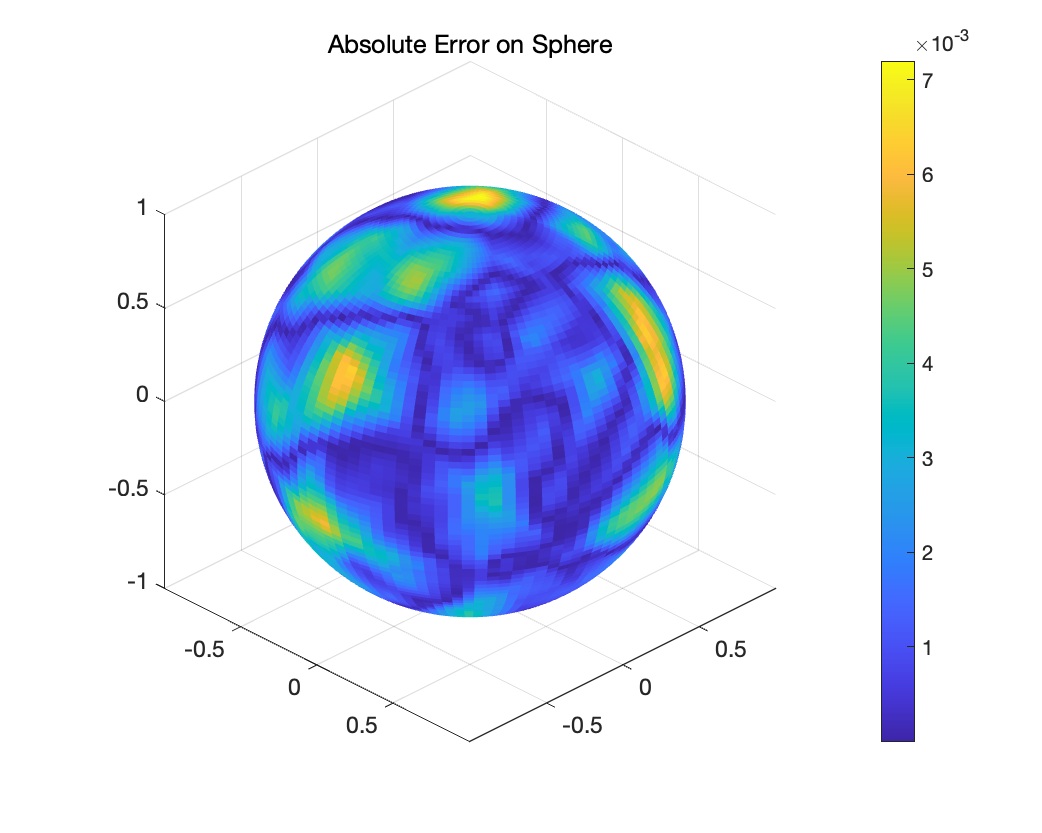}}
    \caption{The target function, SNN output, and error of the seventh test.}
    \label{7testdnntarget}
\end{figure}

The hybrid case presents an intriguing intermediate scenario where the weight distribution exhibits mixed frequency characteristics. Analysis of the trained weights reveals (Figure \ref{7testalphabeta} and Figure \ref{7testalphabetam1}):
\begin{itemize}
    \item Non-monotonic distribution of weight magnitudes across frequency components.
    \item Simultaneous activation of both low and high-frequency modes.
    \item Localized clustering of weights in frequency space.
\end{itemize}

This configuration demonstrates how weight initialization and training dynamics can lead to a partial adherence to the frequency principle, resulting in more complex learning patterns than previously theorized (Figure \ref{7testdnntarget}).

\textbf{Case 3: Anomalous Frequency Learning Pattern}
\begin{figure}
    \centering
    \includegraphics[width=6cm, height =4cm]{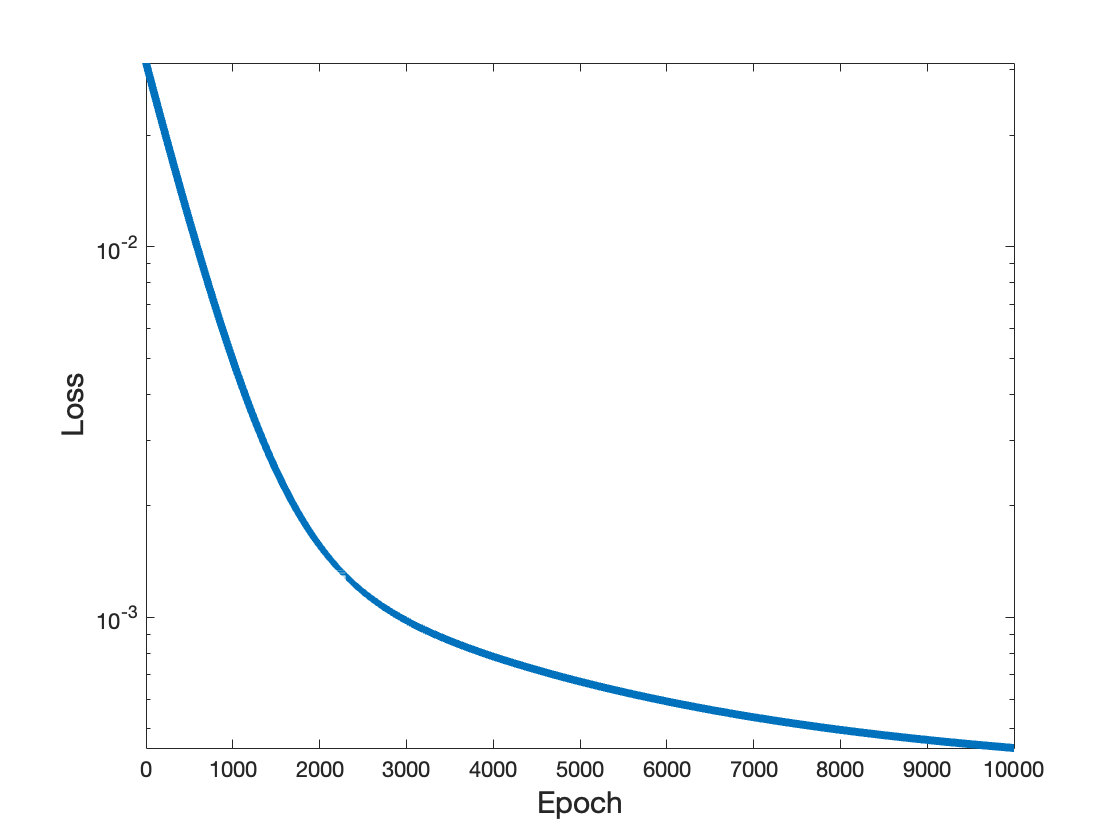}
    \caption{Image of Loss of the eighth test (Best Loss: 4.45e-4).}
    \label{8test}
\end{figure}
\begin{figure}
    \centering
    \subfigure[$l = 1,2,\ldots, 5$, $j=0$]{\includegraphics[width=6cm, height =4cm]{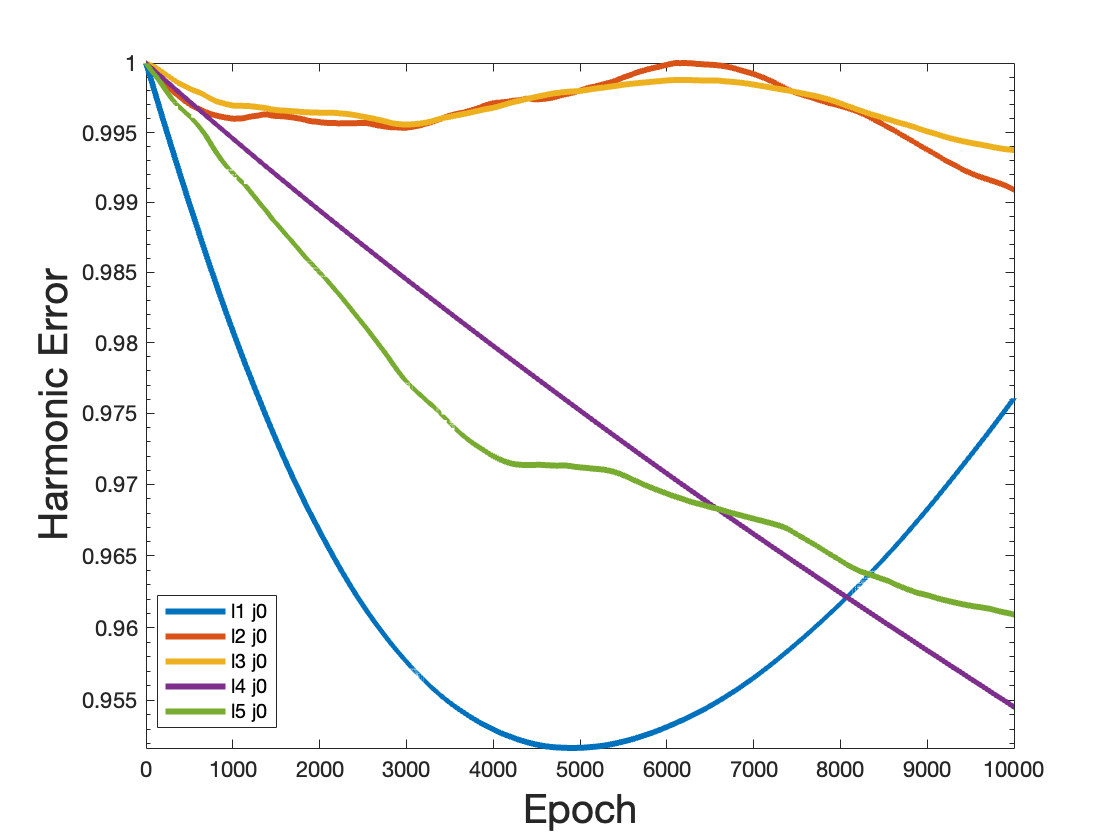}}
    \subfigure[$l = 6,7,\ldots, 10$, $j=0$]{\includegraphics[width=6cm, height =4cm]{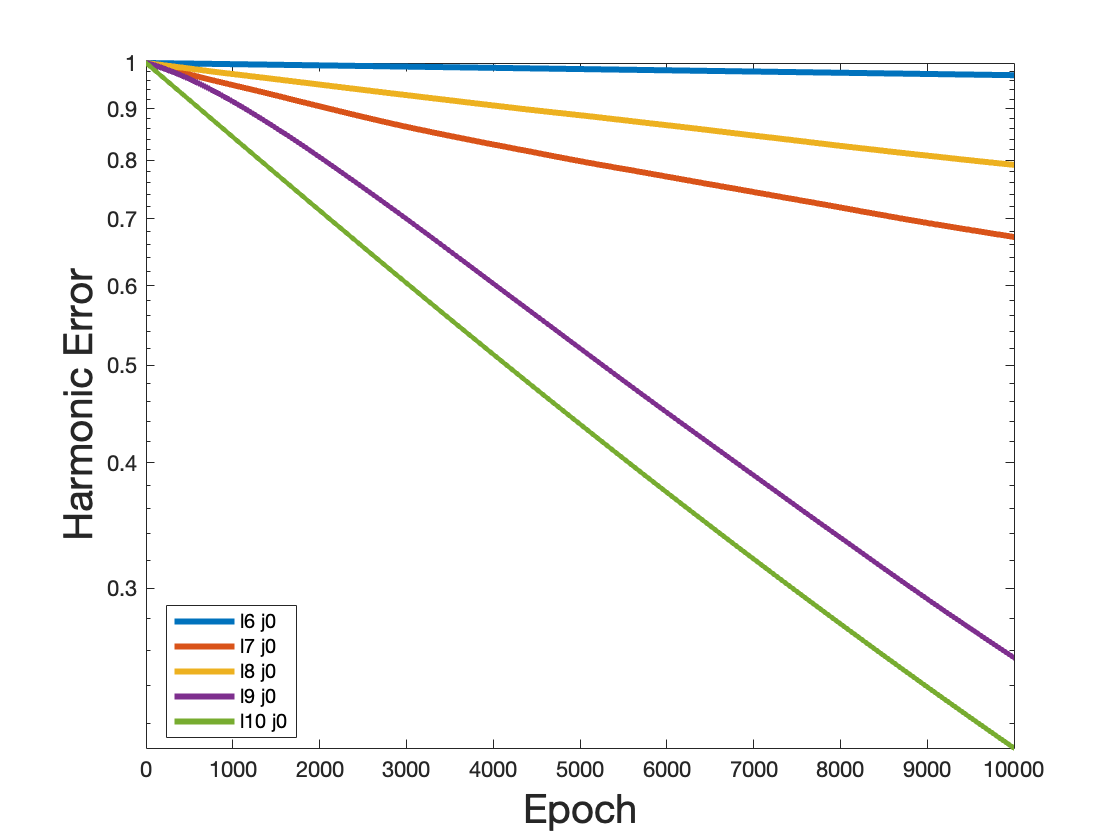}}
    \caption{Harmonic Errors of $c_{lj}$ When $l = 1,2,\ldots, 10$, $j=0$.}
    \label{8testalphabeta}
\end{figure}
\begin{figure}
    \centering
    \subfigure[$l = 1,2,\ldots, 5$, $j=1$]{\includegraphics[width=6cm, height =4cm]{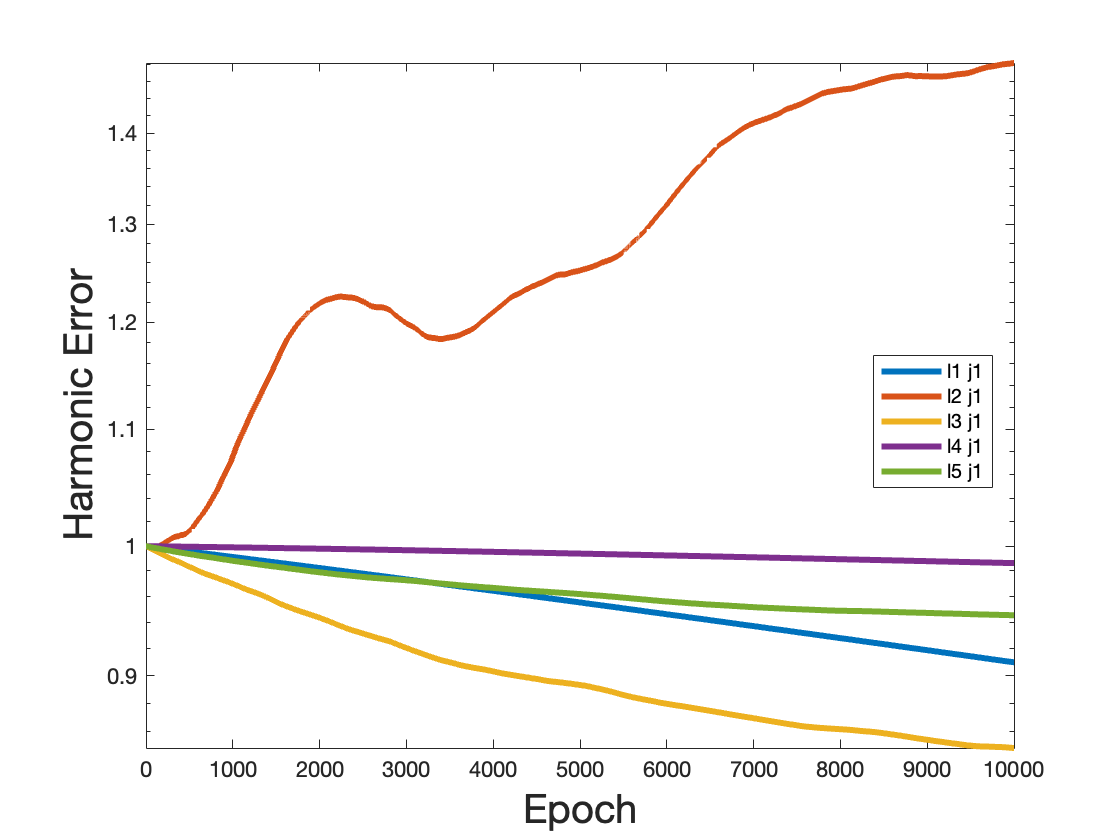}}
    \subfigure[$l = 6,7,\ldots, 10$, $j=1$]{\includegraphics[width=6cm, height =4cm]{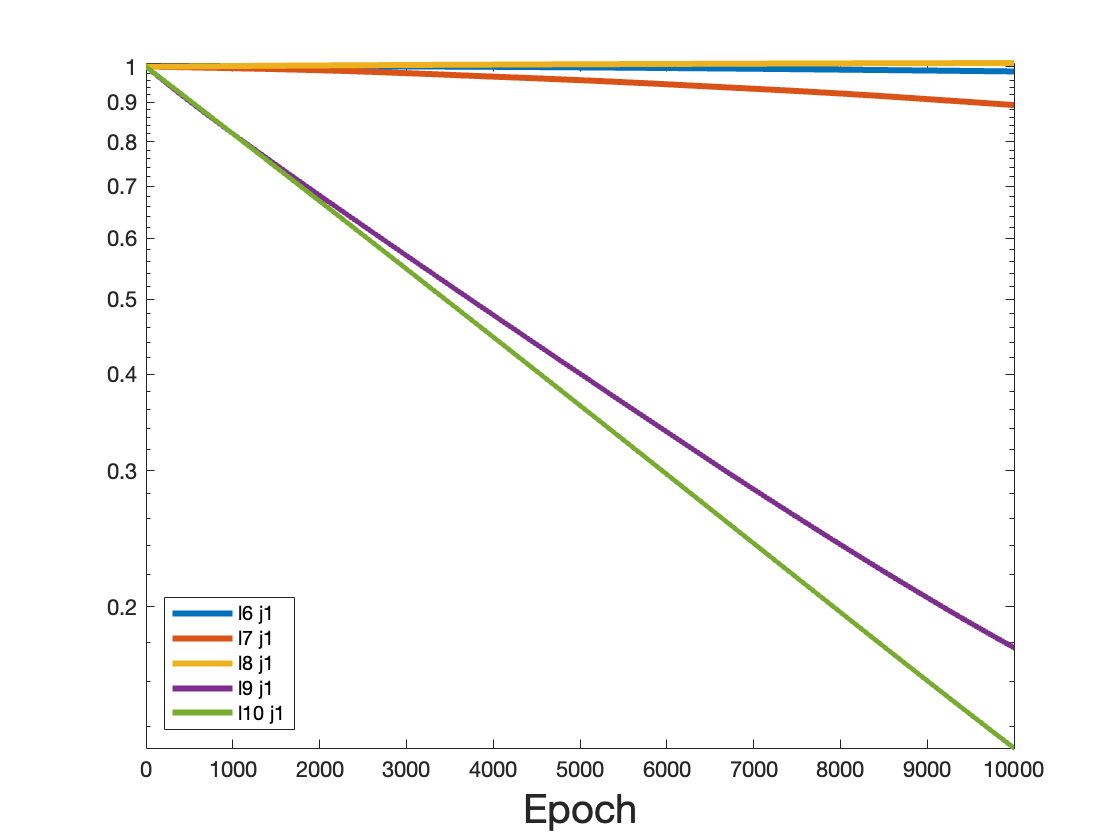}}
    \caption{Harmonic Errors of $c_{lj}$ When $l = 1,2,\ldots, 10$, $j=1$.}
    \label{8testalphabetam1}
\end{figure}

\begin{figure}
    \centering
    \subfigure[Image of the target function]{\includegraphics[width=6cm, height =4cm]{fqsphere/target_function_u0fp.png}}
    \subfigure[Image of the SNN output]{\includegraphics[width=6cm, height =4cm]{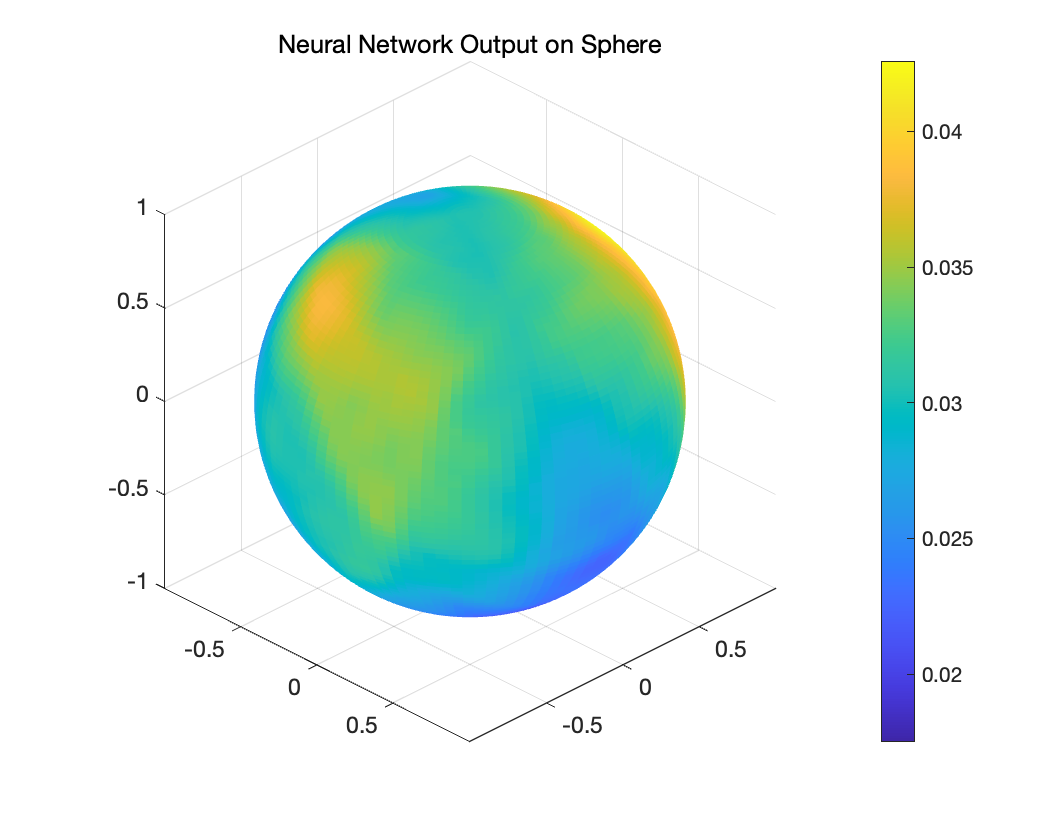}}

   \subfigure[Image of the error between the target function and SNN]{\includegraphics[width=6cm, height =4cm]{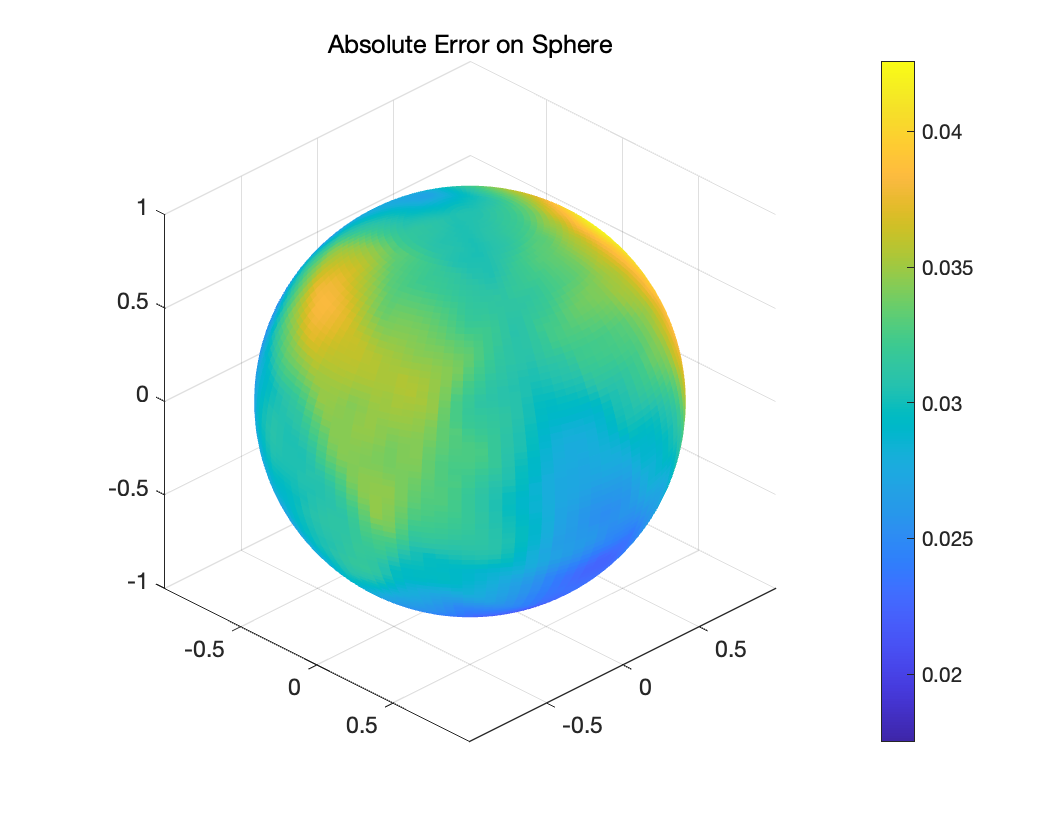}}
    \caption{The target function, SNN output, and error of the eighth test.}
    \label{8testdnntarget}
\end{figure}

Like Section \ref{Sec:SphericalAnalysisFW}, we give a high-frequency initialization $\sin(10\tau)\cos(10\phi)$ as Figure \ref{3testhighinitial} shown. In the most striking case, the trained weights generate a frequency response that contradicts conventional frequency principle expectations. Detailed analysis shows (Figure \ref{8testalphabeta} and Figure \ref{8testalphabetam1}):
\begin{itemize}
    \item Preferential weighting of high-frequency components.
    \item Suppression of low-frequency modes despite their fundamental importance.
\end{itemize}

This unexpected weight configuration challenges our understanding of neural network learning dynamics on manifolds and suggests the existence of alternative stable solutions in the optimization landscape.

The observed weight distributions and their corresponding frequency characteristics provide several key insights:
\begin{itemize}
    \item The relationship between network weights and frequency components is more complex than previously understood.
    \item Initial weight configurations significantly influence the final frequency learning pattern.
    \item The optimization landscape contains multiple stable solutions with distinct frequency characteristics.
\end{itemize}

These findings extend classical frequency principle theory by demonstrating how weight configurations directly influence the network's frequency learning behavior. Furthermore, they suggest that the frequency principle might be better understood as a property of typical optimization trajectories rather than an inherent characteristic of neural networks.

The analysis of weight-induced frequency patterns raises important computational considerations:
\begin{itemize}
    \item The choice of weight initialization schemes significantly impacts the final frequency distribution.
    \item Training dynamics exhibit sensitivity to both learning rate and batch size.
    \item  The stability of different frequency learning patterns varies with network architecture.
\end{itemize}
These observations provide practical guidelines for controlling frequency learning behavior through careful selection of network parameters and training protocols.

\subsubsection{$u(\tau,\phi) = \sin(\tau)\cos(3\phi) + \sin(3\tau)\cos(5\phi)$}

We extend our analysis to examine the learning dynamics when both weights and biases are trained simultaneously. The target function remains $u(\tau,\phi) = \sin(\tau)\cos(3\phi) + \sin(3\tau)\cos(5\phi)$, maintaining consistency with our fixed-weight analysis. The neural network architecture employs 100 neurons in the hidden layer with ReLU activation functions, but now all parameters are updated during training. We utilize SGD with a learning rate of $1e{-3}$ over 10,000 epochs, training on 100 uniformly distributed sampling points on the sphere.

\textbf{Case 1: Partial Adherence to Frequency Principle}
\begin{figure}
    \centering
    \includegraphics[width=6cm, height=4cm]{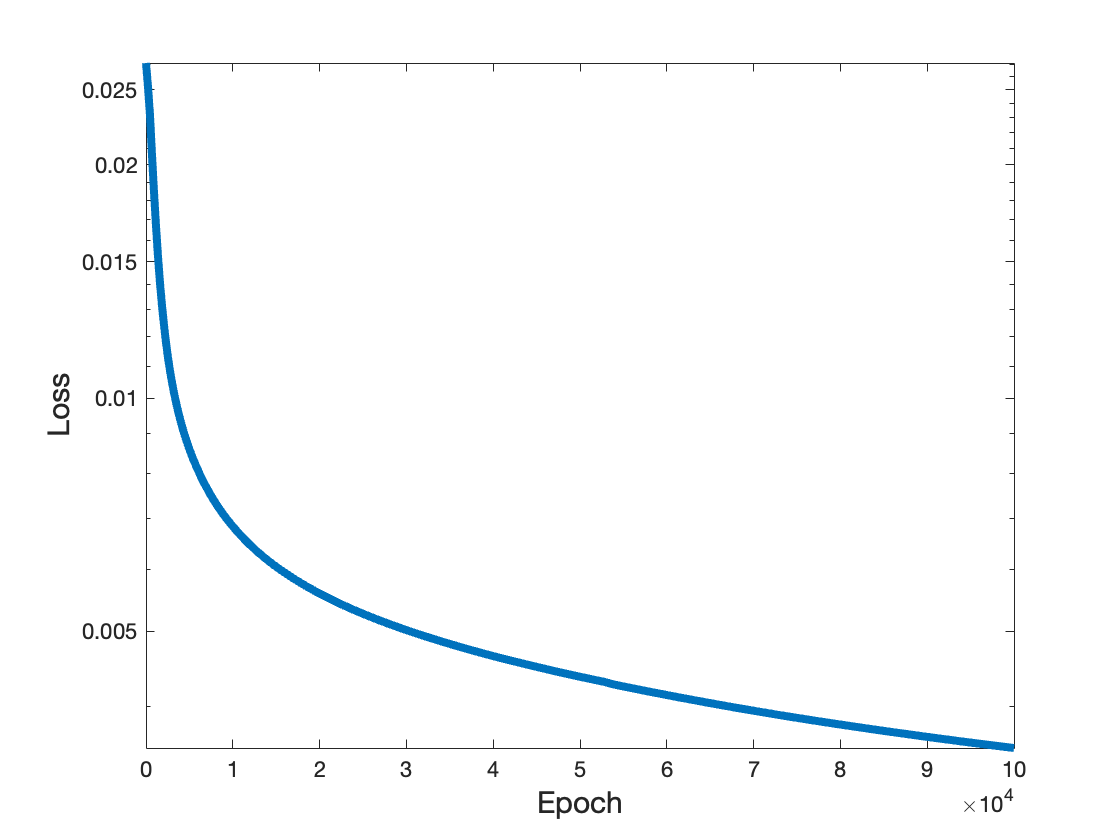}
    \caption{Training loss evolution for trainable weights showing mixed convergence patterns}
    \label{trigtrain1test}
\end{figure}

\begin{figure}
    \centering
    \subfigure[$l = 1,3,5$, $j=0$]{\includegraphics[width=6cm, height=4cm]{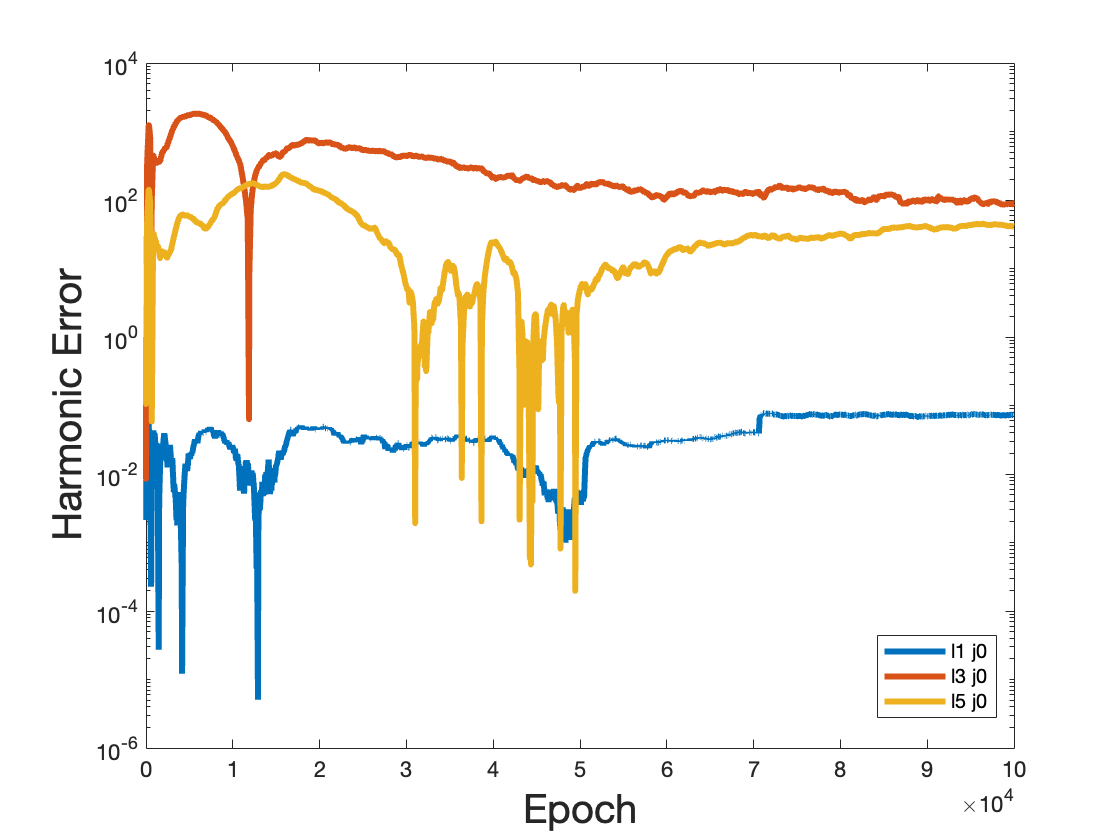}}
    
    \subfigure[$l = 1,2,\ldots,5$, $j=0$]{\includegraphics[width=6cm, height=4cm]{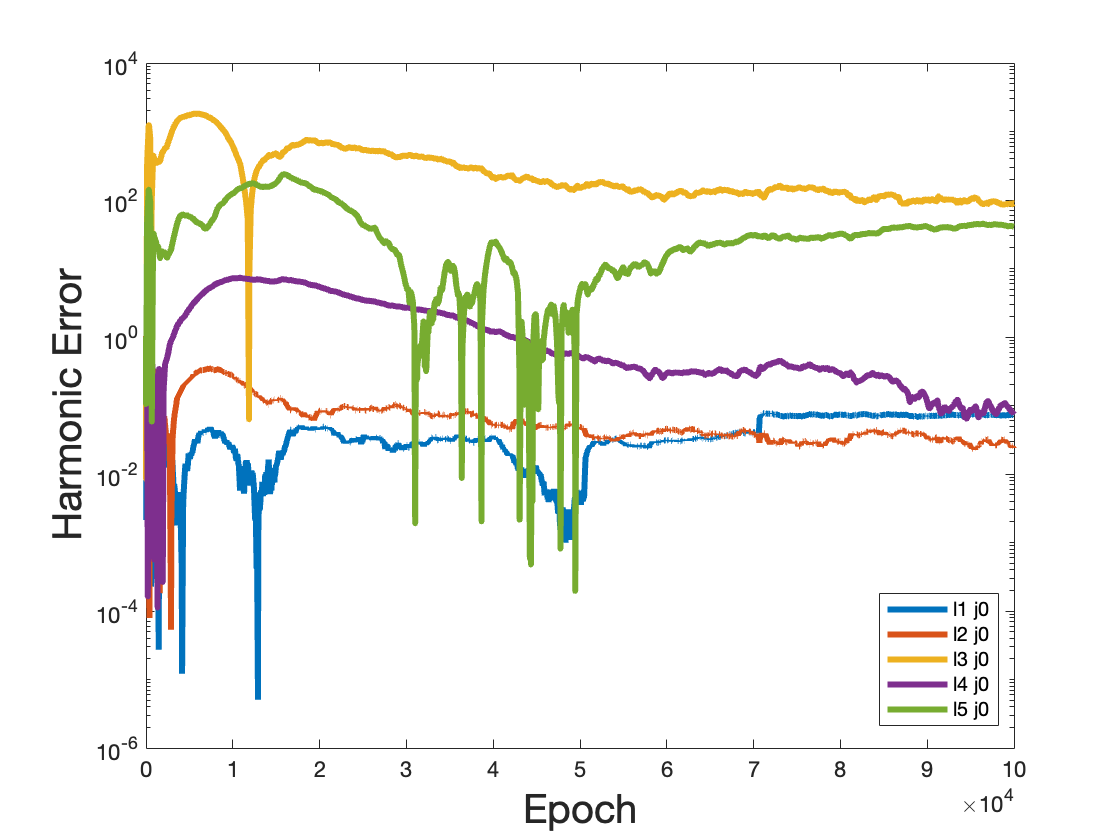}}
    \subfigure[$l = 6,7,\ldots,10$, $j=0$]{\includegraphics[width=6cm, height=4cm]{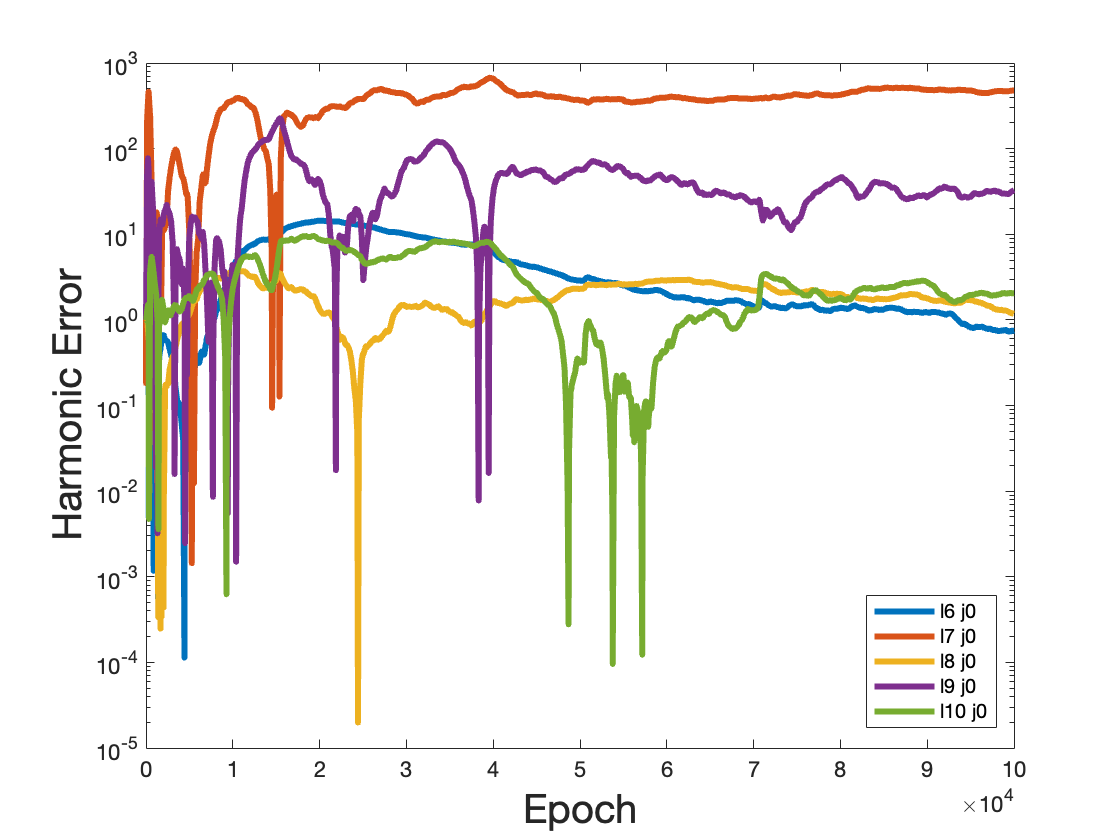}}
    \caption{Harmonic Errors Revealing Interleaved Frequency Learning}
    \label{trigtrain1testharmonic}
\end{figure}

\begin{figure}
    \centering
    \subfigure[$l = 1,3,5$, $j=1$]{\includegraphics[width=6cm, height=4cm]{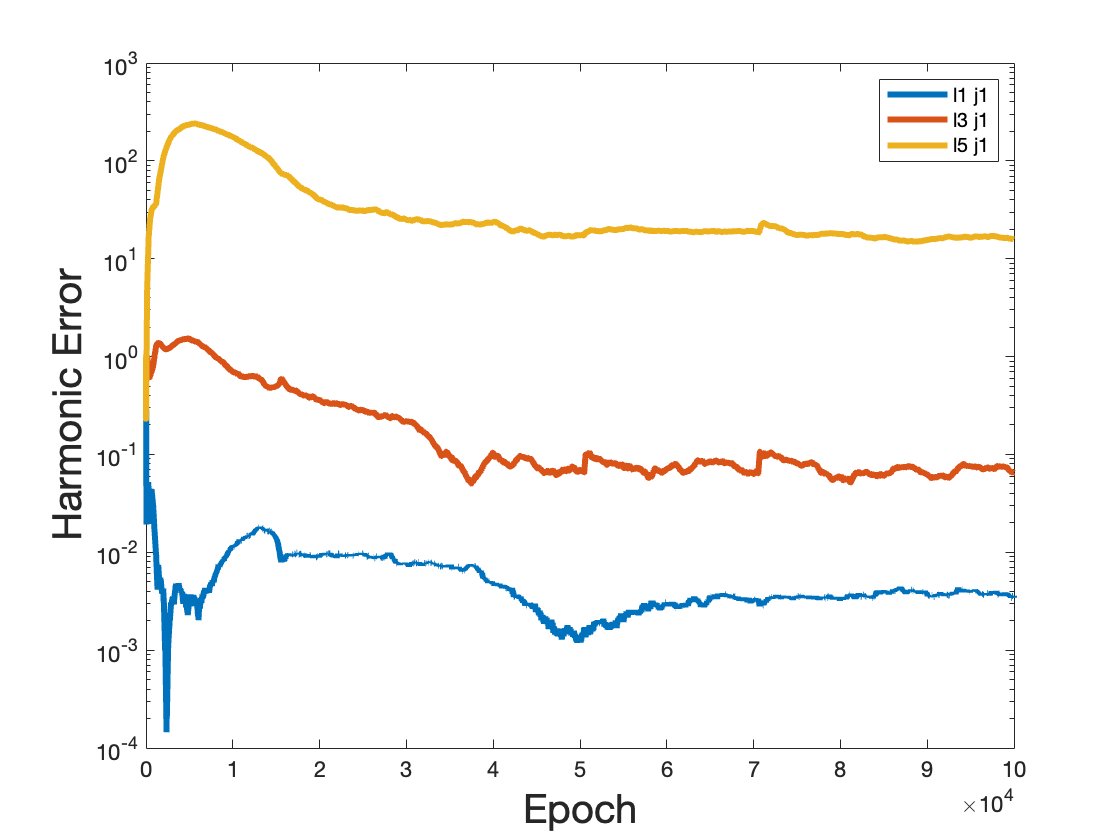}}
    
    \subfigure[$l = 1,2,\ldots,5$, $j=1$]{\includegraphics[width=6cm, height=4cm]{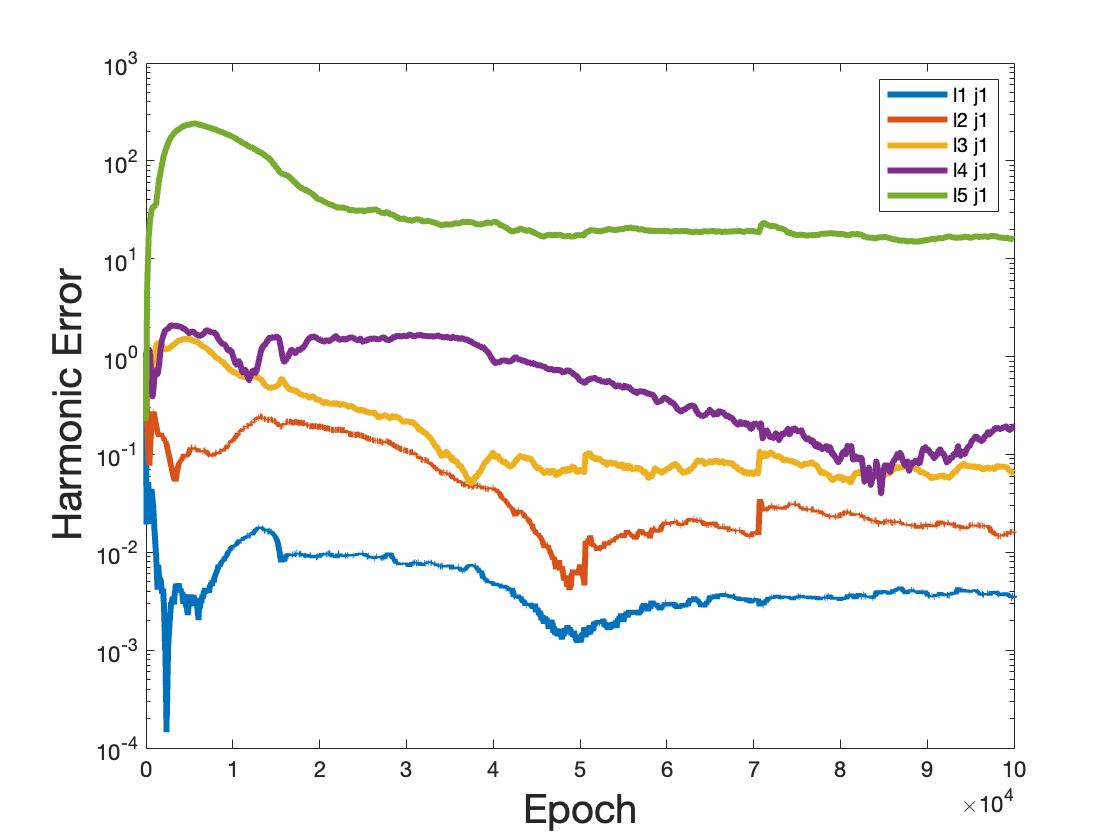}}
    \subfigure[$l = 6,7,\ldots,10$, $j=1$]{\includegraphics[width=6cm, height=4cm]{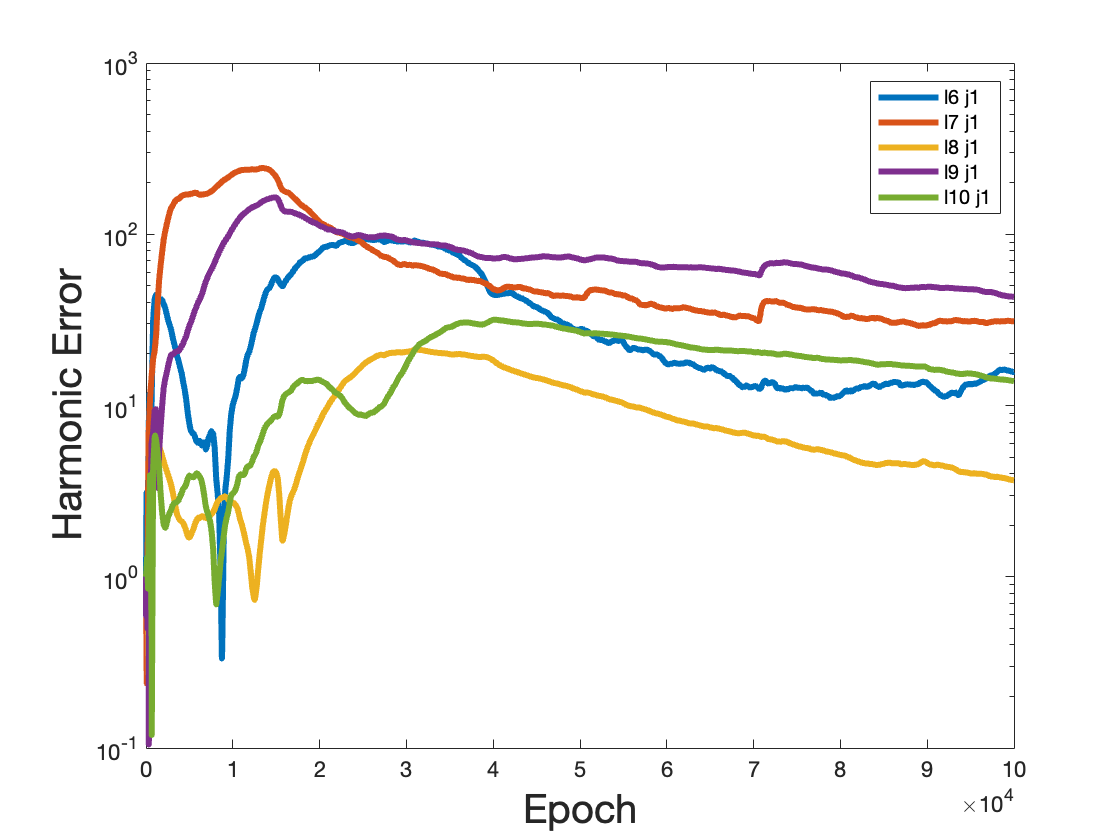}}
    \caption{Harmonic Errors Revealing Interleaved Frequency Learning}
    \label{trigtrain1testharmonicm1}
\end{figure}

\begin{figure}
    \centering
    \subfigure[Target function]{\includegraphics[width=6cm, height=4cm]{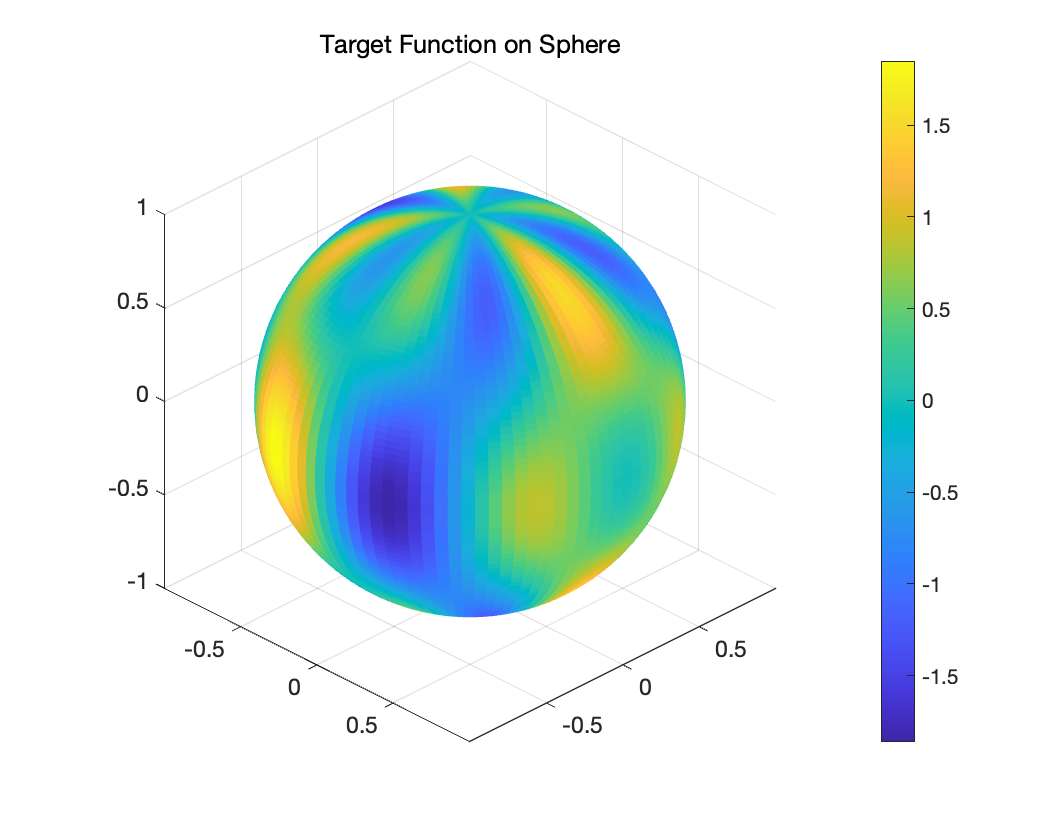}}
    \subfigure[Neural network output]{\includegraphics[width=6cm, height=4cm]{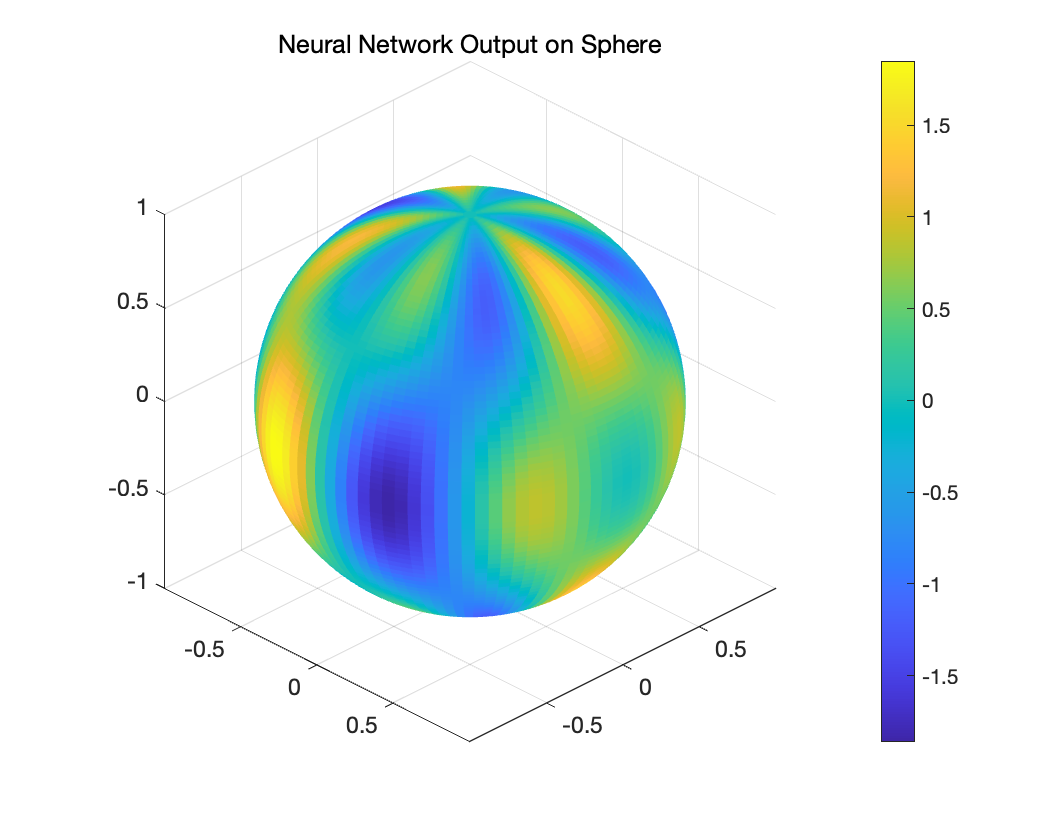}}
    \subfigure[Absolute error distribution]{\includegraphics[width=6cm, height=4cm]{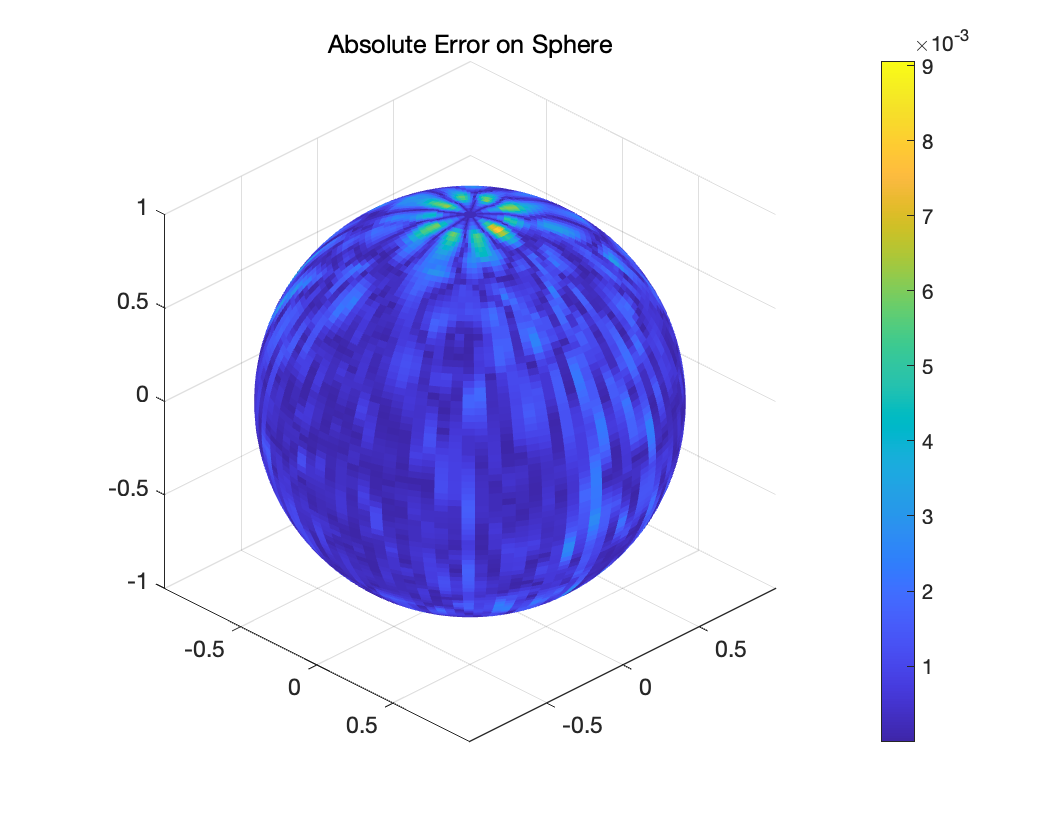}}
    \caption{Spatial comparison showing mixed-frequency approximation}
    \label{trigtrain1testvisual}
\end{figure}

With trainable weights, the network achieves a final loss of 3.2e-3 (Figure \ref{trigtrain1test}), demonstrating improved convergence compared to the fixed-weight case. The harmonic error analysis (Figure \ref{trigtrain1testharmonic}) reveals an intriguing pattern where certain high-frequency components are learned concurrently with their low-frequency counterparts. This behavior is particularly evident in the simultaneous learning of $\sin(\tau)$ and $\sin(3\tau)$ components, suggesting that weight adaptation enables more flexible frequency learning patterns.

The spatial error distribution (Figure \ref{trigtrain1testvisual}) exhibits a more uniform structure compared to the fixed-weight case, indicating that weight training allows the network to balance the approximation of different frequency components better. However, the learning trajectory still shows preferential treatment of certain frequencies, maintaining partial adherence to the frequency principle.

\textbf{Case 2: Contradiction to Frequency Principle}
\begin{figure}
    \centering
    \includegraphics[width=6cm, height=4cm]{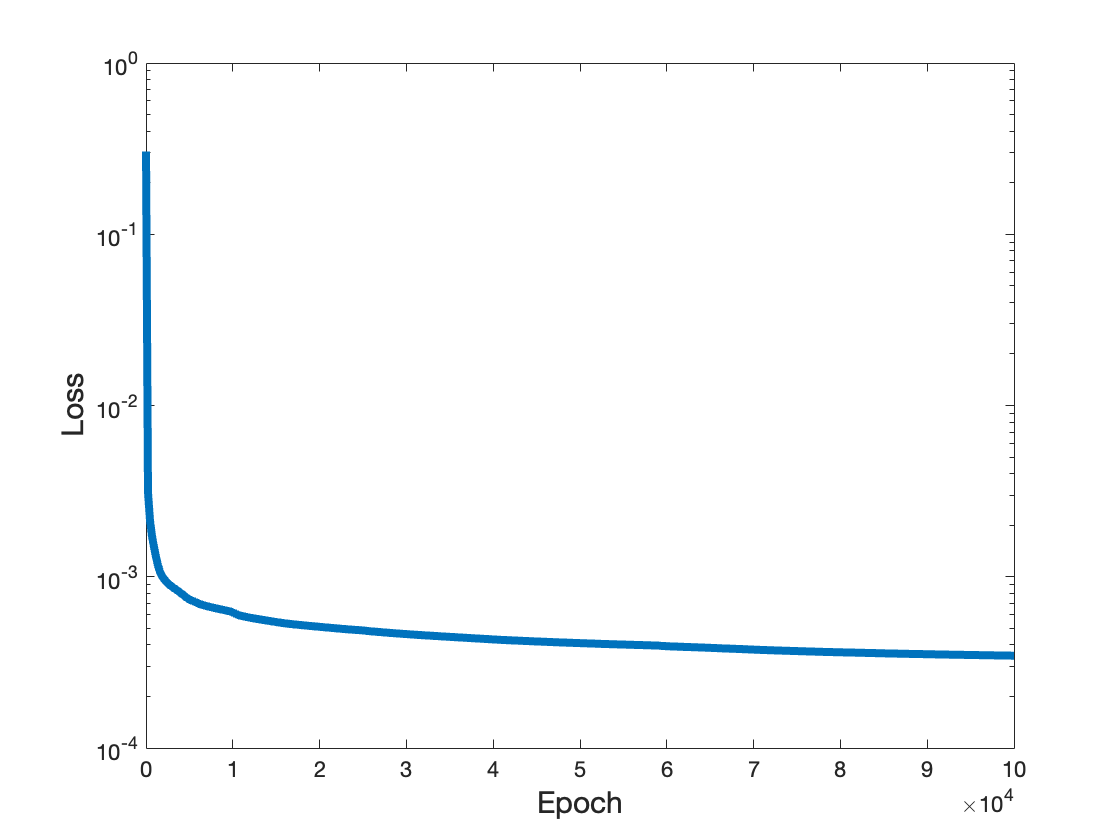}
    \caption{Training loss evolution showing inverse frequency learning}
    \label{trigtrain3test}
\end{figure}

\begin{figure}
    \centering
    \subfigure[$l = 1,3,5$, $j=0$]{\includegraphics[width=6cm, height=4cm]{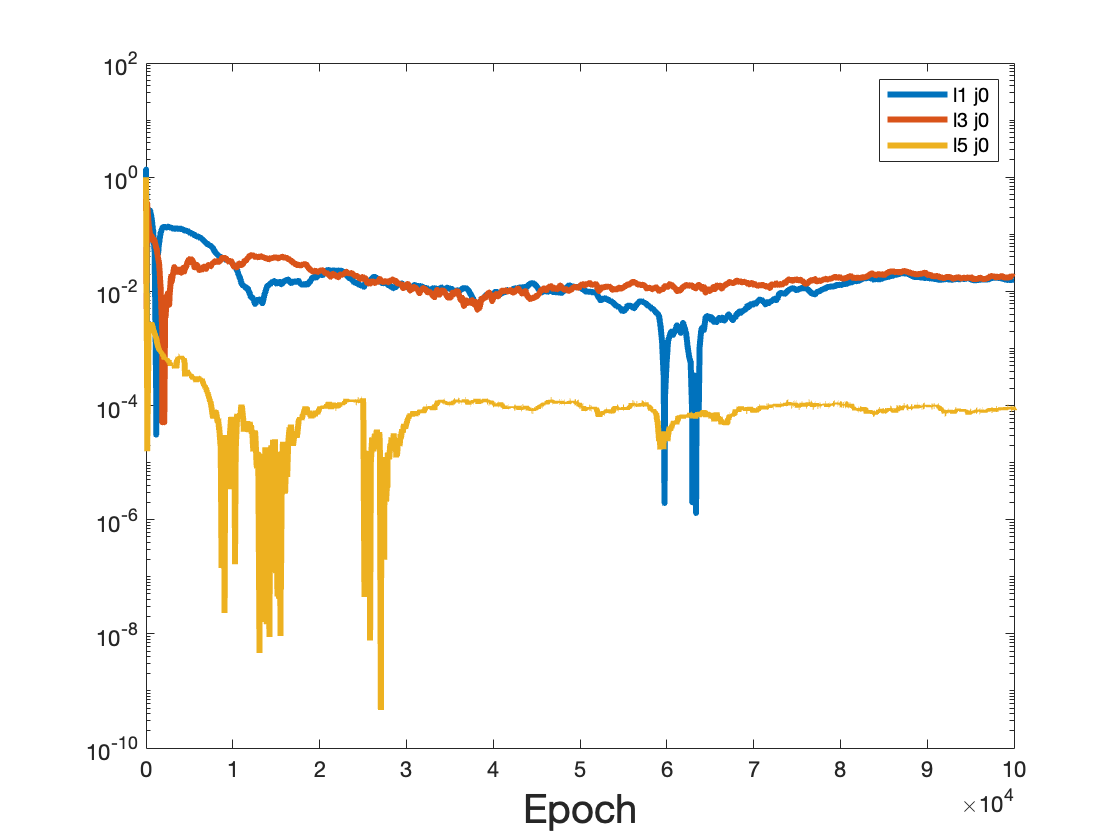}}
    
    \subfigure[$l = 1,2,\ldots,5$, $j=0$]{\includegraphics[width=6cm, height=4cm]{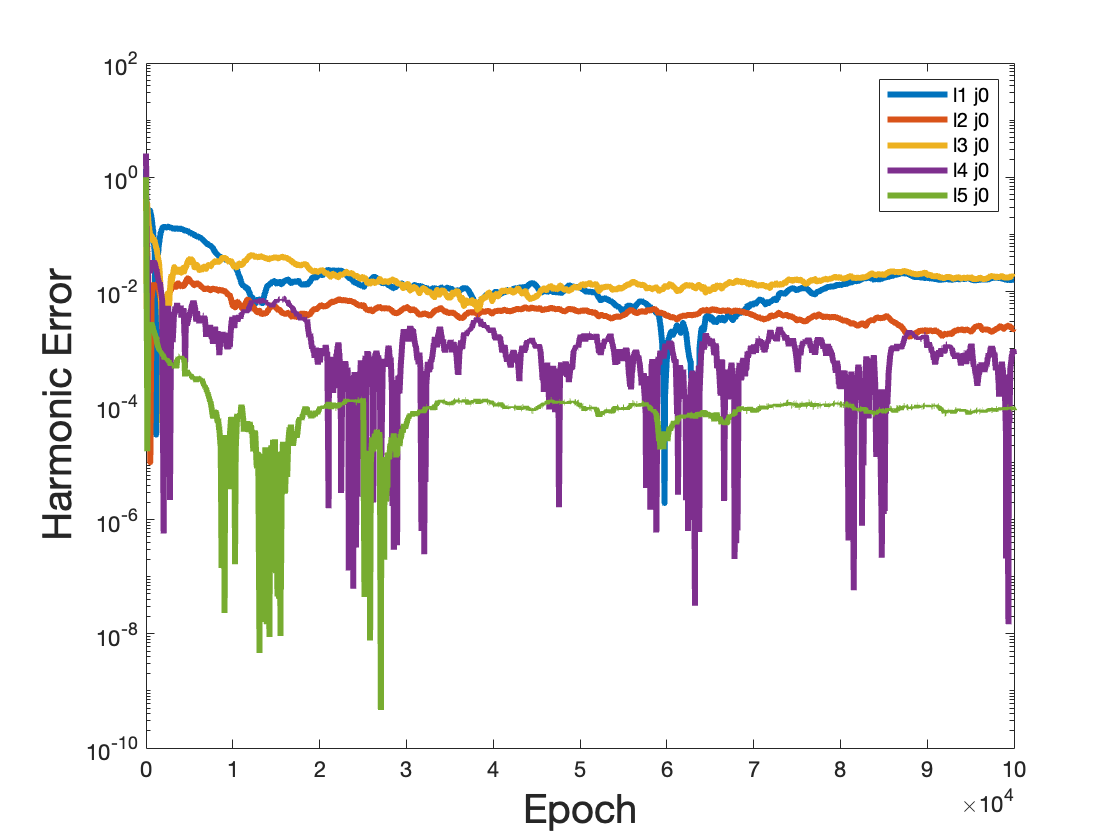}}
    \subfigure[$l = 6,7,\ldots,10$, $j=0$]{\includegraphics[width=6cm, height=4cm]{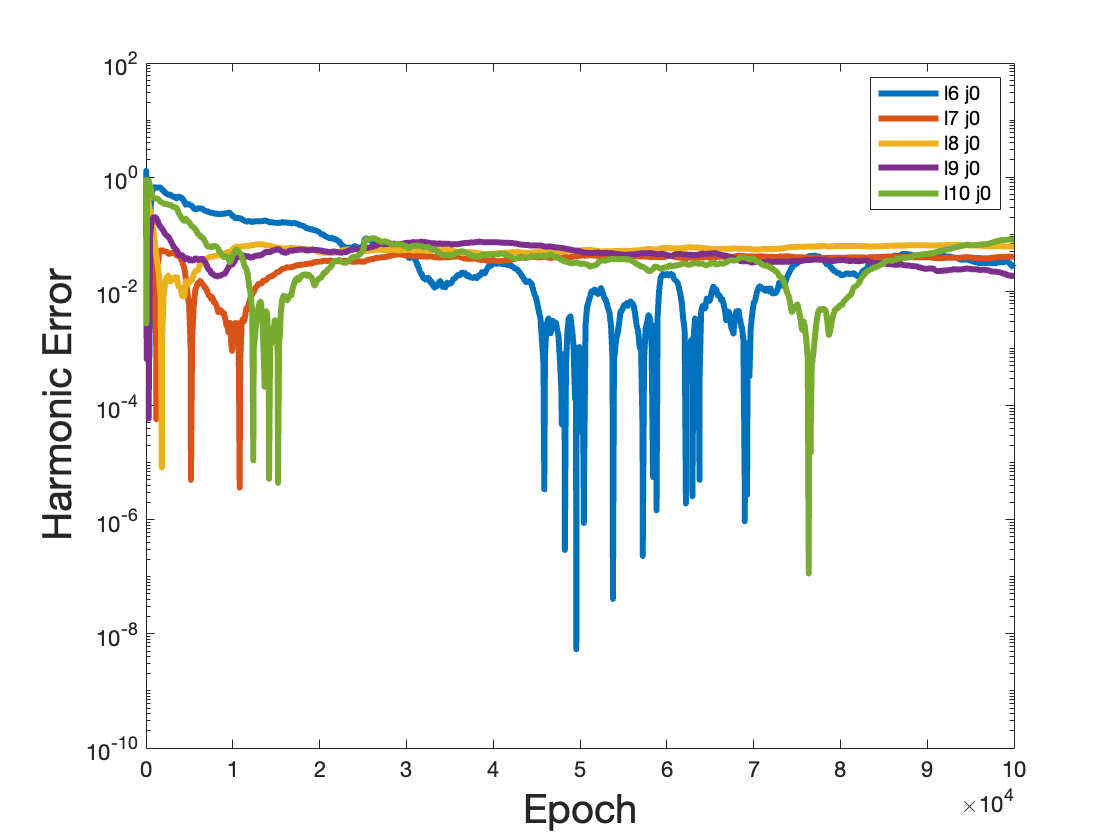}}
    \caption{Harmonic Errors showing prioritized high-frequency learning}
    \label{trigtrain3testharmonic}
\end{figure}
\begin{figure}
    \centering
    \subfigure[$l = 1,3,5$, $j=1$]{\includegraphics[width=6cm, height=4cm]{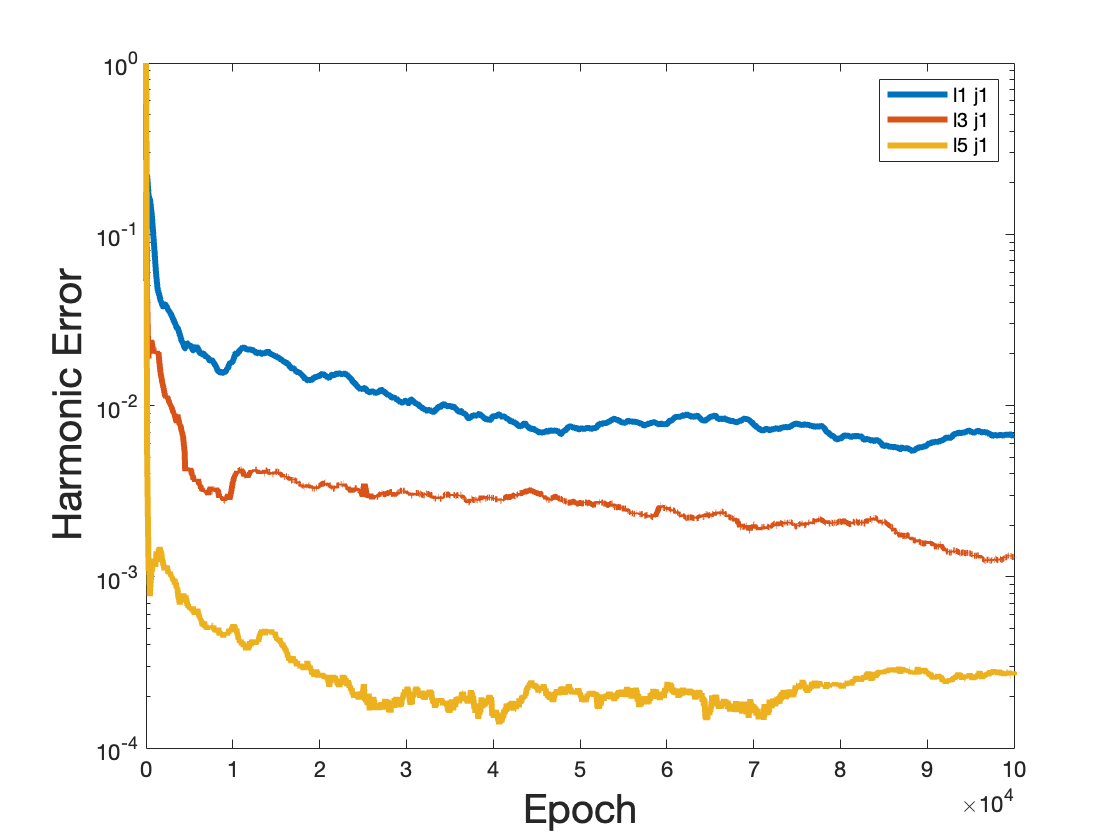}}
    
    \subfigure[$l = 1,2,\ldots,5$, $j=1$]{\includegraphics[width=6cm, height=4cm]{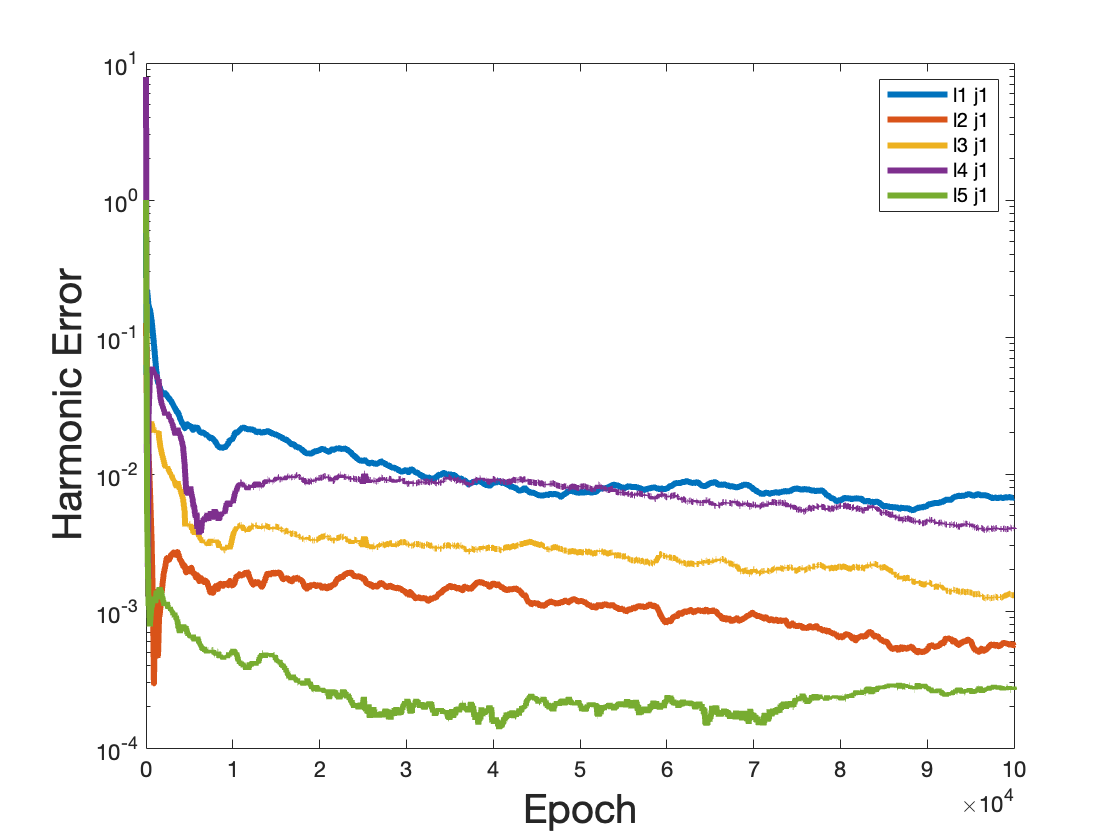}}
    \subfigure[$l = 6,7,\ldots,10$, $j=1$]{\includegraphics[width=6cm, height=4cm]{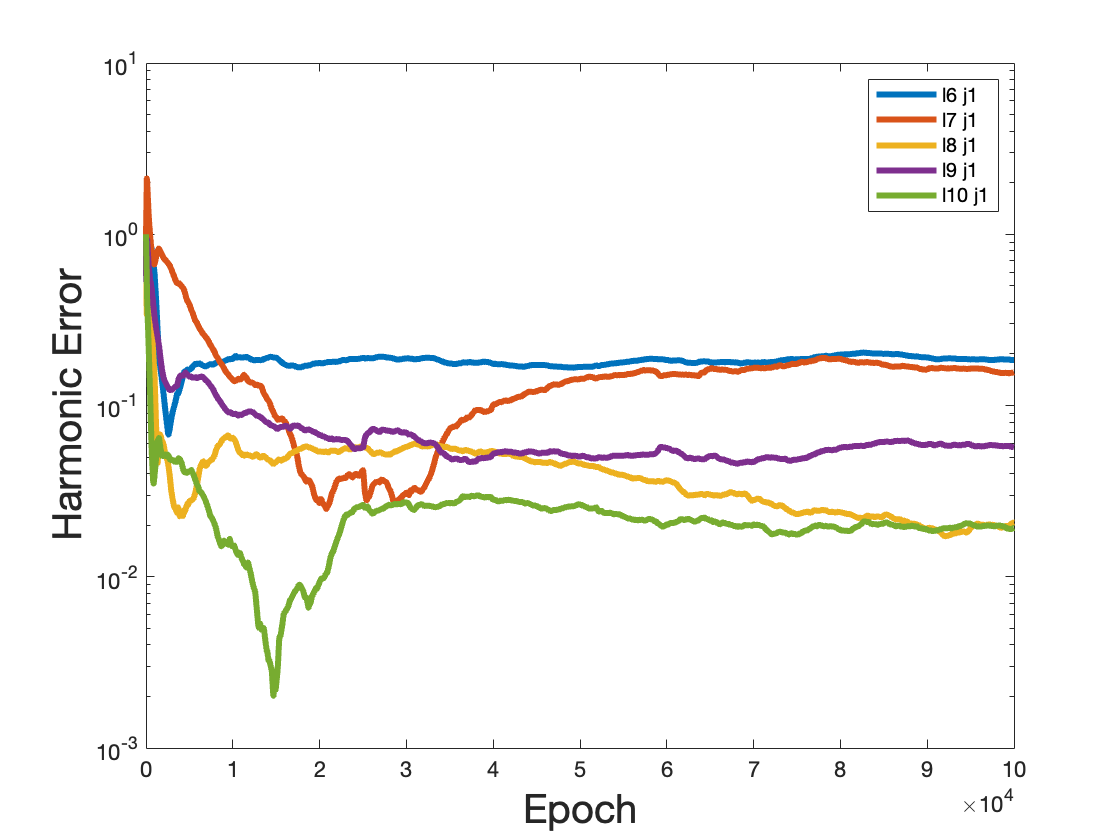}}
    \caption{Harmonic Errors showing prioritized high-frequency learning}
    \label{trigtrain3testharmonicm1}
\end{figure}

\begin{figure}
    \centering
    \subfigure[Target function]{\includegraphics[width=6cm, height=4cm]{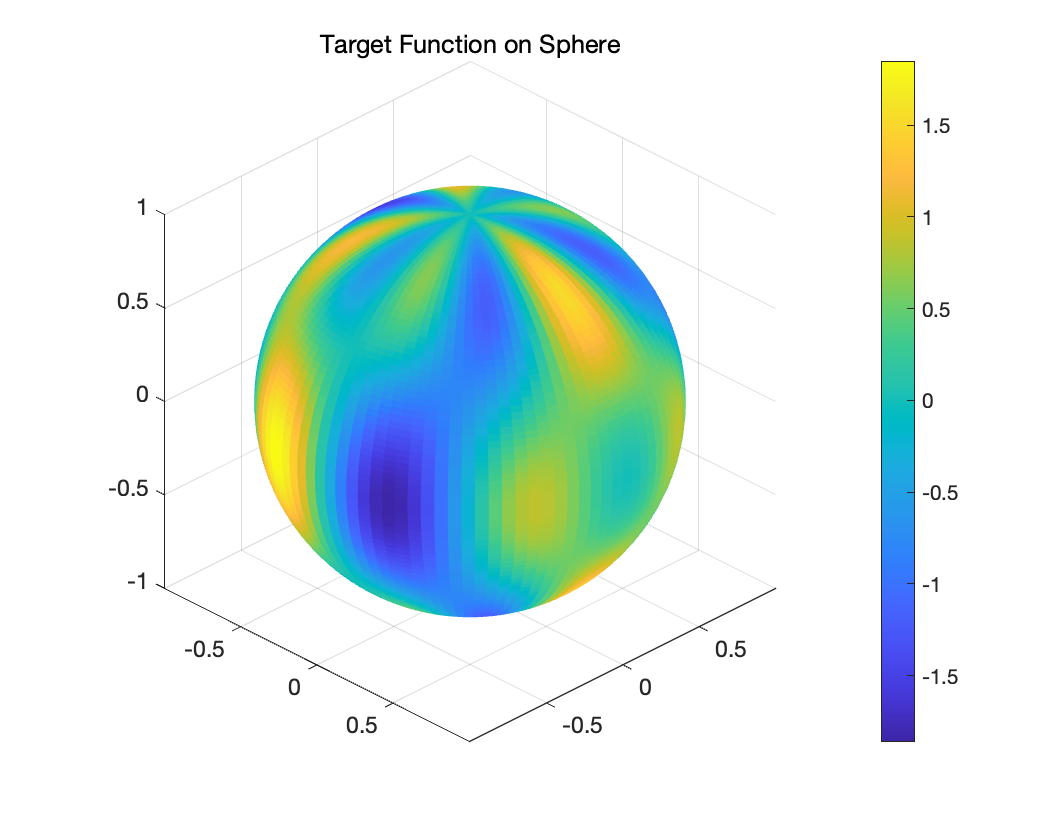}}
    \subfigure[Neural network output]{\includegraphics[width=6cm, height=4cm]{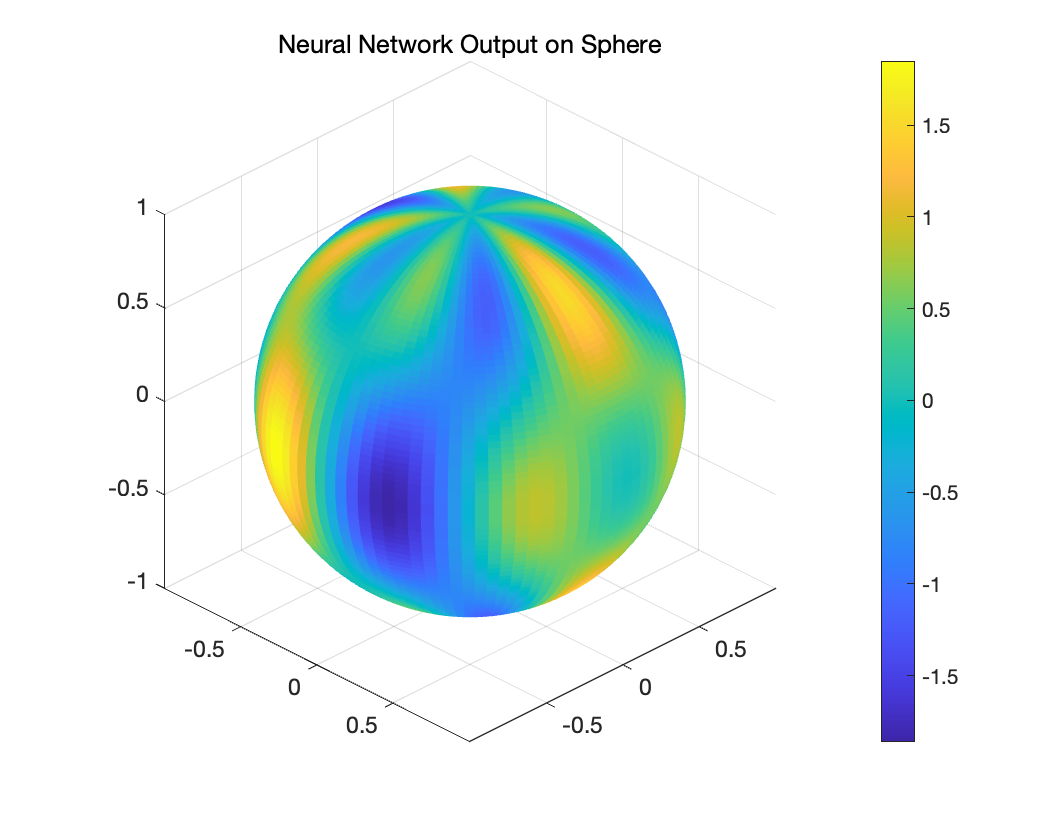}}
    \subfigure[Absolute error distribution]{\includegraphics[width=6cm, height=4cm]{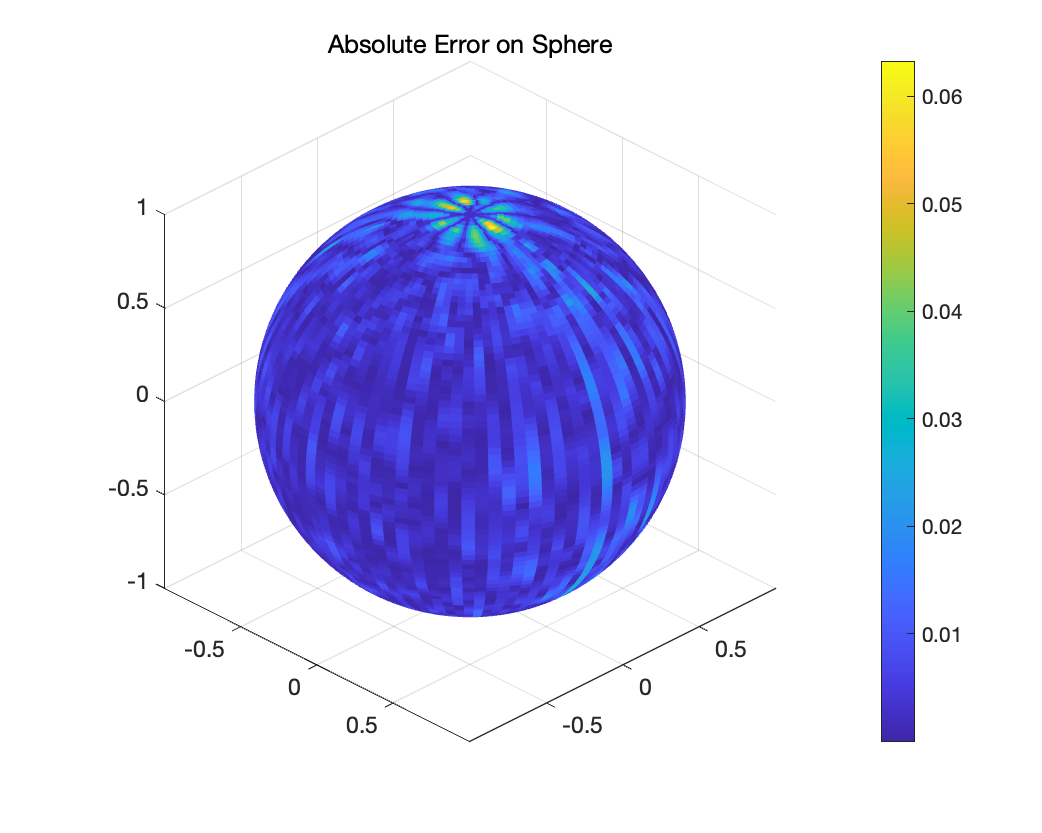}}
    \caption{Spatial comparison revealing high-frequency dominance}
    \label{trigtrain3testvisual}
\end{figure}

In this case, like Section \ref{Sec:SphericalAnalysisFW}, a high-frequency initialization $\sin(10\tau)\cos(10\phi)$ also given as Figure \ref{3testhighinitial} shown. We observe a more pronounced contradiction to the frequency principle, with the network reaching a final loss of 8.4e-4 (Figure \ref{trigtrain3test}). The harmonic error analysis (Figure \ref{trigtrain3testharmonic}) demonstrates that the network prioritizes learning the high-frequency components $\cos(5\phi)$ and $\sin(3\tau)$ before addressing the lower-frequency terms. This behavior is more distinct than in the fixed-weight case, suggesting that weight adaptation can enhance the network's capacity to learn high-frequency features independently.

The spatial error distribution (Figure \ref{trigtrain3testvisual}) shows concentrated errors in regions dominated by low-frequency components, while high-frequency features are captured with remarkable accuracy. This inverse learning pattern is maintained throughout the training process, providing strong evidence that trainable weights can fundamentally alter the frequency-dependent learning dynamics.

\textbf{Comparative Analysis}
The investigation of trainable weights reveals several key distinctions from the fixed-weight scenario:
\begin{itemize}
    \item Weight adaptation enables more efficient convergence, as evidenced by lower final loss values across all cases
    \item The network demonstrates enhanced flexibility in learning frequency components, allowing for more complex learning patterns
    \item The contradiction to the frequency principle becomes more pronounced, suggesting that weight training can fundamentally alter the natural frequency bias of neural networks
\end{itemize}

These observations extend our understanding of the frequency principle in neural networks on spherical domains. The additional degrees of freedom provided by trainable weights appear to both facilitate better function approximation and enable more diverse learning trajectories. This suggests that the frequency principle, while still relevant, may need to be reconsidered in the context of fully trainable networks, particularly when approximating functions with coupled frequency components on non-Euclidean domains.

\section*{Conclusion}
In this paper, we examined shallow ReLU networks defined on the unit sphere by exploiting the spherical harmonic expansion of both the target functions and the activation functions themselves. In the fixed-weight setting, each ReLU neuron exhibits an inherent axisymmetry around its weight vector. This local perspective explains why the spherical harmonic coefficients decay exponentially, thus naturally prioritizing low-frequency modes when weights remain static. However, we also demonstrated numerically that nontrivial high-frequency components can occasionally converge more rapidly if there is substantial error in those regions.

Extending the analysis to fully trainable networks uncovered an additional mechanism: \emph{weight rotation}. Because changing the direction of each neuron effectively reorients the local axisymmetry in the global spherical coordinate system, higher frequencies may be selectively amplified or subdued, depending on the interplay between gradient flows and the error distribution. This adaptive realignment challenges the typical assumption that deep learning proceeds strictly from low to high frequencies. Our numerical experiments substantiate the theoretical predictions that both full and partial violations of the Frequency Principle can occur under certain conditions, especially when deliberate high-frequency perturbations are introduced or when the network initialization favors higher-frequency modes.

Overall, our findings highlight that while exponential decay in the intrinsic ReLU expansion remains a dominant force toward lower frequencies, the network’s capacity to adjust weights provides a pathway for significantly more complex frequency-learning trajectories. These insights extend the classical Frequency Principle beyond unbounded or periodic domains, emphasizing the need for careful consideration of geometry, weight adaptation, and target-function structure when analyzing or designing neural networks on manifolds. Future work may explore deeper architectures, other activation functions with similar or weaker spectral biases, and more systematic strategies for mitigating adverse high-frequency effects in real-world applications.

\section{Acknowledgment}

The author would like to express their sincere gratitude to Professor Jinchao Xu for his valuable suggestions and kind invitation to KAUST, which made this research possible. We are also grateful to Research Assistant Jongho Park for insightful discussions that contributed significantly to the development of this work.

\bibliographystyle{siamplain}
\bibliography{refs_A-principle}
\end{document}